\documentclass{article}
\usepackage{rotating}
\usepackage{microtype}
\usepackage{graphicx}
\usepackage{subcaption}
\usepackage{multirow}
\usepackage{booktabs} 
\usepackage{enumitem} 
\usepackage[table]{xcolor}
\usepackage{colortbl}
\usepackage{wrapfig}
\usepackage{algorithm}
\usepackage{textcomp}

\usepackage{hyperref}

\usepackage{booktabs, multirow, tabularx}

\usepackage[preprint]{icml2026}

\usepackage{amsmath}
\usepackage{amssymb}
\usepackage{mathtools}
\usepackage{amsthm}
\usepackage{marvosym}

\usepackage[capitalize,noabbrev]{cleveref}

\theoremstyle{plain}
\newtheorem{theorem}{Theorem}[section]
\newtheorem{proposition}[theorem]{Proposition}
\newtheorem{lemma}[theorem]{Lemma}
\newtheorem{corollary}[theorem]{Corollary}
\theoremstyle{definition}
\newtheorem{definition}[theorem]{Definition}

\theoremstyle{remark}
\newtheorem{remark}[theorem]{Remark}

\usepackage[textsize=tiny]{todonotes}

\icmltitlerunning{PRM-as-a-Judge: A Dense Evaluation Paradigm for Fine-Grained Robotic Auditing}

\begin{document}

\twocolumn[
  \icmltitle{%
  \smash{\raisebox{-0.35\height}{\includegraphics[height=1.5cm]{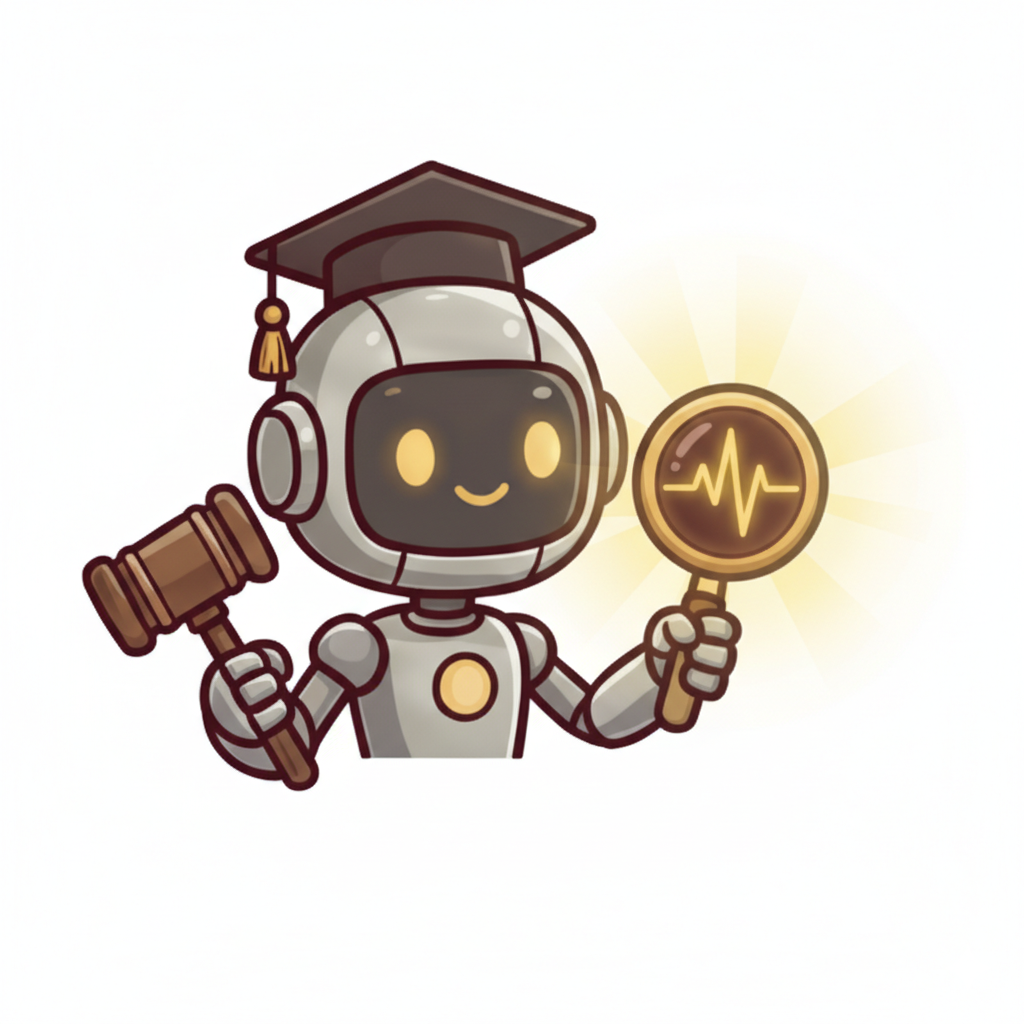}}}%
  \hspace{0.15em}%
  PRM-as-a-Judge: A Dense Evaluation Paradigm \\for Fine-Grained Robotic Auditing}



\icmlsetsymbol{equal}{*}
\icmlsetsymbol{lead}{\textdagger}
\icmlsetsymbol{cor}{\Letter}

\begin{icmlauthorlist}
  \icmlauthor{Yuheng Ji}{ia,ucas,pkuskl,equal,lead}
  \icmlauthor{Yuyang Liu}{ia,ucas,equal}
  \icmlauthor{Huajie Tan}{pkuskl,baai,equal}
  \icmlauthor{Xuchuan Huang}{pkuskl,baai}
  \icmlauthor{Fanding Huang}{baai,tsinghua}
  \icmlauthor{Yijie Xu}{baai,sydney}
  \icmlauthor{Cheng Chi}{baai}
  \icmlauthor{Yuting Zhao}{ia,ucas}
  \icmlauthor{Huaihai Lyu}{ia,ucas,pkuskl}
  \icmlauthor{Peterson Co}{pkuskl,baai}
  \icmlauthor{Mingyu Cao}{baai}
  \icmlauthor{Qiongyu Zhang}{baai,sydney}
  \icmlauthor{Zhe Li}{baai}
  \icmlauthor{Enshen Zhou}{baai,beihang}
  \icmlauthor{Pengwei Wang}{baai,lead}
  \icmlauthor{Zhongyuan Wang}{baai}
  \icmlauthor{Shanghang Zhang}{pkuskl,baai,cor}
  \icmlauthor{Xiaolong Zheng}{ia,ucas,cor}
\end{icmlauthorlist}

\icmlaffiliation{ia}{State Key Laboratory of Multimodal Artificial Intelligence Systems, Institute of Automation, Chinese Academy of Sciences}
\icmlaffiliation{ucas}{School of Artificial Intelligence, University of Chinese Academy of Sciences}
\icmlaffiliation{pkuskl}{State Key Laboratory of Multimedia Information Processing, School of Computer Science, Peking University}
\icmlaffiliation{baai}{Beijing Academy of Artificial Intelligence}
\icmlaffiliation{tsinghua}{Tsinghua University}
\icmlaffiliation{sydney}{University of Sydney}
\icmlaffiliation{beihang}{Beihang University}

\begin{center}
\textit{\textbf{Project website:}} \href{https://PRM-as-a-Judge.github.io/}{https://PRM-as-a-Judge.github.io}
\end{center}

\icmlcorrespondingauthor{Shanghang Zhang}{shanghang@pku.edu.cn}
\icmlcorrespondingauthor{Xiaolong Zheng}{xiaolong.zheng@ia.ac.cn}

\icmlkeywords{Process Reward Models, Evaluation, Robotics, Long-Horizon Manipulation, Benchmarking}

  \vskip 0.3in
]



\printAffiliationsAndNotice{$^{*}$ Equal contribution. $^{\text{\textdagger}}$ Project leaders. $^{\text{\Letter}}$ Corresponding authors.}

\begin{abstract}

Current robotic evaluation is still largely dominated by binary success rates, which collapse rich execution processes into a single outcome and obscure critical qualities such as progress, efficiency, and stability. To address this limitation, we propose \textit{\textbf{PRM-as-a-Judge}}, a dense evaluation paradigm that leverages Process Reward Models (PRMs) to audit policy execution directly from trajectory videos by estimating task progress from observation sequences.
Central to this paradigm is the \textit{OPD (Outcome–Process–Diagnosis)} metric system, which explicitly formalizes execution quality via a task-aligned progress potential. 
We characterize dense robotic evaluation through two axiomatic properties: \textit{macro-consistency}, which requires additive and path-consistent aggregation, and \textit{micro-resolution}, which requires sensitivity to fine-grained physical evolution. Under this formulation, potential-based PRM judges provide a natural instantiation of dense evaluation, with macro-consistency following directly from the induced scalar potential.
We empirically validate the micro-resolution property using \textit{RoboPulse}, a diagnostic benchmark specifically designed for probing micro-scale progress discrimination, where several trajectory-trained PRM judges outperform discriminative similarity-based methods and general-purpose foundation-model judges.
Finally, leveraging \textit{\textbf{PRM-as-a-Judge}} and the OPD metric system, we conduct a structured audit of mainstream policy paradigms across long-horizon tasks, revealing behavioral signatures and failure modes that are invisible to outcome-only metrics.
\vspace{-1em}
\end{abstract}

\begin{figure}[t!]
    \centering
    \includegraphics[width=\linewidth]{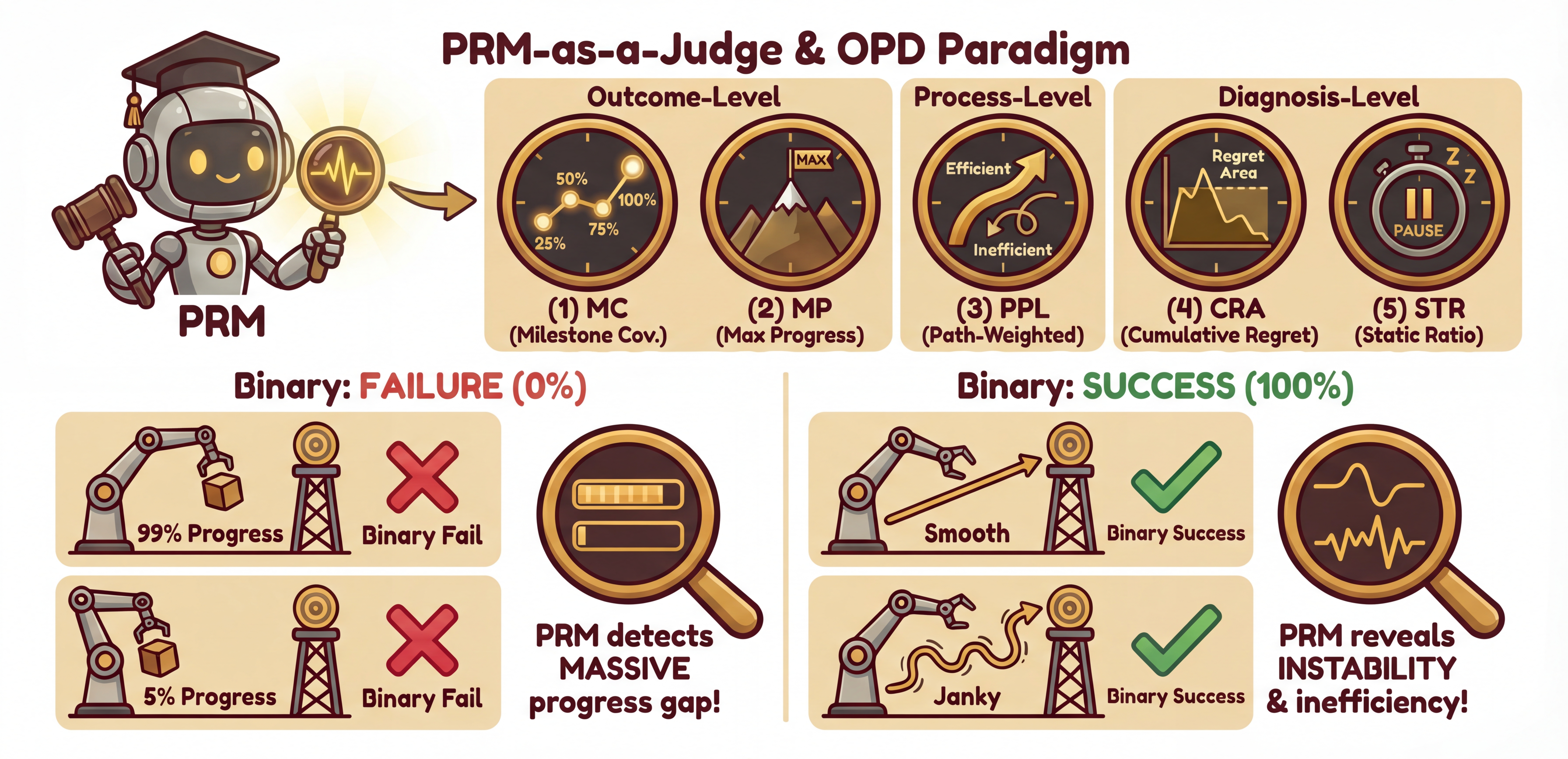}
    \caption{\textbf{PRM-as-a-Judge and OPD.}
    Binary success rate compresses an execution into a terminal outcome and obscures progress, efficiency, and stability.
    We use a Process Reward Model to induce a dense progress potential and derive OPD metrics at the outcome, process, and diagnosis levels.
    This yields fine-grained auditing that distinguishes near miss failures from early collapse and separates smooth successes from inefficient or unstable ones.}
    \label{fig:teaser}
    \vspace{-1em}
\end{figure}

\section{Introduction}
\label{sec:intro}

Robotic manipulation is rapidly moving from short, single-step skills to long-horizon, contact-rich tasks that require sustained coordination, stability, and recovery.
As policy families diversify, evaluation increasingly determines what the community optimizes, compares, and ultimately considers progress.
Yet most benchmarks still reduce an entire execution to a binary success rate~\cite{li2024simplerenv, james2020rlbench, mees2022calvin}.

This outcome-only view creates a mismatch between what we measure and what we want.
{\textit{1) Binary success lacks resolution.}}
It collapses qualitatively different executions into the same label, treating near-miss attempts and early collapses identically, and failing to separate smooth, efficient successes from ones achieved through detours, backtracking, or hesitation \cite{wang2025roboeval, elmallah2025score}.
{\textit{2) Binary success is diagnostically opaque.}}
When an episode succeeds or fails, it offers little evidence about \emph{where} progress stalled, \emph{how} the trajectory deteriorated (regression versus stagnation), or whether the outcome is limited by reachability or by instability during execution.
Fig.~\ref{fig:teaser} illustrates these two failure modes of outcome-only evaluation, motivating a dense, process-aware alternative.

A natural direction is to audit execution along the trajectory and derive intermediate signals for progress, efficiency, and stability.
However, a dense evaluator must satisfy two axioms.
\textit{Macro-consistency} requires an additive, path-independent potential so that local assessments aggregate into a coherent episode-level picture.
\textit{Micro-resolution} requires sensitivity to subtle, task-relevant physical evolution.


In this work, we propose \textit{\textbf{PRM-as-a-Judge}}, a dense evaluation paradigm that audits policy execution from trajectory videos and produces an interpretable diagnostic report beyond binary success.
The core idea is to induce a task-aligned progress signal from observations and summarize it through a multi-level metric system that separates \emph{how far} a policy reaches from \emph{how} it gets there and \emph{why} it fails.
We introduce the \textbf{OPD (Outcome–Process–Diagnosis)} metric system: outcome-level metrics describe stage-wise reachability, process-level metrics capture progress efficiency, and diagnosis-level metrics quantify regression and stagnation patterns that expose failure mechanisms and recovery behavior.
This decomposition turns trajectory videos into actionable evaluation signals.
More broadly, our goal is not only to instantiate a judge, but to establish a process-aware evaluation framework in which outcome, efficiency, and failure modes can be analyzed in a unified manner.

A key question is what kind of model can serve as such a judge.
Our central observation is that \emph{Process Reward Models} (PRMs), originally introduced as dense reward providers for reinforcement learning ~\cite{sermanet2018time, ma2022vip, gilpin2024generative, chen2025sarm}, can be naturally repurposed for fine-grained evaluation under an appropriate potential-based formulation.
PRMs are trained with dense supervision from robotic trajectories and are designed to produce progress-aware signals grounded in physical state evolution~\cite{vlac, Dopamine}.
This makes them a promising class of judges for progress auditing, rather than only a training component for improving success rates.
Our empirical study further supports this view, suggesting that trajectory-supervised PRMs can serve as effective judges for dense evaluation.

We validate \textit{\textbf{PRM-as-a-Judge}} through two complementary tracks.
First, we introduce \textbf{\textit{RoboPulse}}, a diagnostic benchmark that probes micro-scale progress discrimination under controlled relative comparison scales and diverse collection settings, and we show that PRM judges achieve substantially higher accuracy under fine-grained comparisons than alternative evaluation paradigms.
Second, we apply OPD to policy auditing on a long-horizon manipulation suite, where dense evaluation reveals stage-wise reachability profiles, separates high-quality successes from inefficient or unstable successes even when success rates appear similar, and uncovers reproducible failure fingerprints that are invisible to outcome-only metrics.
Together, these results demonstrate that dense judging not only measures more than success, 
but also provides concrete diagnostic signals that can inform future policy analysis and improvement. In summary, our main contributions are as follows:
\begin{itemize}
    \item We propose \textit{\textbf{PRM-as-a-Judge}}, a dense evaluation paradigm for auditing execution from trajectory videos using PRM judges.
    \item We introduce the \textbf{OPD} metric system that decomposes execution into outcome, process, and diagnosis signals for interpretable policy auditing beyond binary success.
    \item We introduce \textit{\textbf{RoboPulse}}, a benchmark for fine-grained progress discrimination, and show that trajectory-supervised PRM judges achieve the strongest performance under its evaluation protocol, especially on the smallest changes.
    \item We demonstrate the diagnostic utility of OPD on diverse evaluation tasks, revealing stage-wise reachability, success-conditioned execution quality trade-offs, and failure fingerprints across policy families.
\end{itemize}

\section{Related Work}

\textbf{Evaluation Metrics and Paradigms in Robotics.}
Robotic manipulation policies are predominantly evaluated using outcome-driven metrics, with Binary Success Rate (SR) serving as the standard across widely adopted benchmarks and simulation suites \cite{li2024simplerenv, james2020rlbench, bai2025embodied}. While SR supports coarse-grained comparison, it collapses the execution process into a binary verdict, obscuring distinctions between near-miss failures and catastrophic breakdowns, and masking execution pathologies such as redundancy, stagnation, or unsafe contacts \cite{wang2025roboeval, elmallah2025score,bai2025towards}. To move beyond purely outcome-based evaluation, prior work has introduced auxiliary metrics that characterize execution quality along additional dimensions, such as path length, smoothness, or mechanical effort \cite{anderson2018evaluation, zhu2020robosuite, gu2023maniskill2}. However, these metrics are typically derived from privileged simulator states such as ground-truth object poses, contact forces, or joint torques, which are unavailable in real-world or black-box settings \cite{ma2022vip, li2024simplerenv}. Consequently, despite incremental progress toward multi-dimensional evaluation, the community still lacks a general, dense, and fine-grained paradigm that operates directly on raw observations and supports systematic diagnostic assessment.

\textbf{Process Reward Models in Robotics.}
Process Reward Models (PRMs) have been extensively studied to alleviate reward sparsity in long-horizon reinforcement learning by providing dense supervision along execution trajectories. 
Early approaches employ contrastive or discriminative objectives to capture relative temporal progress or instruction–trajectory alignment, while subsequent work introduces generative or stage-aware formulations that incorporate semantic or structural priors into reward representations \cite{sermanet2018time, ma2022vip, ma2023liv, gilpin2024generative, chen2025sarm}. 
More recently, large-scale foundation reward models trained on diverse robotic datasets demonstrate the ability to learn absolute state potentials over the task manifold, yielding additive and path-independent task progress estimates \cite{vlac, Dopamine}. 
Despite their effectiveness for policy optimization, PRMs are predominantly treated as auxiliary training signals, leaving their potential as stand-alone evaluators for post-hoc execution assessment largely unexplored.

\textbf{Foundation Models as Automated Judges.}
The paradigm of using foundation models as automated judges has been widely adopted in natural language processing, where \emph{LLM-as-a-Judge} frameworks demonstrate strong agreement with human preferences in text evaluation \cite{zheng2023judging, alpaca_eval, pandalm2024, li2024llmasajudge}. 
Inspired by this success, recent studies have explored extending large vision-language models (VLMs) to robotic evaluation, such as safety assessment and heuristic reward scoring \cite{gdm2024autort, ma2023eureka}. 
However, directly transferring this paradigm to robotic manipulation reveals a fundamental limitation: general-purpose foundation models lack the physical resolution required to reliably evaluate fine-grained state transitions, contact dynamics, and geometric feasibility \cite{huang2023voxposer, wang2025roboeval}. 
This mismatch motivates the need for evaluators grounded in physical interaction data, leading us to adopt PRMs as domain-specific judges for physically meaningful assessment.

\section{The PRM-as-a-Judge Paradigm}
\label{sec:method}

To overcome the low resolution of binary success metrics, we introduce a dense evaluation paradigm that audits execution from trajectory videos and yields interpretable OPD signals.
We first state two axioms, macro-consistency and micro-resolution, then define the OPD metric system, and finally describe an effective instantiation via PRMs.

\subsection{Theoretical Formulation and Axioms}
\label{subsec:theory}

We represent an execution as an information-state trajectory
$\tau = (x_0, x_1, \dots, x_T)$, where $x_t$ denotes the task-relevant
information available to the judge at time $t$ (e.g., the current observation
or an observation-derived state). For a fixed task instruction, we assume a
potential $\Phi(x_t) \in [0,1]$ that measures progress toward the goal.
A valid dense evaluator must satisfy two properties, macro-consistency for additive and path-independent aggregation, and micro-resolution for discriminating fine-grained task-relevant state changes.

\paragraph{Property 1: Macro-Consistency via Temporal Additivity.}
A valid dense evaluator must induce a consistent progress signal across different temporal scales.
Let $S(x_i, x_j)$ denote the estimated progress from observed state $x_i$ to $x_j$.
For any trajectory segment $[t_0, t_2]$ and any intermediate time $t_1 \in (t_0, t_2)$, the evaluator should satisfy
\begin{equation}
S(x_{t_0}, x_{t_2}) = S(x_{t_0}, x_{t_1}) + S(x_{t_1}, x_{t_2}).
\end{equation}
This property means that progress measured over a long segment should agree with the sum of progress measured over its subsegments.
As a result, episode-level evaluation does not depend on how the trajectory is temporally partitioned.
Under the proposed formulation, this property is satisfied when progress is induced by a scalar potential through
\begin{equation}
S(x_i, x_j) = \Phi(x_j) - \Phi(x_i),
\end{equation}
whereas evaluators that score each state pair independently do not in general guarantee such additivity.

\paragraph{Property 2: Micro-Resolution of Progress Signals.}
Beyond temporal coherence, the progress signal must be sufficiently sensitive to fine-grained, task-relevant state evolution.
Specifically, the evaluator should assign distinguishable progress values to observed states that differ in meaningful physical execution, even when such differences occur at a small temporal or spatial scale.
Formally, for nearby observed states $(x_t, x_{t+\Delta})$ drawn from the same execution trajectory with small $\Delta$, task-relevant local changes should induce non-degenerate differences in the potential $\Phi$, rather than collapsing uniformly toward zero.
This property prevents trivial or collapsed evaluators that assign nearly constant progress scores throughout execution, rendering dense diagnosis impossible.

\begin{table*}[t]
\centering
\renewcommand{\arraystretch}{1.3}
\caption{\textbf{Overview of the OPD Metric System.} \textbf{Range}: Valid value domain, where MC takes discrete quartile values and others are continuous. \textbf{Pref.}: Preference direction ($\uparrow$: higher is better; $\downarrow$: lower is better).}
\label{tab:opd_definitions}
\resizebox{0.95\textwidth}{!}{%
\begin{tabular}{llccl}
\toprule
\textbf{Level} & \textbf{Metric} & \textbf{Range} & \textbf{Pref.} & \textbf{Interpretation} \\
\midrule
\multirow{2}{*}{Outcome} & Milestone Coverage (MC) & $\{0, 0.25, \dots, 1\}$ & $\uparrow$ & The furthest milestone reached during execution. \\
 & Max Progress (MP) & $[0, 1]$ & $\uparrow$ & The maximum progress score attained during the episode. \\
\midrule
Process & Path-weighted Progress Length (PPL) & $[0, 1]$ & $\uparrow$ & Efficiency of the execution path relative to net progress and path variation. \\
\midrule
\multirow{2}{*}{Diagnosis} & Cumulative Regret Area (CRA) & $[0, 1]$ & $\downarrow$ & Measures the severity and duration of regression from the best-so-far progress level. \\
 & Stagnation Ratio (STR) & $[0, 1]$ & $\downarrow$ & Fraction of time steps with negligible task-relevant progress change. \\
\bottomrule
\end{tabular}%
}
\end{table*}

\subsection{The OPD Metric System}
\label{subsec:opd}

Grounded in the potential function $\Phi(x)$, we construct the OPD (Outcome-Process-Diagnosis) metric system. 
Tab.~\ref{tab:opd_definitions} summarizes the definitions.
Detailed mathematical derivations, including robustness proofs against trivial solutions and {analyses of discarded alternative formulations}, are provided in Appendix~\ref{app:metric_derivation}.

\paragraph{1. Outcome Level.}
Metrics at this level quantify how far the execution progresses toward task completion.
\begin{itemize}
\item \textbf{Milestone Coverage (MC):} 
    Instead of a single binary bit, we discretize the potential space into quartiles to robustly identify the furthest stage reached.
    \begin{equation}
    \begin{split}
        \text{MC}(\tau) = \max \{q \in \{0, 0.25, 0.5, 0.75, 1\} \\
        \mid \exists t, \Phi(x_t) \ge q\}
    \end{split}
    \end{equation}
    It serves as a ``soft success rate'', distinguishing policies that fail at the approach stage (0.25) from those that fail during final alignment (0.75).
    
\item \textbf{Max Progress (MP):} 
    Defined as the peak scalar potential achieved during the episode, capturing the policy's capacity boundary:
    \begin{equation}
        \text{MP}(\tau) = \max_{t \in [0, T]} \Phi(x_t)
    \end{equation}
\end{itemize}

\paragraph{2. Process Level.}
This level evaluates the efficiency of the execution path.
\begin{itemize}
\item \textbf{Path-weighted Progress Length (PPL):}
To distinguish efficient execution from redundant motion, we define PPL as a
gated ratio between net progress and total potential variation. The terminal
potential $\Phi(x_T)$ downweights incomplete attempts, while a rectified
net-progress term and a small constant $\delta > 0$ ensure numerical stability
(see Appendix~\ref{app:metric_derivation}):
\begin{equation}
\text{PPL}(\tau)
=
\Phi(x_T)\cdot
\frac{[\Phi(x_T)-\Phi(x_0)]_{+}}
{\sum_{t=1}^{T}\left|\Phi(x_t)-\Phi(x_{t-1})\right|+\delta}
\end{equation}
A higher PPL indicates more monotonic progress with less redundant
back-and-forth.
\end{itemize}

\paragraph{3. Diagnosis Level.}
This level localizes the nature of failures, distinguishing between stability issues and planning latencies.
\begin{itemize}

    \item \textbf{Cumulative Regret Area (CRA):}
    We define regret as the deviation from the historical peak potential. Physically, CRA measures the persistence of state regression during execution.
    \begin{equation}
    \text{CRA}(\tau)
    =
    \frac{1}{T+1}\sum_{t=0}^{T}
    \left[
    \left(\max_{0\le k\le t}\Phi(x_k)\right)-\Phi(x_t)
    \right]
    \end{equation}
    CRA measures how long and how severely the execution regresses from its best-so-far state.

    \item \textbf{Stagnation Ratio (STR):} 
    To diagnose inference latency or decision uncertainty, STR measures the proportion of time steps where the potential change falls below a noise threshold $\epsilon$:
    \begin{equation}
        \text{STR}(\tau) = \frac{1}{T} \sum_{t=1}^T \mathbb{I}(|\Phi(x_t) - \Phi(x_{t-1})| < \epsilon)
    \end{equation}
    STR measures how often progress is negligible under the task-conditioned potential, indicating hesitation or stalled interaction.

\end{itemize}

\subsection{Effective Instantiation via PRMs}
\label{subsec:prm}

Having established the axiomatic requirements for dense evaluation, we consider potential-based PRM judges as a practical instantiation of the progress potential. 
In our framework, a PRM is treated as a task-conditioned evaluator that assigns each observed state a location on a latent progress manifold, inducing a scalar potential
\begin{equation}
\Phi:\mathcal{X}\rightarrow[0,1],
\end{equation}
which serves as the foundation for all OPD metrics in Sec.~\ref{subsec:opd}. 
This abstraction is policy-agnostic and enables post-hoc evaluation of execution quality.

\paragraph{Macro-consistency.}
Under the proposed potential-based formulation, a PRM judge assigns each observed state a scalar progress score under a fixed task context, thereby inducing a globally comparable progress ordering within that task.
As a result, evaluations at different time steps lie on a shared scale, and the additivity requirement in Sec.~\ref{subsec:theory} follows directly from the potential-difference construction introduced above.
A formal proof under this formulation is provided in Appendix~\ref{app:macro_consistency_proof}.
In contrast, discriminative similarity-based methods, and more generally evaluators based on relative or pairwise comparison, can drift across task-equivalent states under viewpoint or multimodal ambiguity, preventing the formation of a globally consistent potential field, as analyzed in Appendix~\ref{app:inconsistency_proof}.

\paragraph{Micro-resolution.}
Dense trajectory supervision makes PRMs plausible candidates for capturing small but task-relevant state evolution.
We evaluate this capability on RoboPulse in Sec.~\ref{sec:benchmark}, which is designed for micro-scale progress discrimination.

\paragraph{Summary.}
Under the proposed potential-based formulation, PRM-as-a-Judge is macro-consistent by construction, and dense trajectory supervision makes PRM judges well-motivated candidates for micro-resolution, which we validate empirically in Sec.~\ref{sec:experiments}.

\section{RoboPulse: A Diagnostic Benchmark for Progress Evaluation}
\label{sec:benchmark}

As established in Sec.~\ref{subsec:theory}, dense evaluation requires both macro-consistency and micro-resolution. Under the proposed potential-based formulation, the former is guaranteed structurally, whereas the latter must be examined empirically.
To this end, we introduce \textit{RoboPulse}, a diagnostic benchmark for probing judge limitations under controlled fine-grained comparisons.

\subsection{Problem Formulation}

RoboPulse formulates progress evaluation as a \emph{pairwise progress judgment} problem. Given two observations sampled from the same execution episode, the evaluator is required to determine whether the latter state represents progress or regression with respect to the task objective.
This formulation enables systematic probing of signed progress discrimination across a range of relative scales, including cases involving minimal physical change, without requiring recovery of an absolutely calibrated progress scale. We therefore construct RoboPulse within curated monotonic phases and use it to evaluate signed progress discrimination under controlled ambiguity.
Each RoboPulse instance is defined as a tuple:
\[
(c, O_{\text{ref}}^{\text{start}}, O_{\text{ref}}^{\text{end}}, O_{\text{before}}, O_{\text{after}}, y),
\]
where $c$ denotes the task instruction, $O_{\text{before}}$ and $O_{\text{after}}$ are two sets of synchronized multi-view observations sampled from the same execution episode, and $y \in \{+1, -1\}$ indicates whether the latter state represents progress or regression relative to the former.

Reference start and end observations, when available, serve as optional anchors for the overall task scope. Evaluators may utilize these references according to their respective input capabilities. RoboPulse focuses exclusively on signed progress direction and does not include a neutral category because the benchmark is constructed from curated monotonic phases, with negligible or ambiguous intervals filtered out during annotation.

\begin{figure}[t!]
    \centering
    \includegraphics[width=\linewidth]{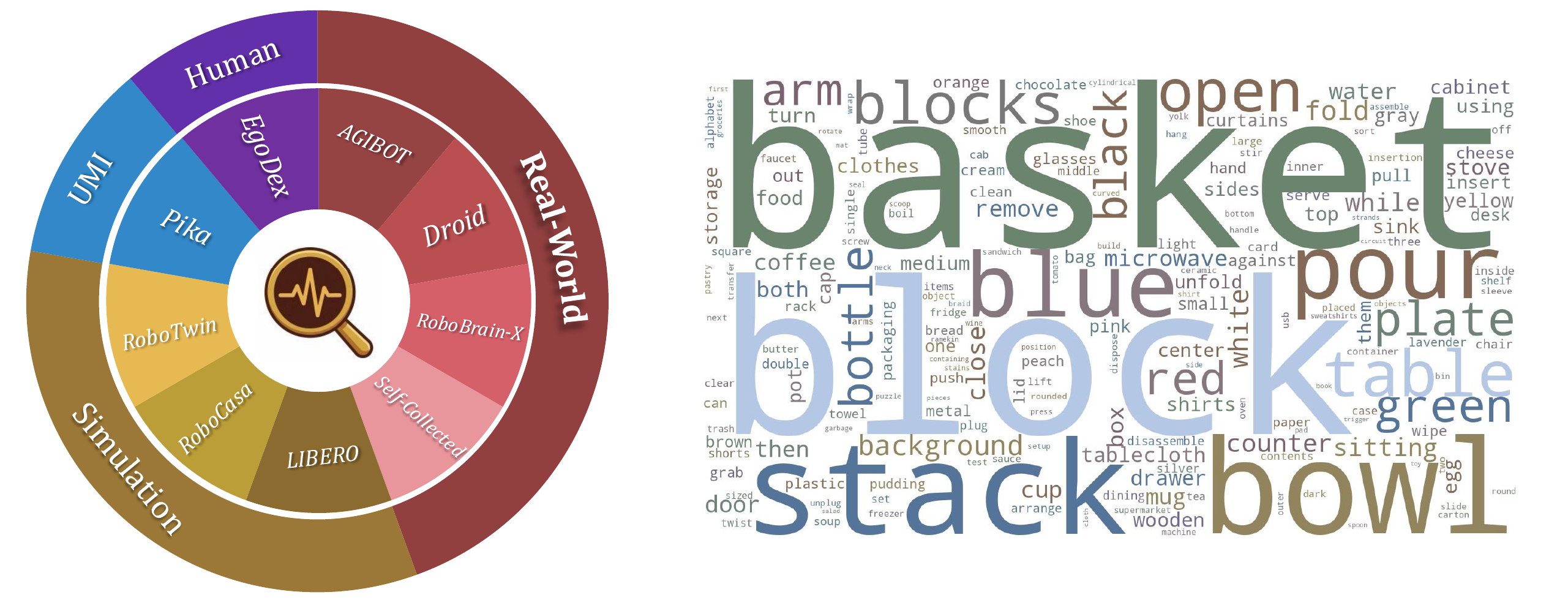}
    \caption{\textbf{RoboPulse overview.} \textbf{Left:} data composition across collection settings and embodiment--setting categories. \textbf{Right:} task semantic coverage via a token word cloud extracted from task names. RoboPulse is designed to probe micro-scale progress discrimination under diverse embodiments and visual domains.}
    \label{fig:robopulse_overview}
    \vspace{-1em}
\end{figure}

\subsection{Benchmark Overview}
RoboPulse is a large-scale, multi-source benchmark designed for evaluating fine-grained progress discrimination.
It comprises 1,800 pairwise progress judgment cases constructed from 1,622 execution episodes spanning 816 tasks, collected across 7 data sources.
The benchmark is organized into 9 embodiment--setting categories, each contributing 200 progress pairs with balanced positive and negative labels, and collectively covering multiple robot platforms, sensing configurations, and data collection settings.
These cases are further stratified into Small, Medium, and Large relative comparison scales to probe progress discrimination at different granularities.
Fig.~\ref{fig:robopulse_overview} provides an overview of the data composition and task coverage of RoboPulse, with detailed statistics reported in Appendix~\ref{app:robopulse_data}.
The hop-size distribution spans a wide range of relative comparison scales, with approximately balanced coverage over progress and regression directions (see Appendix~\ref{app:norm_sampling}).

\subsection{Benchmark Construction}
RoboPulse is built from manipulation trajectories collected in real robot teleoperation, simulation rollouts, UMI-based collection, and egocentric human demonstrations, with synchronized multiview observations when available.
To control the relative comparison scale, we adopt hop-based normalization~\cite{Dopamine} within semantically coherent phases.
We annotate key frames to segment each episode into semantically coherent phases, retain only phases in which task progress is monotonic, and exclude intervals with negligible, oscillatory, or annotation-ambiguous progress.
We then sample progress pairs from controlled normalized hop ranges (Small/Medium/Large) within each retained phase, yielding broad scale coverage and balanced progress directions, as summarized in Appendix~\ref{app:norm_sampling}.

\newcommand{\best}[1]{\textbf{#1}}
\newcommand{\second}[1]{\underline{#1}}

\section{Experiments}
\label{sec:experiments}

In this section, we validate PRM-as-a-Judge on RoboPulse and apply it to policy auditing across several long-horizon manipulation tasks.

\begin{itemize}
    \item \textbf{RQ1: Micro-Resolution of PRM-as-a-Judge.}
    Do PRM-based judges exhibit stronger micro-resolution than alternatives under the Small hop range?

    \item \textbf{RQ2: Stage-wise Decomposition of Reachability.}
    When binary success rate loses resolution on long-horizon manipulation, can OPD provide an interpretable decomposition that localizes \emph{how far} a policy can typically reach and \emph{where} failures concentrate along the execution?

    \item \textbf{RQ3: Execution Quality Conditioned on Success.}
    When outcome-level success appears comparable, can process and diagnosis metrics separate \emph{high-quality successes} from \emph{inefficient or unstable successes}, revealing systematic trade-offs in execution quality?

    \item \textbf{RQ4: Failure Fingerprints for Mechanistic Diagnosis.}
    Do policy families exhibit reproducible \emph{failure fingerprints} in the OPD space when focusing on failed episodes, enabling mechanistic diagnosis and actionable improvement directions?
\end{itemize}

\begin{table*}[!ht]
\centering
\footnotesize
\setlength{\tabcolsep}{3.2pt}
\renewcommand{\arraystretch}{1.12}
\caption{\textbf{Pairwise progress-judgment accuracy on RoboPulse under different relative comparison scales.}
We report accuracy (0--1) over Small/Medium/Large hop ranges, stratified by four collection settings (Real, Sim, UMI, Human), together with per-scale and overall averages.
All evaluators are tested zero-shot through their native input interfaces, so results reflect practical judging capability rather than a strictly input-matched comparison.}
\label{tab:robopulse_main}
\scalebox{0.99}{%
\begin{tabular}{l|ccccc|ccccc|ccccc|c}
\toprule
& \multicolumn{5}{c|}{\textbf{Small}} 
& \multicolumn{5}{c|}{\textbf{Medium}} 
& \multicolumn{5}{c|}{\textbf{Large}} 
& \textbf{AVG} \\
\cmidrule{2-16}
\textbf{Method}
& Real & Sim & UMI & Human & Avg
& Real & Sim & UMI & Human & Avg
& Real & Sim & UMI & Human & Avg
&  \\
\midrule

\multicolumn{17}{l}{\cellcolor{gray!12}\textbf{Discriminative Similarity-Based Methods}} \\
CLIP ViT-B/32 (I2I) & 0.56 & 0.52 & 0.58 & 0.58 & 0.56 & 0.56 & 0.64 & 0.52 & 0.50 & 0.56 & 0.60 & 0.69 & 0.63 & 0.51 & 0.61 & 0.57 \\
CLIP ViT-L/14 (I2I) & 0.52 & 0.56 & 0.55 & 0.56 & 0.54 & 0.61 & 0.65 & 0.50 & 0.59 & 0.59 & 0.64 & 0.72 & 0.68 & 0.54 & 0.65 & 0.59 \\
CLIP ViT-B/32 (T2I) & 0.48 & 0.49 & 0.56 & 0.49 & 0.51 & 0.51 & 0.48 & 0.35 & 0.34 & 0.42 & 0.38 & 0.42 & 0.61 & 0.54 & 0.49 & 0.47 \\
CLIP ViT-L/14 (T2I) & 0.50 & 0.46 & 0.48 & 0.52 & 0.49 & 0.49 & 0.42 & 0.35 & 0.52 & 0.45 & 0.46 & 0.42 & 0.42 & 0.51 & 0.45 & 0.46 \\

\midrule
\multicolumn{17}{l}{\cellcolor{gray!12}\textbf{General Foundation-Model Judges}} \\
Gemini 3 Pro Preview & 0.55 & 0.62 & 0.43 & 0.56 & 0.54 & 0.65 & 0.70 & 0.73 & 0.59 & 0.67 & 0.72 & \second{0.85} & 0.77 & 0.74 & 0.77 & 0.66 \\
GPT-5.2 & 0.46 & 0.46 & 0.47 & 0.49 & 0.47 & 0.58 & 0.57 & 0.54 & 0.34 & 0.51 & 0.57 & 0.62 & 0.70 & 0.59 & 0.62 & 0.53 \\
Qwen3-VL-4B-Instruct & 0.47 & 0.49 & 0.34 & 0.53 & 0.46 & 0.56 & 0.56 & 0.59 & 0.50 & 0.55 & 0.60 & 0.71 & 0.65 & 0.68 & 0.66 & 0.56 \\
Qwen3-VL-8B-Instruct & 0.49 & 0.51 & 0.44 & 0.47 & 0.48 & 0.61 & 0.65 & 0.61 & 0.41 & 0.57 & 0.68 & 0.79 & 0.77 & 0.69 & 0.74 & 0.59 \\

\midrule
\multicolumn{17}{l}{\cellcolor{gray!12}\textbf{Progress Reward Model Judges}} \\
VLAC & {0.60} & 0.61 & \second{0.66} & 0.57 & 0.61 & 0.67 & \second{0.75} & 0.70 & \second{0.75} & \second{0.72} & \best{0.79} & 0.78 & \second{0.81} & \second{0.78} & \second{0.79} & \second{0.71} \\
GVL  & \second{0.61} & \second{0.67} & {0.58} & \second{0.67} & \second{0.63} & \second{0.71} & 0.71 & \second{0.75} & 0.69 & \second{0.72} & \second{0.78} & 0.83 & 0.78 & 0.75 & 0.78 & \second{0.71} \\
Robo-Dopamine 
& \best{0.66} & \best{0.89} & \best{0.87} & \best{0.76} & \best{0.80}
& \best{0.79} & \best{0.89} & \best{0.88} & \best{0.84} & \best{0.85}
& \second{0.78} & \best{0.90} & \best{0.88} & \best{0.83} & \best{0.85}
& \best{0.83} \\

\bottomrule
\end{tabular}%
}
\vspace{0.5em}
\end{table*}

\subsection{Experimental Setup}
\label{subsec:exp_setup}

We conduct two complementary sets of experiments to validate the \textbf{\textit{PRM-as-a-Judge}} paradigm.
The first evaluates evaluator capability on RoboPulse, focusing on fine-grained progress discrimination under controlled relative comparison scales.
The second applies PRM-as-a-Judge to policy auditing on RoboTwin 2.0~\cite{robotwin20}, examining what OPD reveals beyond binary success when comparing different policy families.

\paragraph{RoboPulse protocol and evaluation paradigms.}
We follow the RoboPulse protocol introduced in Sec.~\ref{sec:benchmark}.
Evaluators take as input two observations sampled from the same execution episode and predict whether the latter state represents progress or regression with respect to the task objective.
We report pairwise judgment accuracy aggregated over the benchmark, further stratified by hop magnitude (Small/Medium/Large) and by the four collection settings (Real/Sim/UMI/Human).

We evaluate PRM-as-a-Judge against two commonly used non-PRM evaluation paradigms.
The first consists of \emph{discriminative similarity-based methods} based on CLIP-style retrieval~\cite{clip,vlc}, instantiated as image--image and text--image similarity variants.
The second consists of \emph{general foundation-model judges} that perform zero-shot pairwise progress judgment via multimodal reasoning, including Gemini 3 Pro Preview~\cite{gemini3pro}, GPT-5.2~\cite{gpt5.2}, and Qwen3-VL~\cite{qwen3vl}.
Within PRM-as-a-Judge, we instantiate the judge using representative trajectory-supervised progress or reward models, including VLAC~\cite{vlac}, GVL~\cite{gvl}, and Robo-Dopamine~\cite{Dopamine}.
All evaluators are tested without task-specific fine-tuning and are queried through interfaces compatible with their native input formats.
The comparison therefore focuses on practical zero-shot judging capability under interface-compatible prompting, rather than a strictly input-standardized evaluation.
Importantly, our goal is not to claim that current PRMs are perfect evaluators, but to establish a diagnostic paradigm whose limitations are explicit and measurable.

\paragraph{RoboTwin policy auditing protocol.}
For policy auditing, we evaluate 5 representative policy families (ACT~\cite{act}, DP~\cite{dp}, RDT~\cite{rdt}, $\pi_0$~\cite{pi0}, and OpenVLA-OFT~\cite{openvlaoft}) on seven long-horizon manipulation tasks in RoboTwin 2.0.
All policies follow the official training configuration with the same data budget and are evaluated with 50 rollouts per policy per task.
Unless otherwise specified, we use Robo-Dopamine as the default PRM judge, since it achieves the strongest micro-resolution in RQ1. We therefore instantiate OPD on RoboTwin with the strongest validated judge under RoboPulse. Due to space constraints, we report three representative tasks in the main paper, with full results in Appendix~\ref{more_results}.

\paragraph{Reporting conventions.}
On RoboPulse, we report accuracy in $[0,1]$.
On RoboTwin, all OPD metrics are reported on a percentage scale for readability (multiplying by 100).

\subsection{Micro-Resolution of PRM Judges on RoboPulse (RQ1)}
\label{subsec:rq_micro_resolution}

To validate micro-resolution, we evaluate whether a judge can distinguish progress from regression as the relative comparison scale becomes increasingly fine-grained.
We report pairwise judgment accuracy under Small, Medium, and Large hop ranges across four collection settings (Tab.~\ref{tab:robopulse_main}).

\textbf{\textit{1) PRM-based judges deliver the strongest micro-resolution among evaluated paradigms.}}
Robo-Dopamine achieves the highest overall accuracy of 0.83.
The best alternative baselines remain lower, including Gemini at 0.66 and Qwen3-VL-8B at 0.59.
Discriminative CLIP variants cluster between 0.46 and 0.59 overall, indicating limited resolution under this protocol.

\textbf{\textit{2) The performance margin widens under the smallest relative scale.}}
Under the Small hop range, Robo-Dopamine reaches 0.80 average accuracy.
The other PRM baselines reach 0.61 and 0.63, while Gemini drops to 0.54 and GPT-5.2 to 0.47.
This gap suggests that fine-grained comparisons reduce reliance on coarse visual cues and favor progress-grounded supervision.

\textbf{\textit{3) Micro-resolution is consistent across collection settings.}}
Under Small hops, Robo-Dopamine stays above 0.75 in every setting and exceeds 0.85 on simulation and UMI data.
In contrast, several non-PRM baselines show larger cross setting variance, with visible degradation on Human and UMI cases.
Overall, Tab.~\ref{tab:robopulse_main} supports micro-resolution under fine-grained scales across diverse collection settings.

\begin{table*}[t]
\centering
\setlength{\tabcolsep}{3.5pt}
\renewcommand{\arraystretch}{1.2} 
\caption{\textbf{OPD auditing on RoboTwin 2.0.}
Five policy families on three representative long-horizon tasks (50 rollouts each).
MC@25/50/75/100 denotes milestone coverage.
\textbf{P.} denotes the process level (PPL), and \textbf{D.} denotes the diagnosis level (CRA and STR).
Full results are in Appendix~\ref{more_results}.}

\label{tab:robotwin_audit_main}

\resizebox{\textwidth}{!}{%
\begin{tabular}{l | cccc | c | c | c c | cccc | c | c | c c | cccc | c | c | c c}
\toprule

\multirow{4}{*}{} 
& \multicolumn{8}{c|}{\textbf{Blocks Ranking RGB}} 
& \multicolumn{8}{c|}{\textbf{Handover Mic}} 
& \multicolumn{8}{c}{\textbf{Place Bread Basket}} \\
\cmidrule(lr){2-9} \cmidrule(lr){10-17} \cmidrule(lr){18-25}

& \multicolumn{5}{c|}{\textbf{Outcome Level}} & \textbf{P.} & \multicolumn{2}{c|}{\textbf{D.}} 
& \multicolumn{5}{c|}{\textbf{Outcome Level}} & \textbf{P.} & \multicolumn{2}{c|}{\textbf{D.}} 
& \multicolumn{5}{c|}{\textbf{Outcome Level}} & \textbf{P.} & \multicolumn{2}{c}{\textbf{D.}} \\
\cmidrule(lr){2-6} \cmidrule(lr){7-7} \cmidrule(lr){8-9}
\cmidrule(lr){10-14} \cmidrule(lr){15-15} \cmidrule(lr){16-17}
\cmidrule(lr){18-22} \cmidrule(lr){23-23} \cmidrule(lr){24-25}

& \multicolumn{4}{c|}{\textbf{MC}} & \multirow{2}{*}{\textbf{MP}} & \multirow{2}{*}{\textbf{PPL}} & \multirow{2}{*}{\textbf{CRA}} & \multirow{2}{*}{\textbf{STR}} 
& \multicolumn{4}{c|}{\textbf{MC}} & \multirow{2}{*}{\textbf{MP}} & \multirow{2}{*}{\textbf{PPL}} & \multirow{2}{*}{\textbf{CRA}} & \multirow{2}{*}{\textbf{STR}} 
& \multicolumn{4}{c|}{\textbf{MC}} & \multirow{2}{*}{\textbf{MP}} & \multirow{2}{*}{\textbf{PPL}} & \multirow{2}{*}{\textbf{CRA}} & \multirow{2}{*}{\textbf{STR}} \\
\cmidrule(lr){2-5} \cmidrule(lr){10-13} \cmidrule(lr){18-21}

& {@25} & {@50} & {@75} & {@100} & & & & 
& {@25} & {@50} & {@75} & {@100} & & & & 
& {@25} & {@50} & {@75} & {@100} & & & & \\
\midrule

ACT & 
84 & 44 & 22 & 2 & 49.9 & 11.7 & 8.99 & 59.7 & 
100 & 100 & 94 & 74 & 96.8 & 72.3 & 4.08 & 44.1 & 
100 & 74 & 46 & 4 & 73.1 & 17.6 & 15.5 & 65.4 \\ 

DP & 
94 & 40 & 18 & 0 & 51.7 & 4.07 & 16.3 & 43.8 & 
100 & 94 & 88 & 44 & 93.8 & 66.0 & 5.49 & 57.2 & 
100 & 94 & 74 & 16 & 87.6 & 21.3 & 16.9 & 48.0 \\ 

RDT & 
100 & 62 & 30 & 0 & 61.2 & 6.19 & 16.3 & 39.0 & 
100 & 100 & 100 & 100 & 100 & 84.2 & 1.45 & 39.8 & 
100 & 100 & 78 & 8 & 90.4 & 16.6 & 22.7 & 37.1 \\ 

pi0 & 
96 & 66 & 40 & 8 & 63.4 & 15.9 & 11.5 & 48.4 & 
100 & 100 & 100 & 98 & 99.4 & 88.1 & 1.03 & 42.7 & 
100 & 94 & 62 & 16 & 83.7 & 21.2 & 18.9 & 47.6 \\ 

OpenVLA-OFT & 
98 & 42 & 6 & 0 & 48.3 & 2.39 & 17.8 & 38.6 & 
100 & 100 & 100 & 76 & 94.2 & 66.2 & 5.66 & 45.1 & 
100 & 100 & 84 & 2 & 92.6 & 8.86 & 26.3 & 31.9 \\ 

\bottomrule
\end{tabular}%
}
\end{table*}

\vspace{0.5em}
\subsection{Stage-wise Decomposition of Reachability with OPD Metrics (RQ2)}
\label{subsec:rq_stage_decomposition}

To demonstrate the capability of OPD in handling long-horizon tasks, we compute milestone coverage at 25\%, 50\%, 75\%, and 100\% and visualize the resulting reachability curves in Fig.~\ref{fig:rq2_reachability}.
We report the corresponding values for three representative tasks in Tab.~\ref{tab:robotwin_audit_main}.

\textbf{\textit{1) OPD localizes failure stage by separating partial completion from terminal success.}}
On \textit{Blocks Ranking RGB}, most policies reliably reach early stages, with MC@25 percent ranging from 84 to 100, while terminal completion is rare, with MC@100 percent ranging from 0 to 8.
This separation indicates that many rollouts fail late after substantial progress, which binary success cannot resolve.

\textbf{\textit{2) OPD distinguishes qualitatively different ``zero-success'' regimes.}}
On \textit{Blocks Ranking RGB}, $\pi_0$ reaches MC@75 of 40, while OpenVLA-OFT reaches MC@75 of 6, although both have near-zero MC@100.
OPD reveals that $\pi_0$ failures are closer to completion, whereas OpenVLA-OFT tends to fall short earlier.
A similar separation occurs between RDT (MC@75 of 30) and DP (MC@75 of 18), despite both having MC@100 of 0.

\textbf{\textit{3) Late-stage bottlenecks dominate several long-horizon tasks.}}
On \textit{Handover Mic}, multiple policies reach MC@75 above 88, but differ sharply at MC@100: DP drops to 44, while $\pi_0$ reaches 98 and RDT reaches 100.
This pattern indicates that many failures concentrate in the final 75-to-100 stage, suggesting ``last-mile'' precision or stabilization bottlenecks.

\begin{figure}[t!]
    \centering
    \includegraphics[width=\linewidth]{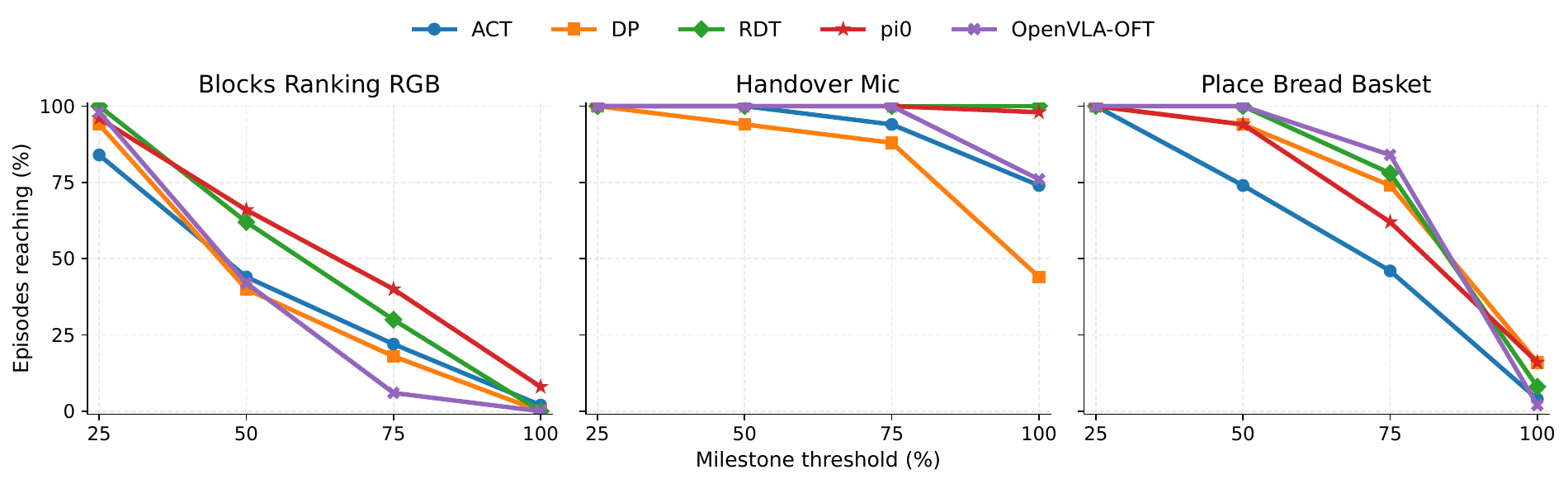}
    \caption{\textbf{Reachability and failure-stage decomposition by milestone coverage.}
For each task, we plot the fraction of rollout episodes that reach milestone thresholds (25/50/75/100\%), revealing where execution progress tends to stall along the horizon.}
    \label{fig:rq2_reachability}
    \vspace{-1.5em}
\end{figure}

\begin{figure}[t!]
    \centering
    \begin{subfigure}[t]{0.49\linewidth}
        \centering
        \includegraphics[width=\linewidth]{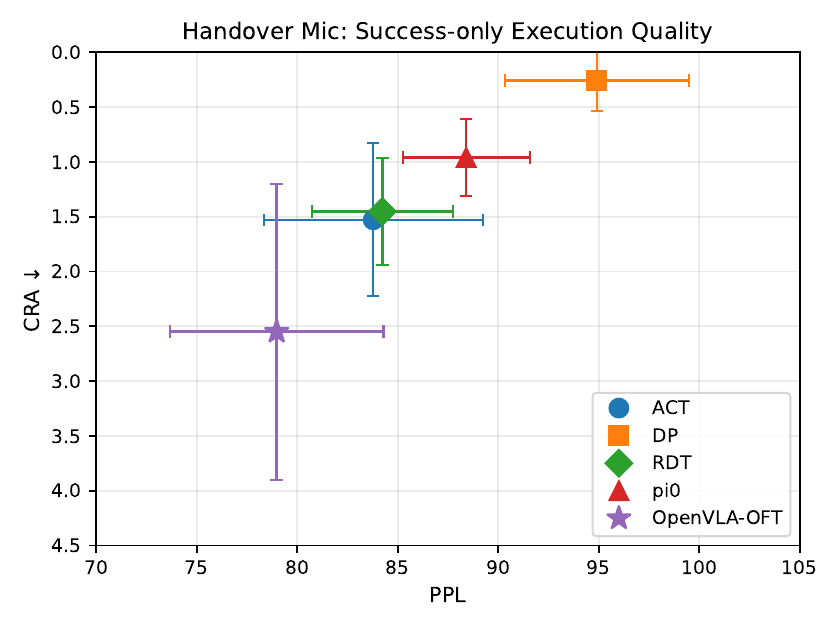}
        \label{fig:rq3_ppl_cra}
    \end{subfigure}
    \begin{subfigure}[t]{0.49\linewidth}
        \centering
        \includegraphics[width=\linewidth]{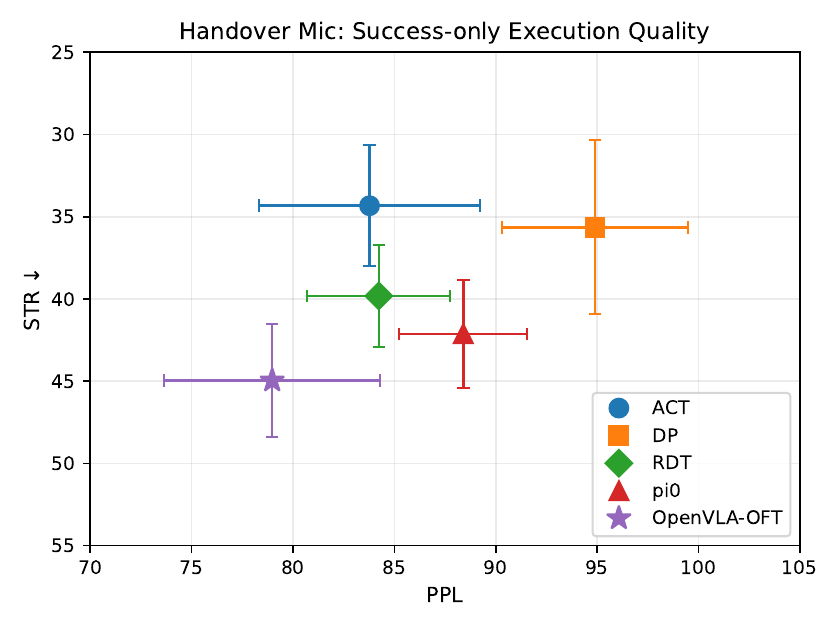}
        \label{fig:rq3_ppl_str}
    \end{subfigure}
    \vspace{-1em}
    \caption{\textbf{Success-only execution quality on Handover Mic.} We report success-conditioned OPD metrics and compare path efficiency, measured by PPL, against regression, measured by CRA, and stagnation, measured by STR, across policy families. Error bars denote standard deviation across successful episodes.}

    \label{fig:rq3_success_quality}
    \vspace{-1.5em}
\end{figure}

\subsection{Execution Quality Conditioned on Success (RQ3)}
\label{subsec:rq_success_quality}

To avoid conflating execution quality with outcome frequency, we evaluate process and diagnosis metrics on successful episodes only and visualize trade-offs in Fig.~\ref{fig:rq3_success_quality}.
We focus on \textit{Handover Mic} and use Tab.~\ref{tab:robotwin_audit_main} for outcome-level context.
We draw three observations.

\textbf{\textit{1) Success does not imply uniformly high-quality execution.}}
Among successful episodes on \textit{Handover Mic}, policies remain separable by OPD metrics.
DP achieves mean PPL of 94.9, exceeding $\pi_0$ at 88.4 and RDT at 84.5.
DP also shows lower mean CRA of 0.26, compared with $\pi_0$ at 0.96 and OpenVLA-OFT at 2.55.
These gaps indicate that even when policies succeed, their trajectories can differ substantially in efficiency and regret.

\textbf{\textit{2) The best success-conditioned quality can reflect a narrow success regime rather than broad reliability.}}
Although DP produces the most efficient and low-regret successful trajectories on \textit{Handover Mic}, its outcome-level success remains lower: MC@100 is 44 for DP, versus 98 for $\pi_0$ and 100 for RDT in Tab.~\ref{tab:robotwin_audit_main}.
This pattern suggests that DP operates in a relatively narrow success regime: when execution stays on track, it tends to succeed with high efficiency and low regret, but once it deviates from that regime, recovery appears weaker and successful completion becomes much less likely.
Read jointly with the outcome-level results, the success-conditioned plot therefore reveals a sharp trade-off between conditional execution quality and overall reliability.


\textbf{\textit{3) Process and diagnosis metrics provide interpretable axes for ``smooth'' versus ``unstable'' successes.}}
PPL differentiates efficient, direct progression from detours and redundancy, CRA captures the degree of backtracking or repeated correction, and STR reflects hesitation or stagnation.
In Fig.~\ref{fig:rq3_success_quality}, OpenVLA-OFT exhibits substantially higher CRA than DP and $\pi_0$ under comparable PPL ranges, consistent with successes that incur larger correction cost.
Overall, Fig.~\ref{fig:rq3_success_quality} supports RQ3 by showing that OPD metrics separate high-quality successes from inefficient or unstable successes beyond binary outcomes.

\begin{figure}[!t]
    \centering
    \includegraphics[width=\linewidth]{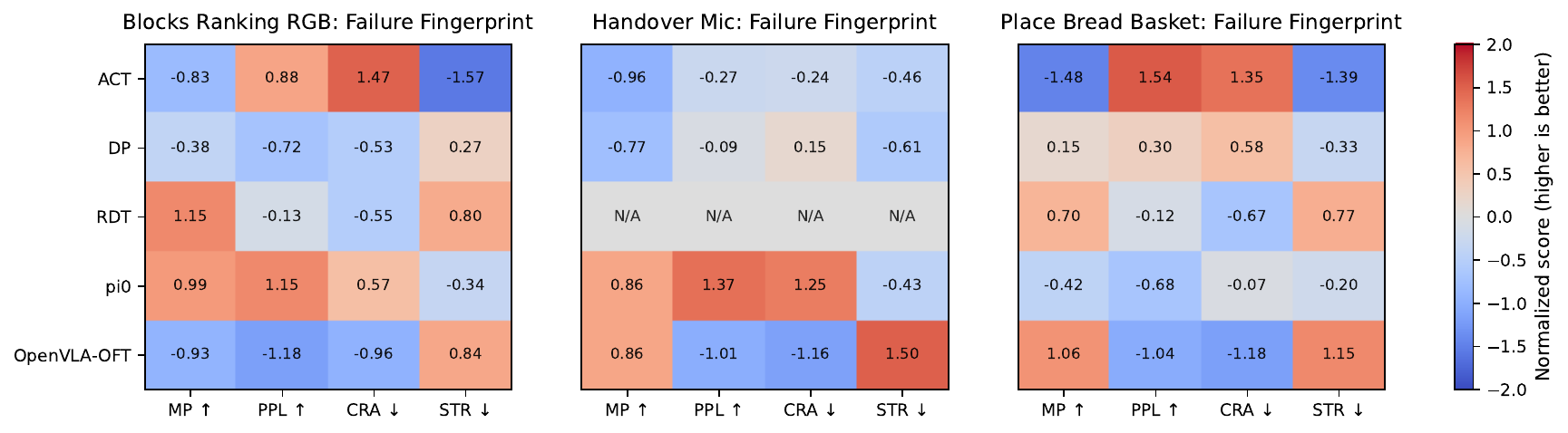}
    \caption{\textbf{Failure-only OPD fingerprints.}
We normalize OPD metrics over failed episodes within each task to highlight policy-specific failure patterns.}
    \label{fig:rq4_failure_fingerprint}
\end{figure}

\subsection{Failure Fingerprints for Mechanistic Diagnosis (RQ4)}
\label{subsec:rq_failure_fingerprint}

To validate the effectiveness of OPD in failure analysis, we aggregate MP, PPL, CRA, and STR over failed episodes and normalize each metric within a task using z-scores across policy families.
For CRA and STR, we flip the sign before normalization so that higher values consistently indicate more desirable behavior.
The resulting failure fingerprints are shown in Fig.~\ref{fig:rq4_failure_fingerprint}.
We draw three observations.

\textbf{\textit{1) Policy families exhibit distinct and task-consistent failure signatures in OPD space.}}
On \textit{Place Bread Basket}, OpenVLA-OFT fails late with mean MP of 92.6 but incurs large mean CRA of 26.3.
ACT failures reach lower mean MP of 73.1 and show higher stagnation, with mean STR of 65.4.
This contrast separates late-stage instability with backtracking from early-stage stagnation, even though both regimes map to the same terminal ``failure'' outcome.

\textbf{\textit{2) Fingerprints separate stagnation-dominant and regret-dominant failure mechanisms.}}
On \textit{Handover Mic}, DP failures show high stagnation with mean STR of 57.2.
OpenVLA-OFT failures reach MP of 94.2 but exhibit low efficiency and higher regret, with mean PPL of 66.2 and mean CRA of 5.66.
The contrast distinguishes stalled interaction from unstable correction loops.

\textbf{\textit{3) Fingerprints suggest interpretable hypotheses for improvement.}}
Stagnation-dominant failures may reflect insufficient end-effector stability or weak contact maintenance, often appearing as hesitation. Regret-dominant failures may reflect limited error absorption or corrective control, leading to repeated backtracking. Efficiency-dominant weaknesses may indicate excessive redundant motion and low progress density. 
These patterns remain visible across tasks in Fig.~\ref{fig:rq4_failure_fingerprint}, highlighting the diagnostic value of PRM-as-a-Judge beyond binary success, although the resulting interpretations remain judge-dependent.

\section{Conclusion}
\label{sec:conclusion}

We introduced \textbf{PRM-as-a-Judge} and the \textbf{OPD} metric system for dense, fine-grained robotic evaluation beyond binary success.
We formalized two axioms for dense evaluation: macro-consistency and micro-resolution.
Under the proposed formulation, potential-based PRM judges satisfy \textit{macro-consistency} by construction, and our experiments on RoboPulse show that trajectory-trained PRMs can achieve strong \textit{micro-resolution} for fine-grained relative progress discrimination.
Applying OPD to long-horizon policy auditing reveals stage-wise reachability, success-conditioned execution quality, and failure fingerprints beyond binary success.
Beyond serving as an evaluation tool, PRM-as-a-Judge provides dense progress signals that may support future training-time diagnosis and improvement.
More broadly, PRMs remain an open and promising design space for dense robotic evaluation. We hope the proposed axioms, benchmark, and OPD analysis protocol provide a concrete target for developing stronger and more physically grounded judges in future work.


\section*{Impact Statement}
This paper aims to advance machine learning research by improving how robotic manipulation policies are evaluated.
Our contributions focus on benchmarking and diagnostic metrics rather than enabling new deployment capabilities.
We do not foresee societal impacts that require specific discussion beyond those commonly associated with robotics and learned control systems.

\section*{Acknowledgements}
This work was supported in part by the National Natural Science Foundation of China under Grant 72434005, Grant 72225011 and Grant 72293575.

\bibliography{example_paper}
\bibliographystyle{icml2026}

\newpage
\appendix
\onecolumn

\section*{Appendix Overview}
This appendix includes the full formalization, theoretical proofs, and detailed descriptions of experimental protocols that were omitted from the main paper for brevity. The sections are organized as follows:

\begin{itemize}
\setlength{\itemsep}{0.5em}
\item \textbf{Appendix~\ref{app:metric_derivation}:} Provides rigorous definitions of the five metrics used within the OPD system (MC, MP, PPL, CRA, and STR). This section includes mathematical proofs about their valid ranges, invariances, and robustness to perturbations in the judging process.
\item \textbf{Appendix~\ref{app:discarded_metrics}:} Describes alternative metrics considered during the design phase, and explains why they were not adopted. It highlights recurring failure modes, such as numerical ill-conditioning, dependence on temporal discretization, and the lack of interpretable macroscopic meaning.
\item \textbf{Appendix~\ref{app:inconsistency_proof}:} Presents a formal proof showing that evaluators based on similarity or relative comparison generally fail to satisfy the axiom of macro-consistency. This section addresses how these methods induce scale drift during temporal resampling.
\item \textbf{Appendix~\ref{app:macro_consistency_proof}:} Provides the proof that the PRM-as-a-Judge framework is structurally macro-consistent under the potential-difference formulation. It demonstrates how PRM defines global progress potential and assigns consistent progress increments across task states.
\item \textbf{Appendices~\ref{app:robopulse_data} --~\ref{prompt}:} These appendices provide detailed descriptions of the experimental setup to ensure reproducibility, including:
\begin{itemize}
    \item Dataset composition and curation protocols for the RoboPulse benchmark;
    \item Methodologies employed for evaluating robustness and the corresponding experimental results;
    \item Additional quantitative results obtained from the RoboTwin environments;
    \item Guidelines for visualizing failure fingerprints across different models;
    \item A complete set of system prompts used for evaluating both baseline models and the proposed judge.
\end{itemize}
\end{itemize}

\section{OPD Metrics: Formal Definitions, Properties, and Robustness Proofs}
\label{app:metric_derivation}

Sec.~\ref{subsec:opd} defines OPD at a high level. This appendix formalizes the proposed metrics as functionals of a progress-potential trajectory $(\Phi_t)_{t=0}^T$, and records elementary properties used in the main text (range, monotonicity, and stability with respect to bounded judge perturbations).

\subsection{Preliminaries and Notation}
Let $\mathcal{X}$ denote the space of observed states available to the judge, and let a rollout trajectory be
$
\tau=(x_0,x_1,\dots,x_T),\ x_t\in\mathcal{X}.
$
Let $\Phi:\mathcal{X}\to[0,1]$ be the task-conditioned progress potential induced by a judge (Sec.~\ref{subsec:prm}).
For shorthand, define $\Phi_t\coloneqq \Phi(x_t)$ and one-step increments $d_t\coloneqq \Phi_t-\Phi_{t-1}$ for $t\ge1$.
We further define:
\begin{align}
\textbf{Net progress:}\quad & \mathrm{NP}(\tau)\coloneqq \Phi_T-\Phi_0,\\
\textbf{Total variation:}\quad & \mathrm{TV}(\tau)\coloneqq \sum_{t=1}^{T}|d_t|,\\
\textbf{Running maximum:}\quad & M_t\coloneqq \max_{0\le k\le t}\Phi_k.
\end{align}

\begin{lemma}[Variation dominates endpoint displacement]
\label{lem:tv_ge_np}
For any real sequence $(\Phi_t)_{t=0}^T$,
\[
|\Phi_T-\Phi_0|\le \sum_{t=1}^T|\Phi_t-\Phi_{t-1}|=\mathrm{TV}(\tau).
\]
\end{lemma}
\begin{proof}
By the triangle inequality,
$
\left|\sum_{t=1}^T(\Phi_t-\Phi_{t-1})\right|
\le \sum_{t=1}^T|\Phi_t-\Phi_{t-1}|.
$
\end{proof}

\subsection{Outcome-Level Metrics}
\label{app:opd_outcome_metrics}

\subsubsection{Milestone Coverage (MC)}
\label{app:mc_def}

\paragraph{Definition.}
Fix a milestone set $\mathcal{Q}=\{0,\frac{1}{K},\dots,1\}$ (we use $K=4$ in the paper). Define
\begin{equation}
\mathrm{MC}(\tau)\coloneqq \max\left\{q\in\mathcal{Q}\ \middle|\ \exists t\in\{0,\dots,T\}\ \text{s.t.}\ \Phi_t\ge q\right\}.
\end{equation}

\begin{lemma}[Range and monotonicity]
\label{lem:mc_basic}
$\mathrm{MC}(\tau)\in\mathcal{Q}$.
Moreover, if $\Phi_t\le \Psi_t$ for all $t$, then $\mathrm{MC}_\Phi(\tau)\le \mathrm{MC}_\Psi(\tau)$.
\end{lemma}
\begin{proof}
The value lies in $\mathcal{Q}$ since the maximum is taken over $\mathcal{Q}$.
If $\Phi_t\le\Psi_t$ pointwise, then every milestone reached by $(\Phi_t)$ is also reached by $(\Psi_t)$, hence the maximal reached milestone cannot decrease.
\end{proof}

\begin{lemma}[Stability under bounded judge error]
\label{lem:mc_stability}
Let $\widehat{\Phi}_t$ satisfy $|\widehat{\Phi}_t-\Phi_t|\le \sigma$ for all $t$.
If for every $t$ and every internal boundary $b\in\{1/K,\dots,(K-1)/K\}$ we have $|\Phi_t-b|>\sigma$,
then $\mathrm{MC}(\widehat{\tau})=\mathrm{MC}(\tau)$.
\end{lemma}
\begin{proof}
Under the margin assumption, $\Phi_t$ and $\widehat{\Phi}_t$ fall into the same quantization bin induced by $\mathcal{Q}$ for every $t$.
Therefore the set of milestones achieved is unchanged, and so is its maximum.
\end{proof}

\subsubsection{Max Progress (MP)}
\label{app:mp_def}

\paragraph{Definition.}
\begin{equation}
\mathrm{MP}(\tau)\coloneqq \max_{0\le t\le T}\Phi_t.
\end{equation}

\begin{lemma}[Range and refinement monotonicity]
\label{lem:mp_basic}
$\mathrm{MP}(\tau)\in[0,1]$.
If $\tau'$ is obtained from $\tau$ by inserting intermediate states, then $\mathrm{MP}(\tau')\ge \mathrm{MP}(\tau)$, with equality if no inserted state exceeds the original maximum.
\end{lemma}
\begin{proof}
Range follows from $\Phi_t\in[0,1]$.
Refinement inserts additional indices into the maximization set, so the maximum cannot decrease.
\end{proof}

\subsection{Process-Level Metric}
\label{app:opd_process_metrics}

\subsubsection{Path-weighted Progress Length (PPL)}
\label{app:ppl_derivation_strict}

\paragraph{Definition.}
We define
\begin{equation}
\mathrm{PPL}(\tau)
\coloneqq
\Phi_T\cdot
\frac{[\Phi_T-\Phi_0]_+}{\mathrm{TV}(\tau)+\delta},
\qquad [x]_+\coloneqq \max(x,0),
\label{eq:ppl_strict}
\end{equation}
where $\delta>0$ is a small numerical constant (e.g., $10^{-8}$) used only to avoid division by zero when $\mathrm{TV}(\tau)=0$.

\begin{theorem}[Range and tight characterization]
\label{thm:ppl_range}
For any trajectory $\tau$, $\mathrm{PPL}(\tau)\in[0,1]$.
Moreover, if $\Phi_T\ge \Phi_0$ and $\delta=0$, then
$\frac{\Phi_T-\Phi_0}{\mathrm{TV}(\tau)}\in[0,1]$ with equality $=1$ if and only if $(\Phi_t)_{t=0}^T$ is monotone non-decreasing.
\end{theorem}
\begin{proof}
Nonnegativity is immediate. For the upper bound (with $\delta=0$), Lemma~\ref{lem:tv_ge_np} gives
$0\le \Phi_T-\Phi_0\le \mathrm{TV}(\tau)$ when $\Phi_T\ge\Phi_0$, hence the ratio is at most $1$; multiplying by $\Phi_T\le1$ preserves the bound.
The characterization follows from the standard fact that total variation equals endpoint displacement iff all increments have the same sign, here $d_t\ge0$ for all $t$.
\end{proof}

\begin{remark}[Why the completion gate is included]
The efficiency ratio $\frac{[\Phi_T-\Phi_0]_+}{\mathrm{TV}(\tau)+\delta}$ alone can be close to $1$ even when the final progress is small.
For example, take $\Phi_0=0$, $\Phi_1=\epsilon$, and $\Phi_t=\epsilon$ for $t\ge1$, then $\mathrm{TV}(\tau)=\epsilon$ and the ratio equals $\epsilon/(\epsilon+\delta)\approx 1$ whenever $\epsilon\gg \delta$.
The multiplicative factor $\Phi_T$ prevents such early-stopping trajectories from receiving a large score, since $\mathrm{PPL}(\tau)\le \Phi_T=\epsilon$.
\end{remark}

\subsection{Diagnosis-Level Metrics}
\label{app:opd_diag_metrics}

\subsubsection{Cumulative Regret Area (CRA)}
\label{app:cra_derivation_strict}

\paragraph{Definition.}
Define
\begin{equation}
R_t\coloneqq M_t-\Phi_t,
\qquad
\mathrm{CRA}(\tau)\coloneqq \frac{1}{T+1}\sum_{t=0}^{T}R_t.
\label{eq:cra_def_strict}
\end{equation}

\begin{theorem}[Range and zero-regret characterization]
\label{thm:cra_range}
$\mathrm{CRA}(\tau)\in[0,1]$.
Moreover, $\mathrm{CRA}(\tau)=0$ if and only if $(\Phi_t)_{t=0}^T$ is monotone non-decreasing.
\end{theorem}
\begin{proof}
Since $M_t,\Phi_t\in[0,1]$ and $M_t\ge\Phi_t$, we have $0\le R_t\le1$, hence $\mathrm{CRA}(\tau)\in[0,1]$.
If $\Phi_t$ is monotone non-decreasing then $M_t=\Phi_t$ and all $R_t=0$.
Conversely, if $\mathrm{CRA}(\tau)=0$ then $R_t=0$ for all $t$, implying $\Phi_t=M_t\ge \Phi_{t-1}$ and thus monotonicity.
\end{proof}

\begin{remark}[CRA captures persistence beyond local regressions]
A related alternative is the local regression mass
$
\mathrm{RR}(\tau)\coloneqq \sum_{t=1}^T[\Phi_{t-1}-\Phi_t]_+,
$
which aggregates instantaneous decreases. CRA depends on the running maximum and thus reflects how long the trajectory stays below its best-so-far value.
\end{remark}

\begin{lemma}[CRA separates persistent drops from quick recovery]
\label{lem:cra_persistence}
There exist two trajectories with the same $\mathrm{RR}$ but different $\mathrm{CRA}$.
\end{lemma}
\begin{proof}
Let $T=4$ and consider
$
\Phi^{(a)}=(0,1,0,0,0)
$
and
$
\Phi^{(b)}=(0,1,0,1,1).
$
Both have $\mathrm{RR}=1$ (a single drop from $1$ to $0$).
However, for $\Phi^{(a)}$ the regret remains $1$ for three steps, yielding $\mathrm{CRA}(\tau^{(a)})=3/5$,
whereas $\Phi^{(b)}$ recovers at $t=3$ so $\mathrm{CRA}(\tau^{(b)})=1/5$.
\end{proof}

\subsubsection{Stagnation Ratio (STR)}
\label{app:str_derivation_strict}

\paragraph{Definition.}
Given a threshold $\epsilon>0$,
\begin{equation}
\mathrm{STR}(\tau)\coloneqq \frac{1}{T}\sum_{t=1}^{T}\mathbb{I}\big(|d_t|<\epsilon\big).
\label{eq:str_def_strict}
\end{equation}

\begin{lemma}[Range]
$\mathrm{STR}(\tau)\in[0,1]$.
\end{lemma}
\begin{proof}
It is an average of Bernoulli indicators.
\end{proof}

\paragraph{Calibrating $\epsilon$ to judge noise.}
Assume a static scene where $\widehat{\Phi}_t=\Phi^\star+\xi_t$ with i.i.d.\ noise $\xi_t\sim\mathcal{N}(0,\sigma^2)$.
Then $\widehat{d}_t=\xi_t-\xi_{t-1}\sim\mathcal{N}(0,2\sigma^2)$.
To target a two-sided tail probability $\alpha$ (e.g., $\alpha=1\%$), choose
\begin{equation}
\epsilon=\sqrt{2}\sigma\cdot \Phi_{\mathrm{std}}^{-1}\!\left(1-\frac{\alpha}{2}\right),
\label{eq:epsilon_calib}
\end{equation}
where $\Phi_{\mathrm{std}}^{-1}$ is the standard normal quantile.

\newpage
\section{Metric Design Notes and Discarded Alternatives}
\label{app:discarded_metrics}

Appendix~\ref{app:metric_derivation} gives the final OPD metrics and their basic properties.
This appendix records alternative candidates considered during metric design and explains why they were not adopted.
The emphasis is on concrete failure modes that repeatedly appeared in practice, including numerical ill-conditioning, dependence on temporal discretization or termination conventions, and loss of an interpretable macroscopic meaning.
Our goal is to make the metric design process transparent and to share insights gained from these failed attempts.
We hope that documenting these explorations and their limitations can provide useful guidance to the community and help inform future work on progress-aware evaluation.

\subsection{Setup and shared primitives}
\label{app:discarded_setup}

We reuse the notation in Appendix~\ref{app:metric_derivation} and introduce a local increment notation
$\Delta\Phi_t$ for convenience in the discarded metrics below.
A rollout trajectory is $\tau=\{x_t\}_{t=0}^{T}$ with $x_t\in\mathcal{X}$, and progress is represented by a potential $\Phi:\mathcal{X}\to[0,1]$.
We write $\Phi_t\coloneqq \Phi(x_t)$ and $\Delta\Phi_t\coloneqq \Phi_{t+1}-\Phi_t$ for $t=0,\dots,T-1$.
We also use the cumulative absolute variation
\begin{equation}
\Delta\Phi_{\mathrm{abs}}(\tau)\coloneqq \sum_{t=0}^{T-1}\lvert \Delta\Phi_t\rvert.
\label{eq:abs_progress_appB}
\end{equation}
All candidates below are functionals of the scalar sequence $\{\Phi_t\}_{t=0}^{T}$.

\subsection{Overview of discarded candidates}
Tab.~\ref{tab:discarded_metrics} summarizes the main alternatives and the dominant reason each one was excluded.

\begin{table*}[t]
\centering
\renewcommand{\arraystretch}{1.45}
\caption{\textbf{Discarded progress based metrics.}
All candidates are defined on the same potential $\Phi$ for comparison.
For each metric we list a formal definition, the intended signal, and the primary failure mode that motivated exclusion.}
\label{tab:discarded_metrics}
\resizebox{0.99\textwidth}{!}{%
\begin{tabular}{p{1.6cm} p{8cm} p{7cm} p{9.3cm}}
\toprule
\textbf{Metric} &
\textbf{Formal definition} &
\textbf{Intended signal} &
\textbf{Failure mode} \\
\midrule

 PPE &
$\displaystyle
\mathrm{PPE}(\tau)=\frac{1}{\sum_{t=0}^{T-1}\lvert \Phi_{t+1}-\Phi_t\rvert+\varepsilon}
$ &
Efficiency in progress space by penalizing back and forth motion. &
Ill conditioned when progress variation is small, producing inflated scores for stalled behavior.
It is completion agnostic and penalizes physically necessary corrections. \\

\midrule

 PTI &
$\displaystyle
\mathrm{PTI}(\tau)=\frac{1}{T+1}\sum_{t=0}^{T}\Phi_t
$ &
Average maintained progress, intended to reward early and sustained advancement. &
Overweights timing of progress and penalizes valid late completion patterns that arise in bottleneck tasks.
It can prefer incomplete trajectories with moderate early progress over successful but late trajectories. \\

\midrule

 EAD &
$\displaystyle
\mathrm{EAD}(\tau)=\frac{1}{T}\sum_{t=0}^{T-1}\mathbb{I}\!\left[\Delta\Phi_t>\epsilon\right]
$ &
Frequency of noticeable positive progress steps, intended to separate purposeful motion from dithering. &
Strong dependence on control frequency and temporal discretization.
It encourages bursty behavior and undervalues smooth continuous progress. \\

\midrule

 PJ &
$\displaystyle
\mathrm{PJ}(\tau)=\frac{1}{T-2}\sum_{t=1}^{T-2}\lvert \Delta\Phi_{t+1}-\Delta\Phi_t\rvert
$ &
Smoothness of progress evolution, intended to penalize jerkiness. &
In contact rich manipulation, valid progress can occur through discrete contact events.
The metric is sensitive to judge noise and can penalize correct executions. \\

\midrule

 CS &
$\displaystyle
\mathrm{CS}_K(\tau)=\mathrm{Var}\!\left(\{\Phi_{T-K},\dots,\Phi_T\}\right)
$ &
Terminal stability, intended to detect oscillatory behavior near termination. &
Dependence on termination conventions and episode length.
Its value can change substantially under different episode truncation rules. \\

\midrule

 RR &
$\displaystyle
\mathrm{RR}(\tau)=\sum_{t=1}^{T}\max\!\left(0,\Phi_{t-1}-\Phi_t\right)
$ &
Accumulated local regressions, intended to quantify backward motion. &
Ignores temporal persistence and severity relative to the best achieved progress.
It cannot separate early exploration from catastrophic late stage drops. \\

\midrule

 GRDTW &
$\displaystyle
\mathrm{GRDTW}(\tau)=\max_k\left(1-\frac{\mathrm{DTW}(\{\Phi_t\},\{\Phi^{(k)}_t\})}{Z}\right)
$ &
Similarity of progress evolution to expert demonstrations, intended to capture human likeness. &
Matches one dimensional progress profiles rather than physical behavior.
It introduces style bias and depends on the availability of demonstrations. \\

\bottomrule
\end{tabular}%
}
\end{table*}

\subsection{Progress Path Efficiency}
\label{app:ppe}

\paragraph{Definition.}
Progress Path Efficiency is defined as
\begin{equation}
\mathrm{PPE}(\tau)\coloneqq \frac{1}{\Delta\Phi_{\mathrm{abs}}(\tau)+\varepsilon},
\label{eq:ppe_appB}
\end{equation}
where $\varepsilon>0$ is a small constant.

\paragraph{Failure mode.}
The metric is numerically ill conditioned when $\Delta\Phi_{\mathrm{abs}}(\tau)$ is small.

\begin{proposition}[Ill conditioning]
\label{prop:ppe_ill}
For any $\varepsilon>0$ and any $c\in[0,1]$, consider a constant progress sequence $\Phi_t=c$ for all $t$.
Then $\Delta\Phi_{\mathrm{abs}}(\tau)=0$ and $\mathrm{PPE}(\tau)=1/\varepsilon$.
\end{proposition}
\begin{proof}
Immediate from the definition.
\end{proof}

This behavior assigns the largest scores to trajectories with essentially no progress variation, which can correspond to stalling or inaction.
It also makes the score depend mainly on the arbitrary constant $\varepsilon$.
In contact rich tasks, small corrective motions that are necessary for success increase $\Delta\Phi_{\mathrm{abs}}$ and are penalized, even when the physical behavior is correct.

\subsection{Progress Time Integral}
\label{app:pti}

\paragraph{Definition.}
Progress Time Integral averages the progress over the execution:
\begin{equation}
\mathrm{PTI}(\tau)\coloneqq \frac{1}{T+1}\sum_{t=0}^{T}\Phi_t.
\label{eq:pti_appB}
\end{equation}

\paragraph{Failure mode.}
PTI prefers early progress regardless of eventual solvability.

\begin{proposition}[Preference for incomplete early progress]
\label{prop:pti_counter}
There exist two trajectories $\tau^{\mathrm{late}}$ and $\tau^{\mathrm{early}}$ such that $\Phi_T^{\mathrm{late}}=1$ and $\Phi_T^{\mathrm{early}}<1$ but $\mathrm{PTI}(\tau^{\mathrm{early}})>\mathrm{PTI}(\tau^{\mathrm{late}})$.
\end{proposition}
\begin{proof}
Fix any $T\ge 4$.
Let $\Phi_t^{\mathrm{late}}=0$ for $t=0,\dots,T-1$ and $\Phi_T^{\mathrm{late}}=1$.
Then $\mathrm{PTI}(\tau^{\mathrm{late}})=1/(T+1)$.
Let $\Phi_t^{\mathrm{early}}=1/2$ for all $t$, so $\Phi_T^{\mathrm{early}}=1/2$ and $\mathrm{PTI}(\tau^{\mathrm{early}})=1/2$.
Hence $\mathrm{PTI}(\tau^{\mathrm{early}})>\mathrm{PTI}(\tau^{\mathrm{late}})$.
\end{proof}

Many manipulation tasks exhibit bottlenecks where progress remains low until a key contact event succeeds.
PTI systematically penalizes such trajectories and can rank incomplete behaviors above successful ones.

\subsection{Effective Action Density}
\label{app:ead}

\paragraph{Definition.}
Effective Action Density counts the fraction of steps with increment above a threshold:
\begin{equation}
\mathrm{EAD}(\tau)\coloneqq \frac{1}{T}\sum_{t=0}^{T-1}\mathbb{I}\!\left[\Delta\Phi_t>\epsilon\right],
\label{eq:ead_appB}
\end{equation}
where $\epsilon>0$ is fixed.

\paragraph{Failure mode.}
EAD is not invariant to temporal discretization.

\begin{proposition}[Discretization dependence]
\label{prop:ead_discret}
Fix any $\epsilon\in(0,1)$.
There exist two temporal discretizations of the same monotone progress curve that yield different EAD values.
\end{proposition}
\begin{proof}
Consider a monotone sequence that increases from $0$ to $1$ linearly.
For a discretization with $T_1$ steps, the increment is $1/T_1$ at each step, so EAD equals $1$ if $1/T_1>\epsilon$ and equals $0$ if $1/T_1\le \epsilon$.
Choose $T_1<1/\epsilon$ and $T_2\ge 1/\epsilon$.
Then the same underlying monotone progress has EAD equal to $1$ under the first discretization and equal to $0$ under the second.
\end{proof}

As a result, EAD depends on the control rate and can be gamed by changing step granularity.
It also favors bursty increments over smooth progress.

\subsection{Progress Jerkiness}
\label{app:pj}

\paragraph{Definition.}
Progress Jerkiness penalizes variation of increments:
\begin{equation}
\mathrm{PJ}(\tau)\coloneqq \frac{1}{T-2}\sum_{t=1}^{T-2}\lvert \Delta\Phi_{t+1}-\Delta\Phi_t\rvert.
\label{eq:pj_appB}
\end{equation}

\paragraph{Failure mode.}
Many correct manipulation behaviors exhibit discrete contact events that induce abrupt progress changes.
PJ treats such events as undesirable and is sensitive to judge noise.

\begin{proposition}[Penalty on discrete progress events]
\label{prop:pj_step}
Let $\Phi_t=0$ for $t=0,\dots,T-1$ and $\Phi_T=1$.
Then $\mathrm{PJ}(\tau)$ is bounded below by a positive constant independent of $T$.
\end{proposition}
\begin{proof}
All increments are zero except the last increment which equals $1$.
Hence at the penultimate index, the difference of consecutive increments has magnitude $1$, so the average in the definition is at least $1/(T-2)$.
\end{proof}

This behavior is undesirable because it penalizes sharp yet physically valid transitions that correspond to successful completion.

\subsection{Convergence Stability}
\label{app:cs}

\paragraph{Definition.}
Convergence Stability measures the variance of the final window:
\begin{equation}
\mathrm{CS}_K(\tau)\coloneqq \mathrm{Var}\!\left(\{\Phi_{T-K},\dots,\Phi_T\}\right),
\label{eq:cs_appB}
\end{equation}
where $K$ is fixed.

\paragraph{Failure mode.}
CS depends on termination conventions and whether an episode is extended after success.

\begin{proposition}[Dependence on episode extension]
\label{prop:cs_extend}
Let $\tau$ be any trajectory with terminal potential $\Phi_T=c$.
Let $\tau'$ be obtained by appending $L$ additional steps that keep the state unchanged so that $\Phi_t=c$ on the appended steps.
Then $\mathrm{CS}_K(\tau')$ can differ from $\mathrm{CS}_K(\tau)$ for the same fixed $K$.
\end{proposition}
\begin{proof}
When $L$ changes, the final window $\{\Phi_{T'-K},\dots,\Phi_{T'}\}$ contains a different mixture of pre terminal values and constant values.
The sample variance over this window therefore changes in general.
\end{proof}

In practice, some benchmarks terminate immediately upon success while others allow post success settling.
A metric whose value changes under such protocol choices is difficult to interpret consistently.

\subsection{Regression Rate}
\label{app:rr}

\paragraph{Definition.}
Regression Rate accumulates all local decreases:
\begin{equation}
\mathrm{RR}(\tau)\coloneqq \sum_{t=1}^{T}\max\!\left(0,\Phi_{t-1}-\Phi_t\right).
\label{eq:rr_appB}
\end{equation}

\paragraph{Failure mode.}
RR does not capture persistence relative to the best achieved progress, so it cannot distinguish catastrophic late drops from early recoverable exploration.

\begin{proposition}[RR cannot separate persistence]
\label{prop:rr_not_expressive}
There exist two trajectories with equal RR but different severity of late stage failure.
\end{proposition}
\begin{proof}
Let $T=4$.
Consider $\Phi^{(a)}=(0,1,0,0,0)$ and $\Phi^{(b)}=(0,1,0,1,1)$.
Both sequences have exactly one local decrease of magnitude $1$, hence $\mathrm{RR}=1$ for both.
However, $\Phi^{(a)}$ stays far below its best achieved value for the remainder of the episode, while $\Phi^{(b)}$ recovers quickly.
\end{proof}

This motivates CRA in Appendix~\ref{app:metric_derivation}, which measures regret to the running maximum and reflects both magnitude and duration of falling behind.

\subsection{Golden Reference DTW Alignment}
\label{app:grdtw}

\paragraph{Definition.}
Golden Reference DTW Alignment compares the progress sequence to expert progress sequences using dynamic time warping:
\begin{equation}
\mathrm{GRDTW}(\tau)\coloneqq
\max_{k\in\{1,\dots,K\}}
\left(
1-\frac{\mathrm{DTW}\!\left(\{\Phi_t\}_{t=0}^{T},\{\Phi_t^{(k)}\}_{t=0}^{T_k}\right)}{Z}
\right),
\label{eq:grdtw_appB}
\end{equation}
where $Z$ is a normalization constant.

\paragraph{Failure mode.}
The metric matches a one dimensional progress profile rather than physical execution.

\begin{proposition}[Indistinguishability under identical progress traces]
\label{prop:grdtw_indist}
If two trajectories $\tau_1$ and $\tau_2$ satisfy $\Phi(x_t^{(1)})=\Phi(x_t^{(2)})$ for all $t$ after temporal alignment, then $\mathrm{GRDTW}(\tau_1)=\mathrm{GRDTW}(\tau_2)$ for any fixed expert set.
\end{proposition}
\begin{proof}
DTW depends only on the scalar sequences being compared.
If the aligned progress sequences are identical, their DTW distance to any expert progress sequence is identical, and the maximization over experts yields the same value.
\end{proof}

In contact rich manipulation, policies can share similar progress traces while exhibiting different contact patterns, collision behavior, or safety properties.
A metric that only aligns progress profiles risks importing style bias and depends on demonstration availability, which limits its generality.

\newpage
\section{On the Inconsistency of Similarity- and Relative-Comparison Evaluators}
\label{app:inconsistency_proof}

This appendix explains why evaluators built from goal-image similarity or pairwise relative comparison generally fail to satisfy the macro-consistency axiom in Sec.~\ref{subsec:theory}.
The key point is structural: macro-consistency requires progress increments to be path-independent and therefore representable as differences of a globally defined potential on task states.
Appearance-based rules, which are anchored in observation-level correspondence, typically violate this requirement either by being ill-defined on task-equivalent states or by being non-additive across triples, which induces scale drift under temporal resampling.

\subsection{Preliminaries and Notation}
\label{app:inconsistency_prelim}

Let $\mathcal{S}$ be the state space of the robot and environment system, and let $\mathcal{O}$ be the observation space.
Observations are generated by an observation function
\begin{equation}
g:\mathcal{S}\to\mathcal{O},\qquad o=g(s),
\end{equation}
which captures viewpoint, occlusion, sensor noise, and rendering.
A task instance is specified by a context variable $c\in\mathcal{C}$, such as a language instruction, a goal specification, or a task identifier.
A rollout trajectory is a sequence $\tau=\{s_0,s_1,\dots,s_T\}$ with $s_t\in\mathcal{S}$.

\paragraph{Potential-based progress.}
A progress potential is a scalar function
\begin{equation}
\Phi:\mathcal{S}\times\mathcal{C}\to\mathbb{R}.
\end{equation}
It induces a progress increment between any two states,
\begin{equation}
\Delta_{\Phi}(s_i,s_j\mid c)\coloneqq \Phi(s_j\mid c)-\Phi(s_i\mid c).
\end{equation}
Such increments satisfy additivity on every triple,
\begin{equation}
\Delta_{\Phi}(s_i,s_k\mid c)=\Delta_{\Phi}(s_i,s_j\mid c)+\Delta_{\Phi}(s_j,s_k\mid c),
\qquad \forall s_i,s_j,s_k\in\mathcal{S}.
\label{eq:additivity_potential_appC}
\end{equation}
Eq.~\eqref{eq:additivity_potential_appC} is the algebraic form of macro-consistency used throughout this work.

\paragraph{Appearance-based evaluators.}
We consider two commonly used forms that operate on observations:
\begin{enumerate}[leftmargin=1.4em]
\item \textbf{Goal-similarity evaluator:}
\begin{equation}
E_{\mathrm{goal}}(s\mid c)\coloneqq f\!\left(g(s),o_{\mathrm{ref}}(c)\right),
\label{eq:goal_sim_eval_appC}
\end{equation}
where $o_{\mathrm{ref}}(c)\in\mathcal{O}$ is a reference observation, and $f$ is a similarity or compatibility score.

\item \textbf{Pairwise relative evaluator:}
\begin{equation}
E_{\mathrm{pair}}(s_i,s_j\mid c)\coloneqq h\!\left(g(s_i),g(s_j),c\right),
\label{eq:pair_eval_appC}
\end{equation}
where $h$ returns a real-valued score for the ordered pair $(s_i,s_j)$.
\end{enumerate}

These formulations cover common discriminative and contrastive scoring rules that are anchored in observation-level correspondence rather than a globally defined progress coordinate on $\mathcal{S}$.

\subsection{A Characterization of Macro-Consistency}
\label{app:inconsistency_condition}

Macro-consistency requires that local increments can be accumulated without dependence on the segmentation of a trajectory.
The following lemma states that this is equivalent to the existence of a global potential.

\begin{lemma}[Additivity implies a potential-difference form]
\label{lem:additivity_implies_potential_appC}
Fix a context $c\in\mathcal{C}$ and suppose an evaluator $E(\cdot,\cdot\mid c):\mathcal{S}\times\mathcal{S}\to\mathbb{R}$ satisfies
\begin{equation}
E(s_i,s_k\mid c)=E(s_i,s_j\mid c)+E(s_j,s_k\mid c),
\qquad \forall s_i,s_j,s_k\in\mathcal{S}.
\label{eq:cocycle_appC}
\end{equation}
Then there exists a function $\Phi(\cdot\mid c):\mathcal{S}\to\mathbb{R}$, unique up to an additive constant, such that
\begin{equation}
E(s_i,s_j\mid c)=\Phi(s_j\mid c)-\Phi(s_i\mid c),
\qquad \forall s_i,s_j\in\mathcal{S}.
\label{eq:potential_difference_appC}
\end{equation}
\end{lemma}

\begin{proof}
Fix an anchor state $s_\star\in\mathcal{S}$ and define $\Phi(s\mid c)\coloneqq E(s_\star,s\mid c)$.
Applying Eq.~\eqref{eq:cocycle_appC} to the triple $(s_\star,s_i,s_j)$ gives
\[
E(s_\star,s_j\mid c)=E(s_\star,s_i\mid c)+E(s_i,s_j\mid c).
\]
Rearranging yields
\[
E(s_i,s_j\mid c)=E(s_\star,s_j\mid c)-E(s_\star,s_i\mid c)=\Phi(s_j\mid c)-\Phi(s_i\mid c).
\]
If a different anchor is used, the resulting potential differs by an additive constant.
\end{proof}

Lemma~\ref{lem:additivity_implies_potential_appC} implies that a macro-consistent evaluator is necessarily path-independent.
This immediately yields temporal refinement invariance for trajectory accumulation.

\begin{corollary}[Telescoping and refinement invariance]
\label{cor:telescoping_appC}
Let $E(\cdot,\cdot\mid c)$ satisfy Eq.~\eqref{eq:cocycle_appC} and let $\tau=\{s_t\}_{t=0}^{T}$.
Then
\begin{equation}
\sum_{t=1}^{T}E(s_{t-1},s_t\mid c)=\Phi(s_T\mid c)-\Phi(s_0\mid c),
\label{eq:telescoping_appC}
\end{equation}
and the left-hand side is invariant under inserting intermediate states along the trajectory.
\end{corollary}

\begin{proof}
Substitute Eq.~\eqref{eq:potential_difference_appC} into the sum to obtain telescoping.
Inserting intermediate states does not change $\Phi(s_T\mid c)-\Phi(s_0\mid c)$.
\end{proof}

\subsection{Why Appearance-Based Evaluators Fail in General}
\label{app:inconsistency_impossibility}

We now show two distinct mechanisms of failure.
Goal-similarity evaluators can be additive only in an observation-dependent sense, which is not well-defined on task-equivalent states.
Pairwise relative evaluators typically violate Eq.~\eqref{eq:cocycle_appC}, producing path-dependence and scale drift.

\subsubsection{Goal-similarity is not well-defined on task-equivalent states}
\label{app:goal_similarity_inconsistency}

A common construction is to form an increment by differencing a scalar score,
\begin{equation}
\Delta_E(s_i,s_j\mid c)\coloneqq E_{\mathrm{goal}}(s_j\mid c)-E_{\mathrm{goal}}(s_i\mid c).
\label{eq:delta_goal_sim_appC}
\end{equation}
Eq.~\eqref{eq:delta_goal_sim_appC} is additive as an identity of real numbers.
However, macro-consistency in robotic evaluation is a statement about progress on the task state space.
For this, the underlying scalar must be a meaningful progress coordinate on $\mathcal{S}$, not merely a similarity to a single reference observation.

To formalize the required invariance, we introduce task-equivalence.

\begin{definition}[Task-equivalence]
\label{def:task_equiv_appC}
Fix a context $c$.
Define a relation $\sim_c$ on $\mathcal{S}$ by $s\sim_c s'$ if $s$ and $s'$ are indistinguishable with respect to task semantics under $c$.
In particular, $s\sim_c s'$ when both states satisfy the same success conditions and agree on all task-relevant attributes.
\end{definition}

A progress potential on task space should respect task-equivalence.

\begin{definition}[Task-state invariance]
\label{def:task_state_invariance_appC}
A function $\Psi(\cdot\mid c):\mathcal{S}\to\mathbb{R}$ is invariant to task-equivalence if
\begin{equation}
s\sim_c s' \;\Rightarrow\; \Psi(s\mid c)=\Psi(s'\mid c).
\label{eq:potential_respects_equiv_appC}
\end{equation}
\end{definition}

Goal-similarity generally violates this invariance because $g$ is not injective on task-equivalence classes and $f$ distinguishes different observations.

\begin{proposition}[Goal-similarity violates task-equivalence in general]
\label{prop:goal_similarity_equiv_violation_appC}
Fix a context $c$.
Assume there exist $s_a,s_b\in\mathcal{S}$ such that $s_a\sim_c s_b$ and $g(s_a)\neq g(s_b)$.
Assume further that the map $o\mapsto f(o,o_{\mathrm{ref}}(c))$ is not constant on $\mathcal{O}$.
Then, for generic choices of the reference observation $o_{\mathrm{ref}}(c)$,
\begin{equation}
E_{\mathrm{goal}}(s_a\mid c)\neq E_{\mathrm{goal}}(s_b\mid c),
\end{equation}
so $E_{\mathrm{goal}}(\cdot\mid c)$ cannot satisfy Eq.~\eqref{eq:potential_respects_equiv_appC}.
\end{proposition}

\begin{proof}
Since $g(s_a)\neq g(s_b)$ and the function $o\mapsto f(o,o_{\mathrm{ref}}(c))$ is non-constant, there exist reference observations for which
$f(g(s_a),o_{\mathrm{ref}}(c))\neq f(g(s_b),o_{\mathrm{ref}}(c))$.
This is exactly $E_{\mathrm{goal}}(s_a\mid c)\neq E_{\mathrm{goal}}(s_b\mid c)$, contradicting task-state invariance.
\end{proof}

The proposition does not claim that goal similarity is uninformative.
It shows that without additional structure, a single reference anchored similarity score cannot serve as a globally consistent progress coordinate on task state space.
Consequently, any progress signal derived from it is unstable under symmetries, viewpoint changes, or multiple valid terminal configurations.

\subsubsection{Pairwise relative evaluators are non-additive and path-dependent in general}
\label{app:pairwise_inconsistency}

Consider the pairwise evaluator $E_{\mathrm{pair}}$ in Eq.~\eqref{eq:pair_eval_appC}.
Macro-consistency requires Eq.~\eqref{eq:cocycle_appC} for all triples, which by Lemma~\ref{lem:additivity_implies_potential_appC} is equivalent to the existence of a scalar potential whose differences reproduce $E_{\mathrm{pair}}$.

The next proposition gives a simple certificate for violation.

\begin{proposition}[A triple violation rules out any global potential]
\label{prop:pair_no_potential_appC}
Fix $c\in\mathcal{C}$.
If there exists a triple $(s_a,s_b,s_c)$ such that
\begin{equation}
E_{\mathrm{pair}}(s_a,s_c\mid c)\neq E_{\mathrm{pair}}(s_a,s_b\mid c)+E_{\mathrm{pair}}(s_b,s_c\mid c),
\label{eq:triple_violate_appC}
\end{equation}
then $E_{\mathrm{pair}}(\cdot,\cdot\mid c)$ violates macro-consistency and cannot be expressed in the potential-difference form of Eq.~\eqref{eq:potential_difference_appC}.
\end{proposition}

\begin{proof}
Eq.~\eqref{eq:triple_violate_appC} contradicts the required cocycle identity in Eq.~\eqref{eq:cocycle_appC}.
Lemma~\ref{lem:additivity_implies_potential_appC} then implies that no such potential exists.
\end{proof}

Such triple violations are common for relative comparison models that operate by heuristic matching, discriminative classification, or preference scoring on observation pairs.
Without explicit constraints enforcing additivity, the resulting rule is non-conservative and therefore path-dependent.

\subsection{Scale Drift Under Temporal Resampling}
\label{app:inconsistency_drift}

We now formalize scale drift as dependence of accumulated score on how the same physical execution is segmented.

Given a trajectory $\tau=\{s_t\}_{t=0}^{T}$, define the cumulative progress under a pairwise evaluator as
\begin{equation}
P_E(\tau\mid c)\coloneqq \sum_{t=1}^{T}E_{\mathrm{pair}}(s_{t-1},s_t\mid c).
\label{eq:cum_pair_appC}
\end{equation}
If macro-consistency holds, Corollary~\ref{cor:telescoping_appC} implies that $P_E(\tau\mid c)$ depends only on endpoints and is invariant to temporal refinement.

When Eq.~\eqref{eq:cocycle_appC} fails, scale drift occurs even under a minimal refinement.

\begin{theorem}[Constructive scale drift from a single triple violation]
\label{thm:scale_drift_constructive_appC}
Fix $c\in\mathcal{C}$.
Suppose there exist $s_a,s_b,s_c\in\mathcal{S}$ such that Eq.~\eqref{eq:triple_violate_appC} holds.
Consider the length-one trajectory $\tau=(s_a,s_c)$ and its refinement $\tau'=(s_a,s_b,s_c)$.
Then
\begin{equation}
P_E(\tau\mid c)\neq P_E(\tau'\mid c).
\end{equation}
\end{theorem}

\begin{proof}
By definition,
\[
P_E(\tau\mid c)=E_{\mathrm{pair}}(s_a,s_c\mid c),
\qquad
P_E(\tau'\mid c)=E_{\mathrm{pair}}(s_a,s_b\mid c)+E_{\mathrm{pair}}(s_b,s_c\mid c).
\]
These are unequal by Eq.~\eqref{eq:triple_violate_appC}.
\end{proof}

Theorem~\ref{thm:scale_drift_constructive_appC} shows that once additivity fails on a single triple, segmentation dependence is unavoidable.
Under higher-frequency resampling, the discrepancy can accumulate over many inserted intermediate states, producing systematic drift with the control rate.

\subsection{Physically Grounded Manipulation Cases}
\label{app:inconsistency_cases}

We list representative manipulation scenarios where the above failure modes arise in non-pathological ways.

\paragraph{Multi-view terminal grasps.}
In pick and lift tasks, success is typically defined by a stable grasp and an object height threshold.
There can exist multiple terminal success states that differ in wrist orientation or camera viewpoint while satisfying identical success predicates.
Such states are task-equivalent but yield different observations, so Proposition~\ref{prop:goal_similarity_equiv_violation_appC} implies that a goal-similarity evaluator can assign inconsistent terminal progress values.

\paragraph{Occlusion during insertion.}
In insertion tasks, physical progress correlates with insertion depth.
As the object enters a cavity, visual evidence of the inserted portion decreases due to occlusion.
An appearance-based score can therefore decrease even when physical progress increases, creating spurious regressions that reflect observation changes rather than execution errors.

\paragraph{Symmetry and multiple valid end configurations.}
Tasks involving symmetric objects or goal regions admit multiple end configurations that are equally valid under task semantics.
A similarity score anchored to a single reference observation effectively breaks this symmetry and induces different scores across equivalent solutions, violating task-state invariance.

\subsection{Implications}
\label{app:inconsistency_discussion}

Lemma~\ref{lem:additivity_implies_potential_appC} identifies macro-consistency with the existence of a globally defined potential on task states.
Goal-similarity scores are not guaranteed to be well-defined on task-equivalence classes and therefore cannot reliably serve as progress coordinates.
Pairwise relative evaluators generally fail the cocycle identity and become path-dependent, which yields scale drift under temporal resampling by Theorem~\ref{thm:scale_drift_constructive_appC}.
These observations motivate evaluators that explicitly induce a globally consistent progress potential, as required by the axiomatic framework in Sec.~\ref{subsec:theory}.

\newpage

\section{Macro-Consistency of PRM-as-a-Judge}
\label{app:macro_consistency_proof}

Appendix~\ref{app:inconsistency_proof} shows that many similarity- and relative-comparison judges fail macro-consistency because their pairwise scores do not reduce to a single global progress coordinate. Under the proposed potential-based formulation, PRM-as-a-Judge is macro-consistent by construction: it assigns each state (or information state) a scalar progress value and defines pairwise increments as differences of that value. This section makes the argument explicit.

\subsection{Macro-Consistency as an Additivity Axiom}
Fix a task context $c\in\mathcal{C}$. Let $\Delta(\cdot,\cdot\mid c):\mathcal{X}\times\mathcal{X}\to\mathbb{R}$ denote the progress increment assigned to an ordered state pair. The macro-consistency axiom in Sec.~\ref{subsec:theory} requires that for any triple of states,
\begin{equation}
\Delta(x_i,x_k\mid c)
=
\Delta(x_i,x_j\mid c)+\Delta(x_j,x_k\mid c),
\qquad \forall x_i,x_j,x_k\in\mathcal{X}.
\label{eq:mc_additivity}
\end{equation}
Intuitively, the total progress from $x_i$ to $x_k$ must be independent of how we split the transition.

\subsection{Additivity Implies the Existence of a Global Potential}
\begin{theorem}[Additivity is equivalent to a potential-difference form]
\label{thm:mc_iff_potential}
Fix $c\in\mathcal{C}$. The additivity condition in Eq.~\eqref{eq:mc_additivity} holds for all triples if and only if there exists a scalar potential function $\Phi(\cdot\mid c):\mathcal{X}\to\mathbb{R}$ such that
\begin{equation}
\Delta(x_i,x_j\mid c)=\Phi(x_j\mid c)-\Phi(x_i\mid c),
\qquad \forall x_i,x_j\in\mathcal{X}.
\label{eq:potential_form_app}
\end{equation}
Moreover, $\Phi(\cdot\mid c)$ is unique up to an additive constant.
\end{theorem}

\begin{proof}
We prove both directions.

\paragraph{Sufficiency.}
Assume Eq.~\eqref{eq:potential_form_app} holds. Then for any $x_i,x_j,x_k$,
\[
\Delta(x_i,x_k\mid c)
=
\Phi(x_k\mid c)-\Phi(x_i\mid c)
=
\bigl(\Phi(x_j\mid c)-\Phi(x_i\mid c)\bigr)+\bigl(\Phi(x_k\mid c)-\Phi(x_j\mid c)\bigr)
=
\Delta(x_i,x_j\mid c)+\Delta(x_j,x_k\mid c),
\]
which is Eq.~\eqref{eq:mc_additivity}.

\paragraph{Necessity.}
Assume Eq.~\eqref{eq:mc_additivity} holds. Fix an arbitrary reference state $x_{\mathrm{ref}}\in\mathcal{X}$ and define
\begin{equation}
\Phi(x\mid c)\coloneqq \Delta(x_{\mathrm{ref}},x\mid c).
\label{eq:phi_from_delta}
\end{equation}
Apply Eq.~\eqref{eq:mc_additivity} to the triple $(x_{\mathrm{ref}},x_i,x_j)$:
\[
\Delta(x_{\mathrm{ref}},x_j\mid c)
=
\Delta(x_{\mathrm{ref}},x_i\mid c)+\Delta(x_i,x_j\mid c).
\]
Rearranging yields
\[
\Delta(x_i,x_j\mid c)=\Delta(x_{\mathrm{ref}},x_j\mid c)-\Delta(x_{\mathrm{ref}},x_i\mid c)=\Phi(x_j\mid c)-\Phi(x_i\mid c),
\]
which is Eq.~\eqref{eq:potential_form_app}.

\paragraph{Uniqueness up to a constant.}
If $\Phi$ and $\Psi$ both satisfy Eq.~\eqref{eq:potential_form_app}, then for any $x$,
\[
\Phi(x\mid c)-\Psi(x\mid c)=\Phi(x_{\mathrm{ref}}\mid c)-\Psi(x_{\mathrm{ref}}\mid c),
\]
so $\Phi(\cdot\mid c)$ and $\Psi(\cdot\mid c)$ differ by a constant.
\end{proof}

\paragraph{Remark.}
Theorem~\ref{thm:mc_iff_potential} is stated in full generality: the potential
$\Phi(\cdot\mid c)$ may take values in $\mathbb{R}$, since additivity only requires
a scalar potential-difference form. In the main paper, however, the OPD metrics are
instantiated on a normalized task progress potential in $[0,1]$, so that the resulting
outcome-, process-, and diagnosis-level quantities are bounded and directly interpretable.

\subsection{PRM-as-a-Judge is Macro-Consistent by Construction}
In this work, PRM-as-a-Judge first produces a single scalar progress score for each input under context $c$. At the theoretical level, this score may be viewed as a real-valued potential. For OPD evaluation, however, we use its normalized form as the task-conditioned
progress potential, denoted by
\begin{equation}
\Phi_\theta(\,\cdot\mid c): \mathcal{X}\to [0,1].
\end{equation}
Here $\mathcal{X}$ is the domain on which the judge is single-valued. In fully observed settings, one may take $\mathcal{X}=\mathcal{S}$. Under partial observability, one may take $\mathcal{X}$ to be an information state that is sufficient for judging progress, for example $x_t=(o_{0:t},a_{0:t-1})$.

We then define the induced increment functional by the difference of the normalized
judge outputs:
\begin{equation}
\Delta_\theta(x_i,x_j\mid c)\coloneqq \Phi_\theta(x_j\mid c)-\Phi_\theta(x_i\mid c).
\label{eq:delta_theta_def}
\end{equation}

\begin{corollary}[Macro-consistency of PRM-as-a-Judge]
\label{cor:prm_macro_consistent}
For any fixed $c$, $\Delta_\theta(\cdot,\cdot\mid c)$ in Eq.~\eqref{eq:delta_theta_def} satisfies the additivity axiom in Eq.~\eqref{eq:mc_additivity} for all triples in $\mathcal{X}$.
\end{corollary}

\begin{proof}
Eq.~\eqref{eq:delta_theta_def} is exactly the potential-difference form in Eq.~\eqref{eq:potential_form_app} with potential $\Phi=\Phi_\theta$. Theorem~\ref{thm:mc_iff_potential} then implies additivity.
\end{proof}

\subsection{Temporal Resampling Invariance via Telescoping}
Let a trajectory be $\tau=(x_0,x_1,\dots,x_T)$ in $\mathcal{X}$. Define accumulated progress by summing local increments:
\begin{equation}
P_\theta(\tau\mid c)\coloneqq \sum_{t=1}^{T}\Delta_\theta(x_{t-1},x_t\mid c).
\label{eq:accumulated_progress}
\end{equation}

\begin{theorem}[Telescoping and invariance to segmentation]
\label{thm:telescoping_prm}
For any $\tau$, $P_\theta(\tau\mid c)=\Phi_\theta(x_T\mid c)-\Phi_\theta(x_0\mid c)$. Consequently, $P_\theta$ is invariant to how the same execution is temporally segmented, and unchanged under temporal refinement obtained by inserting intermediate states along the same rollout.
\end{theorem}

\begin{proof}
Substituting Eq.~\eqref{eq:delta_theta_def} into Eq.~\eqref{eq:accumulated_progress} gives
\[
P_\theta(\tau\mid c)
=
\sum_{t=1}^{T}\bigl(\Phi_\theta(x_t\mid c)-\Phi_\theta(x_{t-1}\mid c)\bigr)
=
\Phi_\theta(x_T\mid c)-\Phi_\theta(x_0\mid c),
\]
by telescoping cancellation. Inserting intermediate states only adds terms that cancel in the same manner.
\end{proof}

\paragraph{Remark.}
The conclusion is structural: any judge that outputs a single absolute score $\Phi_\theta(x\mid c)$ induces an additive increment by differencing. This guarantees macro-consistency of the induced increments, but it does not by itself guarantee that $\Phi_\theta$ is a faithful representation of physical progress. That aspect is evaluated empirically in the main experiments.

\newpage

\section{RoboPulse Data Composition}
\label{app:robopulse_data}

RoboPulse consists of 1,800 pairwise progress judgment cases constructed from trajectories across 7 different data sources.
These sources include real-world robot teleoperation~\cite{agibot,RoboBrainX0,ji2025robobrain,tan2026robobrain,team2025robobrain}, simulation rollouts~\cite{libero,robocasa,robotwin20}, UMI-based data collection, and egocentric human demonstrations~\cite{egodex,ji2025visualtrans,wang2025towards}.
The cases span diverse robot embodiments, sensing configurations, and manipulation tasks.
Tab.~\ref{tab:embodiment_stats} summarizes their distribution across these categories.

\begin{table}[!ht]
    \centering
    \caption{\textbf{Distribution of RoboPulse cases by robot embodiment.} RoboPulse spans diverse robot hardware, from industrial robots to humanoids.}
    \label{tab:embodiment_stats}
    \resizebox{0.95\linewidth}{!}{%
    \begin{tabular}{l|l|r}
        \toprule
        \textbf{Robot Embodiment} & \textbf{Primary Data Sources} & \textbf{\# Cases} \\
        \midrule
        Franka Emika Panda & DROID~\cite{khazatsky2024droid}, LIBERO~\cite{libero}, RoboCasa~\cite{robocasa} & 600 \\
        AGIBot-A2D & AGIBot-World~\cite{agibot} & 200 \\
        Agilex Piper & RoboBrain-X~\cite{RoboBrainX0}, RoboTwin~\cite{robotwin20} & 400\\
        Galaxea R1Lite & RoboBrain-X~\cite{RoboBrainX0} & 200 \\
        Pika & RoboBrain-X~\cite{RoboBrainX0} & 200 \\
        Human & Egodex~\cite{egodex} & 200 \\
        \bottomrule
    \end{tabular}
    }
\end{table}


\section{Normalization and Sampling Protocol}
\label{app:norm_sampling}

To prevent the benchmark from being dominated by trivial frame transitions or biased by specific data collection frequencies, we implement a three-stage construction pipeline. This process transforms raw heterogeneous trajectories into a standardized set of evaluation pairs, ensuring balanced coverage across both physical progress magnitude and temporal duration.

\subsection{Dense Progress Discretization}
We construct a dense benchmark progress signal from curated expert demonstrations.
During annotation, we first segment each episode into semantically coherent phases using manually selected key frames.
We retain only phases in which task progress is monotonic for the given task context.
Intervals that do not support a stable signed progress label are excluded, including near-static intervals with visually negligible task-relevant change, oscillatory intervals with back-and-forth motion but no net advancement, and cases where annotators cannot assign a reliable progress direction under the task context.

Given the retained phases, raw multi-view video trajectories are partitioned using human-annotated keyframes $\{K_0, K_1, \dots, K_N\}$, delimiting the task from start ($K_0$) to completion ($K_N$).
To generate dense benchmark labels, we apply adaptive interpolation within each retained semantic phase.
Given a trajectory length $L$ and a target density chunk size $C = 30$, we calculate the number of sampling points $m$ for each segment $[K_j, K_{j+1}]$ as:
\begin{equation}
m = \left\lfloor \frac{1}{N} \left\lfloor \frac{L}{C} \right\rfloor \right\rfloor.
\end{equation}
This procedure yields a discrete state sequence $\{x_0, x_1, \dots, x_M\}$, where each $x_i$ contains synchronized observations. We assign a linear scalar potential to each state, serving as the benchmark progress coordinate for pair construction:
\begin{equation}
\Phi(x_i) = \frac{i}{M}, \quad \Phi(x_i) \in [0, 1].
\end{equation}

\subsection{Context-Aware Relative Normalization}

A naive approach to labeling progress is to compute the absolute potential difference $\Delta \Phi = \Phi(x_q) - \Phi(x_p)$. However, this formulation suffers from a severe \textit{\textbf{distributional imbalance}}. In dense robotic trajectories, the vast majority of frame pairs exhibit only minute state changes, while large progress leaps are statistically rare. This long-tailed distribution biases the dataset toward near-zero labels, making it difficult to evaluate the model's ability to distinguish significant milestones.

To rectify this skew and ensure a uniform distribution of labels across the $[-1, 1]$ range, we adopt a \textit{\textbf{Hop-based formulation}}. This formulation rescales local state changes according to the current stage of execution, making progress judgments more comparable across the trajectory and preserving sensitivity to small but task-relevant changes near completion.

For a given pair of states, i.e., a \textit{pre-state} $x_p$ and a \textit{post-state} $x_q$, we define the normalized hop score $\mathcal{H}(x_p, x_q)$ based on the direction of the state evolution: Forward Progress and Backward Regression.

\begin{itemize}
    \item \textbf{Forward Progress.} When the agent advances toward the goal (i.e., $\Phi(x_q) \geq \Phi(x_p)$), the progress is normalized by the \textit{remaining potential} to the terminal state:
    \begin{equation}
    \label{eq:hop_progress}
    \mathcal{H}(x_p, x_q) = \frac{\Phi(x_q) - \Phi(x_p)}{\Phi(x_M) - \Phi(x_p)}.
    \end{equation}
    \textit{Interpretation:} The denominator measures how much progress remains from the pre-state. As a result, late-stage changes occupy a larger relative scale than the same absolute change would earlier in the trajectory. This helps preserve discrimination for small but task-critical adjustments near completion, such as final alignment or insertion.

    \item \textbf{Backward Regression.} Conversely, when the execution degrades (i.e., $\Phi(x_q) < \Phi(x_p)$), the negative score is normalized by the \textit{accumulated potential} from the start:
    \begin{equation}
    \label{eq:hop_regress}
    \mathcal{H}(x_p, x_q) = \frac{\Phi(x_q) - \Phi(x_p)}{\Phi(x_p) - \Phi(x_0)}.
    \end{equation}
    \textit{Interpretation:} The denominator measures how much progress has already been accumulated by the pre-state. This yields a stage-aware normalization for backward transitions, so that regressions are compared relative to where they occur along the trajectory rather than only by absolute drop size.
\end{itemize}

\subsection{Dual-Variable Stratified Sampling}
The final dataset is constructed using a stratified sampling technique that operates on two dimensions simultaneously. This approach decouples the physical magnitude of a change from the time interval required to execute it.

\begin{itemize}
    \item \textbf{Magnitude Stratification ($N_{\text{hop}}$):} We categorize the continuous hop values $\mathcal{H}$ into three distinct scales based on their absolute magnitude $|\mathcal{H}|$. These bins are defined as \textit{Small} ($[0, 33.3\%]$), \textit{Medium} ($(33.3\%, 66.7\%]$), and \textit{Large} ($(66.7\%, 100\%]$). This stratification ensures that the benchmark systematically evaluates the model's sensitivity across different granularities, ranging from fine-grained nudges to major state transitions.
    
    \item \textbf{Temporal Decoupling ($N_{\text{dis}}$):} Within each magnitude bucket, we vary the frame distance $\Delta t$ between $s_p$ and $s_q$. By sampling across different time horizons, we prevent evaluators from relying on temporal shortcuts (e.g., assuming large time gaps always imply large progress).
\end{itemize}

The resulting dataset comprises 1,800 pairs distributed across these dimensions, offering comprehensive coverage with {748} Small cases, {500} Medium cases, and {552} Large cases for final {{RoboPulse}} benchmark.

\clearpage


\section{More Results on RoboPulse}
\label{more_results_on_robopulse}

\begin{table*}[!t]
\centering
\caption{\textbf{Comparison of discriminative similarity-based methods and progress reward model judges under different noise levels.} 
We report the performance across Real-World, Simulation, UMI, and Human settings, along with the average (AVG) for Noise levels of 0, 0.05, and 0.10.}
\label{tab:noise_robustness}
\resizebox{\linewidth}{!}{%
\begin{tabular}{l|ccccc|ccccc|ccccc}
\toprule
\multirow{2}{*}{\textbf{Method}} 
& \multicolumn{5}{c|}{\textbf{Noise = 0}} 
& \multicolumn{5}{c|}{\textbf{Noise = 0.05}} 
& \multicolumn{5}{c}{\textbf{Noise = 0.10}} \\
\cmidrule{2-16}
& Real. & Sim. & UMI & Human & AVG
& Real. & Sim. & UMI & Human & AVG
& Real. & Sim. & UMI & Human & AVG \\
\midrule

\multicolumn{16}{l}{\cellcolor{gray!12}\textbf{Discriminative Similarity-Based Methods}} \\
\midrule
CLIP ViT-B/32 (I2I) & 0.57 & 0.62 & 0.58 & 0.53 & 0.57 & 0.57 & 0.63 & 0.62 & 0.59 & 0.60 & 0.53 & 0.63 & 0.59 & 0.53 & 0.57 \\
CLIP ViT-L/14 (I2I) & 0.59 & 0.64 & 0.58 & 0.56 & 0.59 & 0.56 & 0.62 & 0.64 & 0.51 & 0.58 & 0.56 & 0.57 & 0.63 & 0.54 & 0.58 \\
CLIP ViT-B/32 (T2I) & 0.46 & 0.46 & 0.51 & 0.46 & 0.47 & 0.47 & 0.46 & 0.53 & 0.49 & 0.49 & 0.47 & 0.47 & 0.53 & 0.54 & 0.50 \\
CLIP ViT-L/14 (T2I) & 0.48 & 0.44 & 0.42 & 0.52 & 0.46 & 0.48 & 0.44 & 0.43 & 0.49 & 0.46 & 0.49 & 0.47 & 0.49 & 0.51 & 0.49 \\

\midrule
\multicolumn{16}{l}{\cellcolor{gray!12}\textbf{Progress Reward Model Judges}} \\
\midrule
VLAC & 0.68 & 0.71 & 0.72 & 0.70 & 0.71 & 0.71 & 0.70 & 0.70 & 0.63 & 0.69 & 0.63 & 0.64 & 0.72 & 0.65 & 0.66 \\
RoboReward & 0.16 & 0.22 & 0.19 & 0.23 & 0.20 & 0.18 & 0.23 & 0.25 & 0.25 & 0.23 & 0.16 & 0.17 & 0.13 & 0.17 & 0.16 \\
Robo-Dopamine & 0.74 & 0.89 & 0.88 & 0.81 & 0.83 & 0.73 & 0.88 & 0.87 & 0.76 & 0.81 & 0.70 & 0.84 & 0.78 & 0.68 & 0.75 \\

\bottomrule
\end{tabular}%
}
\end{table*}

\subsection{Robustness to Visual Perturbations}
\label{sec:robustness_noise}

Real-world robotic feedback often comes with imperfect visual observations (e.g., sensor noise, compression artifacts, motion blur, or illumination changes). To evaluate robustness under such perturbations, we conduct controlled noise injection experiments by corrupting the input frames with additive Gaussian noise before feeding them into each scorer.

\textbf{Noise injection protocol.}
Given an RGB image $I \in [0,1]^{H \times W \times 3}$, we sample i.i.d. Gaussian noise $\epsilon \sim \mathcal{N}(0, \mathbf{I})$ and construct a corrupted image
\begin{equation}
\tilde{I} = (1-\alpha) I + \alpha \epsilon,
\end{equation}
where $\alpha \in [0,1]$ denotes the noise level. This design yields an intuitive interpretation: $\alpha=0$ corresponds to the original image (no perturbation), while larger $\alpha$ progressively increases corruption; at $\alpha=1$, the input becomes pure noise. In this section, we focus on two practical perturbation levels, $\alpha \in \{0.05, 0.10\}$, which represent mild and moderate observation noise. We apply the same corruption to all methods and all frames used by the evaluator, keeping every other evaluation setting unchanged.

\textbf{Discussion and analysis.}
Tab.~\ref{tab:noise_robustness} reports performance for discriminative baselines and reward models under Noise$=0$, $0.05$, and $0.10$ across Real-World, Simulation, UMI, and Human settings (with AVG denoting the mean over the four categories). Overall, we observe that modern reward models remain comparatively robust under mild noise, while purely discriminative similarity-based baselines vary more across settings.
Moreover, two trends emerge, as follow:
\vspace{-1em}
\begin{itemize}
    \item  \textit{\textbf{First, reward models are generally more noise-tolerant than discriminative CLIP-style baselines}}. For instance, Robo-Dopamine achieves the best overall performance at Noise$=0$ (AVG $0.83$) and exhibits a gradual degradation as noise increases (AVG $0.81$ at $0.05$ and $0.75$ at $0.10$), suggesting that the learned reward signal is not overly sensitive to small pixel-level perturbations. VLAC shows a similar pattern, with a modest drop from AVG $0.71$ (Noise$=0$) to $0.69$ (Noise$=0.05$) and $0.66$ (Noise$=0.10$), indicating stable behavior under mild-to-moderate corruption.

    \item \textit{\textbf{Second, the impact of noise is heterogeneous across evaluation settings}}. Under Noise$=0.10$, Robo-Dopamine degrades most notably in Human and UMI (Human: $0.81 \rightarrow 0.68$, UMI: $0.88 \rightarrow 0.78$), consistent with these settings being visually diverse and potentially requiring finer-grained cues that are more vulnerable to corruption. In contrast, the Simulation setting remains relatively resilient (Sim: $0.89 \rightarrow 0.84$), likely due to cleaner rendering and reduced background complexity.
\end{itemize}
\vspace{-1em}
For discriminative similarity-based methods, the behavior is less consistent. Image-to-image CLIP variants can remain competitive under Noise$=0.05$ (e.g., \textit{CLIP ViT-B/32 (I2I)} improves from AVG $0.57$ to $0.60$), but these gains do not persist at higher noise (AVG returns to $0.57$ at Noise$=0.10$). Text-to-image CLIP baselines are consistently weaker and show limited sensitivity to noise, suggesting that their main bottleneck lies more in semantic alignment with task descriptions than in pixel-level perturbations.
Finally, RoboReward performs poorly across all noise levels (AVG $0.20 \rightarrow 0.23 \rightarrow 0.16$). This pattern suggests that its weakness is not primarily due to visual corruption. Instead, its coarse discrete reward space appears fundamentally mismatched to the fine-grained pairwise progress judgment required by RoboPulse, since large quantization steps make subtle but task-relevant state changes difficult to distinguish reliably.
These results demonstrate that our strongest reward-model evaluator (Robo-Dopamine) maintains high accuracy under realistic noise levels ($0.05$ and $0.10$), and that learned reward models with sufficiently fine progress resolution are preferable when robustness to visual corruption is required.

\subsection{Fine-Grained Analysis of RoboPulse}
\label{sec:detail_gran_eval}

\begin{table*}[!t]
\centering
\caption{Performance comparison on the \textbf{Small} scale datasets.}
\label{tab:robopulse_small}
\resizebox{\linewidth}{!}{%
\begin{tabular}{l|cccc|ccc|c|c}
\toprule
& \multicolumn{4}{c|}{Real-World} & \multicolumn{3}{c|}{Simulation} & UMI & Human \\
\cmidrule(lr){2-5} \cmidrule(lr){6-8} \cmidrule(lr){9-9} \cmidrule(lr){10-10}
\textbf{Method} 
& Agibot-World & Agilex & Droid & Galaxea R1Lite & Libero & RoboCasa & RoboTwin2.0 & Pika & Egodex \\
\midrule

\multicolumn{10}{l}{\cellcolor{gray!12}\textbf{Discriminative Similarity-Based Methods}} \\
CLIP ViT-B/32 (I2I) & 0.54 & 0.53 & 0.75 & 0.45 & 0.55 & 0.41 & 0.61 & 0.58 & 0.58 \\
CLIP ViT-L/14 (I2I) & 0.55 & 0.53 & 0.52 & 0.47 & 0.59 & 0.54 & 0.53 & 0.55 & 0.56 \\
CLIP ViT-B/32 (T2I) & 0.50 & 0.47 & 0.55 & 0.42 & 0.54 & 0.44 & 0.50 & 0.56 & 0.49 \\
CLIP ViT-L/14 (T2I) & 0.54 & 0.46 & 0.55 & 0.47 & 0.45 & 0.50 & 0.44 & 0.48 & 0.52 \\

\midrule
\multicolumn{10}{l}{\cellcolor{gray!12}\textbf{General Foundation-Model Judges}} \\
Gemini 3 Pro Preview & 0.56 & 0.53 & 0.60 & 0.51 & 0.69 & 0.56 & 0.61 & 0.43 & 0.56 \\
GPT-5.2 & 0.45 & 0.48 & 0.48 & 0.42 & 0.45 & 0.47 & 0.47 & 0.47 & 0.49 \\
Qwen3-VL-4B-Instruct & 0.53 & 0.45 & 0.55 & 0.37 & 0.50 & 0.46 & 0.50 & 0.34 & 0.53 \\
Qwen3-VL-8B-Instruct & 0.53 & 0.47 & 0.54 & 0.42 & 0.57 & 0.51 & 0.45 & 0.44 & 0.47 \\

\midrule
\multicolumn{10}{l}{\cellcolor{gray!12}\textbf{Progress Reward Model Judges}} \\
VLAC & 0.64 & 0.51 & 0.76 & 0.52 & 0.72 & 0.47 & 0.66 & 0.66 & 0.57 \\
GVL & 0.73 & 0.58 & 0.65 & 0.49 & 0.66 & 0.63 & 0.73 & 0.58 & 0.67 \\
RoboReward & 0.06 & 0.06 & 0.13 & 0.08 & 0.13 & 0.08 & 0.14 & 0.10 & 0.10 \\
Robo-Dopamine & 0.89 & 0.74 & 0.67 & 0.68 & 0.99 & 1.00 & 0.98 & 0.97 & 0.86 \\

\bottomrule
\end{tabular}%
}
\end{table*}

In this section, we provide a granular analysis of the RoboPulse benchmark by decomposing the pairwise progress-judgment task into three difficulty levels: \textbf{Small}, \textbf{Medium}, and \textbf{Large} time intervals (hops). This breakdown allows us to evaluate the sensitivity of different scorers to varying degrees of visual state changes. The detailed results per dataset are presented in Tab.~\ref{tab:robopulse_small} (Small), Tab.~\ref{tab:robopulse_medium} (Medium), and Tab.~\ref{tab:robopulse_large} (Large).

\textbf{Fine-grained sensitivity (Small Scale).} 
Tab.~\ref{tab:robopulse_small} presents the results for small temporal hops, representing the most challenging setting where visual changes between frames are subtle. In this regime, general-purpose foundation models and discriminative baselines struggle significantly. For instance, large VLMs like \textit{GPT-5.2} and \textit{Qwen3-VL} hover near random chance (avg $\sim0.47-0.50$), indicating they lack the fine-grained understanding of low-level kinematics required to detect immediate progress. 
In contrast, specialized reward models demonstrate superior sensitivity. \textit{Robo-Dopamine} achieves remarkable accuracy even on these short intervals, reaching near-perfect scores in simulation environments (e.g., $0.99$ in \textit{Libero} and $1.00$ in \textit{RoboCasa}) and maintaining strong performance in real-world settings ($0.89$ in \textit{Agibot-World}). This suggests that contrastive learning on dense video data equips the model with a precise understanding of micro-progressions that semantic discriminators miss.

\textbf{Semantic progress understanding (Large Scale).}
As we increase the temporal interval to the \textbf{Large} scale (Tab.~\ref{tab:robopulse_large}), the visual disparity between the starting and ending frames becomes more pronounced, often reflecting the completion of a sub-goal. Consequently, general foundation-model judges show substantial improvement. \textit{Gemini 3 Pro Preview}, for example, improves its performance on \textit{Libero} from $0.69$ (Small) to $0.90$ (Large). This trend confirms that while VLMs may miss fine-grained dynamics, they are capable of identifying high-level semantic state changes (e.g., ``door closed'' vs. ``door open''). However, discriminative similarity-based methods like \textit{CLIP (T2I)} remain ineffective, reinforcing that static text-image alignment is insufficient for capturing temporal progress regardless of the scale.

\textbf{Domain-specific observations.}
Across all three scales, we observe a performance gap between Simulation and Real-World/Human data. Simulation environments (e.g., \textit{Libero}, \textit{RoboCasa}) generally yield higher accuracy for top-performing models due to cleaner visual renderings and consistent lighting. Real-world datasets (e.g., \textit{Droid}, \textit{Agibot}) introduce complexity via lighting variations and camera noise. Notably, \textit{Robo-Dopamine} exhibits the strongest domain robustness, maintaining high accuracy on the challenging \textit{Egodex} dataset ($0.93$ at Large scale), whereas \textit{VLAC} and \textit{GVL} see a performance dip in these diverse human-centric domains.

\textbf{Failure modes.}
Consistent with the main results, \textit{RoboReward} displays extremely low accuracy ($<0.25$) across most sub-tasks and scales. This consistently poor performance suggests that its coarse discrete reward space is poorly matched to the fine-grained progress discrimination required by RoboPulse. In particular, large quantization steps make subtle but task-relevant state changes difficult to distinguish reliably, especially at smaller comparison scales.

In summary, while general VLMs become competitive for judging long-term progress, specialized reward models such as \textit{Robo-Dopamine} remain more reliable when fine-grained discrimination and short-horizon feedback are required.

\begin{table*}[!t]
\centering
\caption{Performance comparison on the \textbf{Medium} scale datasets.}
\label{tab:robopulse_medium}
\resizebox{\linewidth}{!}{%
\begin{tabular}{l|cccc|ccc|c|c}
\toprule
& \multicolumn{4}{c|}{Real-World} & \multicolumn{3}{c|}{Simulation} & UMI & Human \\
\cmidrule(lr){2-5} \cmidrule(lr){6-8} \cmidrule(lr){9-9} \cmidrule(lr){10-10}
\textbf{Method} 
& Agibot-World & Agilex & Droid & Galaxea R1Lite & Libero & RoboCasa & RoboTwin2.0 & Pika & Egodex \\
\midrule

\multicolumn{10}{l}{\cellcolor{gray!12}\textbf{Discriminative Similarity-Based Methods}} \\
CLIP ViT-B/32 (I2I) & 0.57 & 0.49 & 0.57 & 0.62 & 0.69 & 0.59 & 0.66 & 0.52 & 0.50 \\
CLIP ViT-L/14 (I2I) & 0.65 & 0.55 & 0.54 & 0.73 & 0.81 & 0.50 & 0.66 & 0.50 & 0.59 \\
CLIP ViT-B/32 (T2I) & 0.51 & 0.47 & 0.57 & 0.48 & 0.44 & 0.60 & 0.40 & 0.35 & 0.34 \\
CLIP ViT-L/14 (T2I) & 0.49 & 0.53 & 0.46 & 0.48 & 0.33 & 0.45 & 0.47 & 0.35 & 0.52 \\

\midrule
\multicolumn{10}{l}{\cellcolor{gray!12}\textbf{General Foundation-Model Judges}} \\
Gemini 3 Pro Preview & 0.63 & 0.62 & 0.64 & 0.71 & 0.83 & 0.65 & 0.63 & 0.73 & 0.59 \\
GPT-5.2 & 0.49 & 0.64 & 0.51 & 0.69 & 0.68 & 0.47 & 0.56 & 0.54 & 0.34 \\
Qwen3-VL-4B-Instruct & 0.51 & 0.64 & 0.54 & 0.54 & 0.65 & 0.45 & 0.58 & 0.59 & 0.50 \\
Qwen3-VL-8B-Instruct & 0.53 & 0.65 & 0.54 & 0.73 & 0.78 & 0.52 & 0.66 & 0.61 & 0.41 \\

\midrule
\multicolumn{10}{l}{\cellcolor{gray!12}\textbf{Progress Reward Model Judges}} \\
VLAC & 0.53 & 0.55 & 0.92 & 0.63 & 0.89 & 0.60 & 0.77 & 0.70 & 0.75 \\
GVL & 0.74 & 0.75 & 0.66 & 0.69 & 0.83 & 0.58 & 0.72 & 0.75 & 0.69 \\
RoboReward & 0.08 & 0.18 & 0.25 & 0.08 & 0.28 & 0.07 & 0.30 & 0.19 & 0.16 \\
Robo-Dopamine & 0.98 & 1.00 & 0.63 & 0.96 & 1.00 & 1.00 & 0.97 & 0.98 & 0.94 \\

\bottomrule
\end{tabular}%
}
\end{table*}

\begin{table*}[!t]
\centering
\caption{Performance comparison on the \textbf{Large} scale datasets.}
\label{tab:robopulse_large}
\resizebox{\linewidth}{!}{%
\begin{tabular}{l|cccc|ccc|c|c}
\toprule
& \multicolumn{4}{c|}{Real-World} & \multicolumn{3}{c|}{Simulation} & UMI & Human \\
\cmidrule(lr){2-5} \cmidrule(lr){6-8} \cmidrule(lr){9-9} \cmidrule(lr){10-10}
\textbf{Method} 
& Agibot-World & Agilex & Droid & Galaxea R1Lite & Libero & RoboCasa & RoboTwin2.0 & Pika & Egodex \\
\midrule

\multicolumn{10}{l}{\cellcolor{gray!12}\textbf{Discriminative Similarity-Based Methods}} \\
CLIP ViT-B/32 (I2I) & 0.65 & 0.52 & 0.61 & 0.61 & 0.76 & 0.64 & 0.65 & 0.63 & 0.51 \\
CLIP ViT-L/14 (I2I) & 0.69 & 0.58 & 0.56 & 0.73 & 0.80 & 0.63 & 0.73 & 0.68 & 0.54 \\
CLIP ViT-B/32 (T2I) & 0.45 & 0.32 & 0.32 & 0.44 & 0.33 & 0.59 & 0.35 & 0.61 & 0.54 \\
CLIP ViT-L/14 (T2I) & 0.58 & 0.40 & 0.40 & 0.45 & 0.24 & 0.58 & 0.46 & 0.42 & 0.51 \\

\midrule
\multicolumn{10}{l}{\cellcolor{gray!12}\textbf{General Foundation-Model Judges}} \\
Gemini 3 Pro Preview & 0.75 & 0.64 & 0.71 & 0.79 & 0.90 & 0.84 & 0.81 & 0.77 & 0.74 \\
GPT-5.2 & 0.53 & 0.56 & 0.49 & 0.71 & 0.56 & 0.63 & 0.70 & 0.70 & 0.59 \\
Qwen3-VL-4B-Instruct & 0.55 & 0.66 & 0.49 & 0.73 & 0.90 & 0.53 & 0.68 & 0.65 & 0.68 \\
Qwen3-VL-8B-Instruct & 0.67 & 0.72 & 0.58 & 0.77 & 0.94 & 0.66 & 0.76 & 0.77 & 0.69 \\

\midrule
\multicolumn{10}{l}{\cellcolor{gray!12}\textbf{Progress Reward Model Judges}} \\
VLAC & 0.76 & 0.66 & 0.97 & 0.69 & 0.93 & 0.53 & 0.86 & 0.81 & 0.78 \\
GVL & 0.76 & 0.88 & 0.68 & 0.82 & 0.91 & 0.81 & 0.77 & 0.78 & 0.75 \\
RoboReward & 0.22 & 0.22 & 0.39 & 0.11 & 0.51 & 0.22 & 0.25 & 0.30 & 0.42 \\
Robo-Dopamine & 1.00 & 1.00 & 0.63 & 0.97 & 1.00 & 1.00 & 1.00 & 0.98 & 0.93 \\

\bottomrule
\end{tabular}%
}
\end{table*}

\subsection{Detailed Domain Analysis of RoboPulse}
\label{sec:detail_dom_eval}

In this section, we analyze the domain-specific performance of different evaluators, summarized in Tab.~\ref{tab:robopulse_summary}. The RoboPulse benchmark encompasses four distinct domains: \textbf{Simulation} (controlled environments like Libero/RoboCasa), \textbf{Real-World} (diverse robotic setups), \textbf{UMI}-based data collection on AgileX Pika, and \textbf{Human} (ego-centric human manipulation). Analyzing performance across these domains reveals the generalization capabilities of each method.

\textbf{Simulation vs. Real-World Gap.}
Consistent with general computer vision trends, most models perform better in \textit{\textbf{Simulation}} than in the \textit{\textbf{Real-World}}. For example, \textit{Gemini 3 Pro Preview} drops from $0.72$ in simulation to $0.63$ in real-world settings, and \textit{Qwen3-VL-8B} drops from $0.65$ to $0.58$. This performance degradation is attributed to the ``sim-to-real'' gap, where real-world data introduces complex lighting, background clutter, and sensor noise that general foundation-models and discriminative baselines struggle to filter out. 
However, specialized reward models show greater resilience. \textit{Robo-Dopamine} maintains a high accuracy of $0.83$ in the real world, significantly outperforming the best general models ($0.63$). This suggests that contrastive training on large-scale robotic data helps the model learn invariant features of manipulation progress that hold true despite visual domain shifts.

\begin{wraptable}{r}{0.55\textwidth}
\centering
\caption{Overall performance summary across different domains.}
\label{tab:robopulse_summary}
\resizebox{\linewidth}{!}{
\begin{tabular}{l|cccc}
\toprule
\textbf{Method} & \textbf{Human} & \textbf{Real-World} & \textbf{Simulation} & \textbf{UMI} \\
\midrule

\multicolumn{5}{l}{\cellcolor{gray!12}\textbf{Discriminative Similarity-Based Methods}} \\
CLIP ViT-B/32 (I2I) & 0.54 & 0.57 & 0.61 & 0.58 \\
CLIP ViT-L/14 (I2I) & 0.56 & 0.58 & 0.64 & 0.58 \\
CLIP ViT-B/32 (T2I) & 0.48 & 0.46 & 0.46 & 0.52 \\
CLIP ViT-L/14 (T2I) & 0.52 & 0.49 & 0.44 & 0.43 \\

\midrule

\multicolumn{5}{l}{\cellcolor{gray!12}\textbf{General Foundation-Model Judges}} \\
Gemini 3 Pro Preview & 0.62 & 0.63 & 0.72 & 0.60 \\
GPT-5.2 & 0.49 & 0.52 & 0.55 & 0.56 \\
Qwen3-VL-4B-Instruct & 0.57 & 0.53 & 0.58 & 0.50 \\
Qwen3-VL-8B-Instruct & 0.53 & 0.58 & 0.65 & 0.58 \\

\midrule
\multicolumn{5}{l}{\cellcolor{gray!12}\textbf{Progress Reward Model Judges}} \\
VLAC & 0.67 & 0.67 & 0.71 & 0.72 \\
GVL & 0.70 & 0.69 & 0.74 & 0.69 \\
RoboReward & 0.21 & 0.15 & 0.22 & 0.18 \\
Robo-Dopamine & 0.90 & 0.83 & 0.99 & 0.98 \\

\bottomrule
\end{tabular}%
}
\end{wraptable}

\textbf{Generalization to Human and UMI Data.}
The \textit{\textbf{Human}} and \textit{\textbf{UMI}} domains represent a significant challenge due to the domain shift from standard robot arms to human hands or handheld grippers.
Remarkably, \textit{Robo-Dopamine} achieves near-perfect performance in the UMI domain ($0.98$) and robust performance on Human data ($0.90$). This indicates that the model has learned to focus on the \textit{interaction between the end-effector and the object}, rather than overfitting to specific robot morphologies. 
In contrast, discriminative CLIP-based methods (e.g., \textit{CLIP ViT-B/32 I2I}) show limited generalization, hovering around $0.54-0.58$ across these domains. This implies that simple visual similarity is insufficient for capturing progress when the agent's embodiment changes (e.g., from a robot arm to a human hand).

\textbf{Consistency of Reward Models.}
Among the reward models, we observe a hierarchy of robustness. \textit{Robo-Dopamine} consistently dominates all domains, particularly in Simulation ($0.99$). \textit{VLAC} and \textit{GVL} perform respectably (averaging $\sim0.70$ across domains) but lack the extreme precision of Robo-Dopamine in the Simulation and UMI settings. \textit{RoboReward} continues to show poor alignment across all domains ($<0.22$), confirming that its issues are fundamental to the method rather than specific to a single domain.

In conclusion, while simulators provide a cleaner signal for evaluation, the true test of a reward model lies in the complex Real-World and Human domains. Our analysis highlights that specialized reward models like \textit{Robo-Dopamine} effectively bridge the domain gap, offering reliable progress estimation even when applied to visually diverse and morphologically distinct manipulation data.

\newpage

\begin{table*}[!t]
\centering
\setlength{\tabcolsep}{3.5pt} 
\renewcommand{\arraystretch}{1.2} 

\caption{\textbf{OPD auditing on RoboTwin 2.0.} 
Performance of different policy models on three tasks. 
MC@25/50/75/100 denotes milestone coverage; MP/PPL/CRA/STR are averaged episode-level metrics.}
\label{tab:robotwin_audit_from_image}

\resizebox{\textwidth}{!}{%
\begin{tabular}{l | cccc | c | c | c c | cccc | c | c | c c | cccc | c | c | c c}
\toprule

\multirow{4}{*}{\textbf{Model}} 
& \multicolumn{8}{c|}{\textbf{Blocks Ranking RGB}} 
& \multicolumn{8}{c|}{\textbf{Handover Block}} 
& \multicolumn{8}{c}{\textbf{Handover Mic}} \\
\cmidrule(lr){2-9} \cmidrule(lr){10-17} \cmidrule(lr){18-25}

& \multicolumn{5}{c|}{\textbf{Outcome Level}} & \textbf{Process} & \multicolumn{2}{c|}{\textbf{Diagnosis}} 
& \multicolumn{5}{c|}{\textbf{Outcome Level}} & \textbf{Process} & \multicolumn{2}{c|}{\textbf{Diagnosis}} 
& \multicolumn{5}{c|}{\textbf{Outcome Level}} & \textbf{Process} & \multicolumn{2}{c}{\textbf{Diagnosis}} \\
\cmidrule(lr){2-6} \cmidrule(lr){7-7} \cmidrule(lr){8-9}
\cmidrule(lr){10-14} \cmidrule(lr){15-15} \cmidrule(lr){16-17}
\cmidrule(lr){18-22} \cmidrule(lr){23-23} \cmidrule(lr){24-25}

& \multicolumn{4}{c|}{\textbf{MC}} & \multirow{2}{*}{\textbf{MP}} & \multirow{2}{*}{\textbf{PPL}} & \multirow{2}{*}{\textbf{CRA}} & \multirow{2}{*}{\textbf{STR}} 
& \multicolumn{4}{c|}{\textbf{MC}} & \multirow{2}{*}{\textbf{MP}} & \multirow{2}{*}{\textbf{PPL}} & \multirow{2}{*}{\textbf{CRA}} & \multirow{2}{*}{\textbf{STR}} 
& \multicolumn{4}{c|}{\textbf{MC}} & \multirow{2}{*}{\textbf{MP}} & \multirow{2}{*}{\textbf{PPL}} & \multirow{2}{*}{\textbf{CRA}} & \multirow{2}{*}{\textbf{STR}} \\
\cmidrule(lr){2-5} \cmidrule(lr){10-13} \cmidrule(lr){18-21}

& {@25} & {@50} & {@75} & {@100} & & & & 
& {@25} & {@50} & {@75} & {@100} & & & & 
& {@25} & {@50} & {@75} & {@100} & & & & \\
\midrule

ACT & 
84 & 44 & 22 & 2 & 49.93 & 11.67 & 8.99 & 59.74 & 
86 & 60 & 44 & 42 & 66.35 & 47.98 & 9.6 & 65.49 & 
100 & 100 & 94 & 74 & 96.79 & 72.33 & 4.08 & 44.14 \\ 

DP & 
94 & 40 & 18 & 0 & 51.72 & 4.07 & 16.26 & 43.77 & 
92 & 52 & 50 & 44 & 66.88 & 62.18 & 1.05 & 69.64 & 
100 & 94 & 88 & 44 & 93.8 & 65.97 & 5.49 & 57.18 \\ 

RDT & 
100 & 62 & 30 & 0 & 61.23 & 6.19 & 16.3 & 39.03 & 
94 & 82 & 62 & 38 & 78.86 & 53.13 & 9.88 & 65.09 & 
100 & 100 & 100 & 100 & 100 & 84.23 & 1.45 & 39.82 \\ 

pi0 & 
96 & 66 & 40 & 8 & 63.37 & 15.85 & 11.5 & 48.39 & 
84 & 58 & 50 & 40 & 68.22 & 43.49 & 8.61 & 53.48 & 
100 & 100 & 100 & 98 & 99.42 & 88.05 & 1.03 & 42.71 \\ 

OpenVLA-OFT & 
98 & 42 & 6 & 0 & 48.28 & 2.39 & 17.78 & 38.62 & 
84 & 44 & 36 & 2 & 56.44 & 4.74 & 18.6 & 55.72 & 
100 & 100 & 100 & 76 & 94.15 & 66.2 & 5.66 & 45.14 \\ 

\bottomrule
\end{tabular}%
}
\end{table*}
\begin{table*}[!t]
\centering
\setlength{\tabcolsep}{3.5pt} 
\renewcommand{\arraystretch}{1.2} 

\caption{\textbf{OPD Auditing on RoboTwin 2.0: Task Set A.} 
Results for \textit{Hanging Mug} and \textit{Place Bread Basket} (50 rollouts each).}
\label{tab:robotwin_audit_a}

\resizebox{\textwidth}{!}{%
\begin{tabular}{l | cccc | c | c | c c | cccc | c | c | c c}
\toprule

\multirow{4}{*}{\textbf{Model}} 
& \multicolumn{8}{c|}{\textbf{Hanging Mug}} 
& \multicolumn{8}{c}{\textbf{Place Bread Basket}} \\
\cmidrule(lr){2-9} \cmidrule(lr){10-17}

& \multicolumn{5}{c|}{\textbf{Outcome Level}} & \textbf{Process} & \multicolumn{2}{c|}{\textbf{Diagnosis}} 
& \multicolumn{5}{c|}{\textbf{Outcome Level}} & \textbf{Process} & \multicolumn{2}{c}{\textbf{Diagnosis}} \\
\cmidrule(lr){2-6} \cmidrule(lr){7-7} \cmidrule(lr){8-9}
\cmidrule(lr){10-14} \cmidrule(lr){15-15} \cmidrule(lr){16-17}

& \multicolumn{4}{c|}{\textbf{MC}} & \multirow{2}{*}{\textbf{MP}} & \multirow{2}{*}{\textbf{PPL}} & \multirow{2}{*}{\textbf{CRA}} & \multirow{2}{*}{\textbf{STR}} 
& \multicolumn{4}{c|}{\textbf{MC}} & \multirow{2}{*}{\textbf{MP}} & \multirow{2}{*}{\textbf{PPL}} & \multirow{2}{*}{\textbf{CRA}} & \multirow{2}{*}{\textbf{STR}} \\
\cmidrule(lr){2-5} \cmidrule(lr){10-13}

& {@25} & {@50} & {@75} & {@100} & & & & 
& {@25} & {@50} & {@75} & {@100} & & & & \\
\midrule

ACT & 
96 & 84 & 74 & 14 & 84.88 & 23.61 & 17.12 & 51.03 & 
100 & 74 & 46 & 4 & 73.11 & 17.55 & 15.46 & 65.38 \\ 

DP & 
100 & 100 & 98 & 14 & 86.4 & 30.82 & 18.72 & 57.03 & 
100 & 94 & 74 & 16 & 87.55 & 21.33 & 16.88 & 47.96 \\ 

RDT & 
98 & 96 & 92 & 20 & 88.5 & 48.4 & 11.35 & 65.84 & 
100 & 100 & 78 & 8 & 90.4 & 16.57 & 22.68 & 37.1 \\ 

pi0 & 
100 & 96 & 92 & 16 & 95.37 & 50.29 & 8.53 & 60.91 & 
100 & 94 & 62 & 16 & 83.67 & 21.16 & 18.91 & 47.62 \\ 

OpenVLA-OFT & 
100 & 92 & 88 & 8 & 82.3 & 18.9 & 21.19 & 40.89 & 
100 & 100 & 84 & 2 & 92.61 & 8.86 & 26.25 & 31.85 \\ 

\bottomrule
\end{tabular}%
}
\end{table*}

\begin{table*}[!t]
\centering
\setlength{\tabcolsep}{3.5pt} 
\renewcommand{\arraystretch}{1.2} 

\caption{\textbf{OPD Auditing on RoboTwin 2.0: Task Set B.} 
Results for \textit{Place Bread Skillet} and \textit{Place Can Basket} (50 rollouts each).}
\label{tab:robotwin_audit_b}

\resizebox{\textwidth}{!}{%
\begin{tabular}{l | cccc | c | c | c c | cccc | c | c | c c}
\toprule

\multirow{4}{*}{\textbf{Model}} 
& \multicolumn{8}{c|}{\textbf{Place Bread Skillet}} 
& \multicolumn{8}{c}{\textbf{Place Can Basket}} \\
\cmidrule(lr){2-9} \cmidrule(lr){10-17}

& \multicolumn{5}{c|}{\textbf{Outcome Level}} & \textbf{Process} & \multicolumn{2}{c|}{\textbf{Diagnosis}} 
& \multicolumn{5}{c|}{\textbf{Outcome Level}} & \textbf{Process} & \multicolumn{2}{c}{\textbf{Diagnosis}} \\
\cmidrule(lr){2-6} \cmidrule(lr){7-7} \cmidrule(lr){8-9}
\cmidrule(lr){10-14} \cmidrule(lr){15-15} \cmidrule(lr){16-17}

& \multicolumn{4}{c|}{\textbf{MC}} & \multirow{2}{*}{\textbf{MP}} & \multirow{2}{*}{\textbf{PPL}} & \multirow{2}{*}{\textbf{CRA}} & \multirow{2}{*}{\textbf{STR}} 
& \multicolumn{4}{c|}{\textbf{MC}} & \multirow{2}{*}{\textbf{MP}} & \multirow{2}{*}{\textbf{PPL}} & \multirow{2}{*}{\textbf{CRA}} & \multirow{2}{*}{\textbf{STR}} \\
\cmidrule(lr){2-5} \cmidrule(lr){10-13}

& {@25} & {@50} & {@75} & {@100} & & & & 
& {@25} & {@50} & {@75} & {@100} & & & & \\
\midrule

ACT & 
100 & 80 & 34 & 8 & 68.66 & 34.19 & 7.99 & 69.11 & 
98 & 88 & 74 & 4 & 83.84 & 30.62 & 9.06 & 69.9 \\ 

DP & 
100 & 58 & 30 & 4 & 62.51 & 31.45 & 8.36 & 78.07 & 
88 & 60 & 34 & 12 & 60.91 & 39.87 & 5.83 & 81.04 \\ 

RDT & 
100 & 80 & 50 & 6 & 76.42 & 27.87 & 12.62 & 63.74 & 
100 & 100 & 94 & 16 & 95.94 & 52.41 & 4.6 & 69.02 \\ 

pi0 & 
100 & 80 & 54 & 16 & 77.91 & 32.43 & 11.82 & 61.28 & 
100 & 94 & 74 & 28 & 87.38 & 45.87 & 8.03 & 67.06 \\ 

OpenVLA-OFT & 
100 & 60 & 26 & 10 & 63.35 & 14.23 & 13.65 & 30.74 & 
100 & 98 & 78 & 8 & 83.98 & 34.53 & 6.85 & 69.57 \\ 

\bottomrule
\end{tabular}%
}
\end{table*}

\section{More Results on OPD Auditing}
\label{more_results}

We report additional OPD auditing results for five policy families, including \textit{ACT}, \textit{Diffusion Policy (DP)}, \textit{RDT}, $\pi_0$, and \textit{OpenVLA-OFT}, on seven RoboTwin 2.0 tasks.
Tab.~\ref{tab:robotwin_audit_from_image}, Tab.~\ref{tab:robotwin_audit_a}, and Tab.~\ref{tab:robotwin_audit_b} provide per-task OPD breakdowns, while Fig.~\ref{fig:rq2_all_tasks}--\ref{fig:rq4_all_tasks} summarize the same phenomena from an all-task view.

\paragraph{Outcome level reveals where completion collapses.}
Across tasks, milestone coverage separates early reachability from terminal completion in a way that binary success cannot.
On \textit{Blocks Ranking RGB}, all methods frequently reach the first milestone (MC@25 is 84--100), yet none reliably completes the task (MC@100 is 0--8), indicating that failures concentrate after substantial partial progress.
A different pattern appears on \textit{Handover Mic}: several methods reach late milestones, but only $\pi_0$ and RDT consistently close the last stage (MC@100 of 98 and 100), whereas DP drops sharply at MC@100 of 44, matching the last-mile bottleneck emphasized in the main text.
For the remaining tasks in Tab.~\ref{tab:robotwin_audit_a} and Tab.~\ref{tab:robotwin_audit_b}, MC@100 remains non-trivial but low for most policies, suggesting that these manipulation tasks are typically limited by final stabilization and precise placement rather than early-stage navigation.

\paragraph{Process efficiency is not implied by high MP or high MC.}
PPL separates policies that reach similar maximum progress but do so with different path efficiency.
A concrete example is \textit{Blocks Ranking RGB}: $\pi_0$ attains MP of 63.37, comparable to RDT at 61.23, yet $\pi_0$ achieves a substantially higher PPL of 15.85 than RDT at 6.19, indicating that the two policies reach similar endpoints through different progress dynamics.
On \textit{Handover Mic}, high-outcome policies also differ in efficiency, where $\pi_0$ and RDT show strong PPL (88.05 and 84.23), while OpenVLA-OFT remains less efficient even when its MP is high (PPL 66.20 with MP 94.15).

\paragraph{Diagnosis metrics expose stagnation- versus regret-dominant failure modes.}
STR highlights freezing and unproductive interaction, while CRA captures persistent backtracking relative to the historical best progress.
A representative stagnation failure appears on \textit{Place Can Basket}, where DP shows STR of 81.04 with only moderate MP of 60.91, consistent with long stretches of near-zero progress despite occasional improvements.
In contrast, OpenVLA-OFT exhibits regret-dominant behavior on tasks such as \textit{Place Bread Basket}, where it reaches high MP of 92.61 but incurs large CRA of 26.25, consistent with late-stage instability and costly recovery.
These two regimes are behaviorally distinct but would both be labeled as the same terminal failure under success rate.

\paragraph{All-task visual summaries.}
Fig.~\ref{fig:rq2_all_tasks} aggregates milestone reachability across all tasks and makes the stage at which performance saturates visually comparable.
It shows that several tasks share a common pattern of high early reachability with a sharp final-stage drop, aligning with the per-task MC tables.
Fig.~\ref{fig:rq3_all_tasks} summarizes success-conditioned execution quality across tasks.
It complements the main-text observation that DP can be highly efficient on successes while being less reliable at the outcome level on last-mile constrained tasks.
Fig.~\ref{fig:rq4_all_tasks} reports failure-only OPD fingerprints with within-task normalization.
Two stable signatures reappear across tasks: stagnation-dominant failures with high raw STR and modest MP, and regret-dominant failures with high MP but poor CRA or low PPL.
We note that N/A entries indicate that a model has too few failed episodes to form a stable fingerprint under the current sample size, which is itself consistent with near-saturated outcomes on those tasks.

\begin{figure}[!t]
    \centering
    \includegraphics[width=\linewidth]{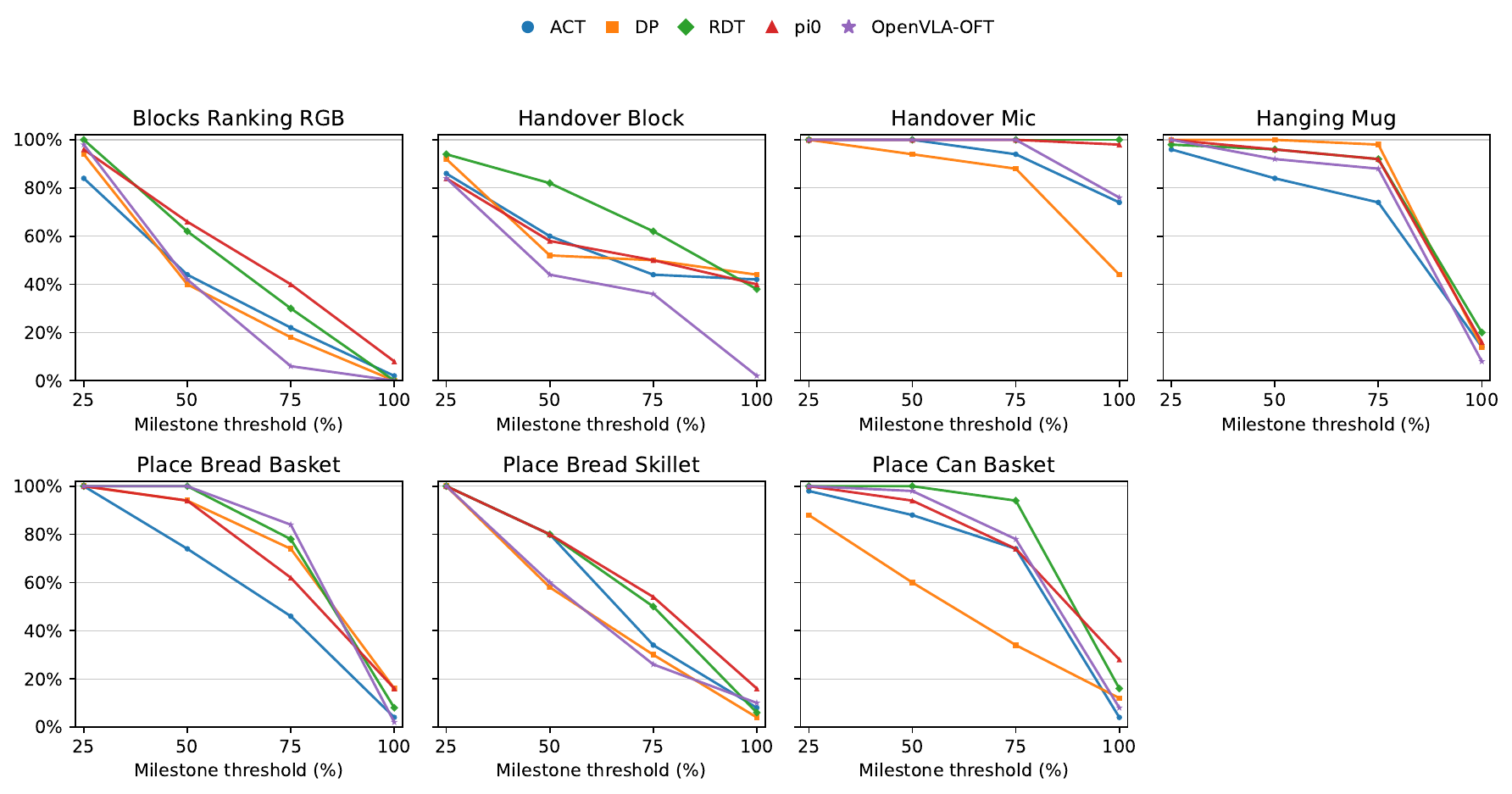}
    \caption{\textbf{Reachability profiles over all RoboTwin 2.0 tasks.}
For each task, we plot the fraction of rollouts that reach milestone thresholds at 25/50/75/100\%.
The curves localize where progress saturates along the horizon and expose last-mile bottlenecks when MC@75 is high but MC@100 drops sharply.}
    \label{fig:rq2_all_tasks}
\end{figure}

\begin{figure}[!ht]
    \centering
    \includegraphics[width=\linewidth]{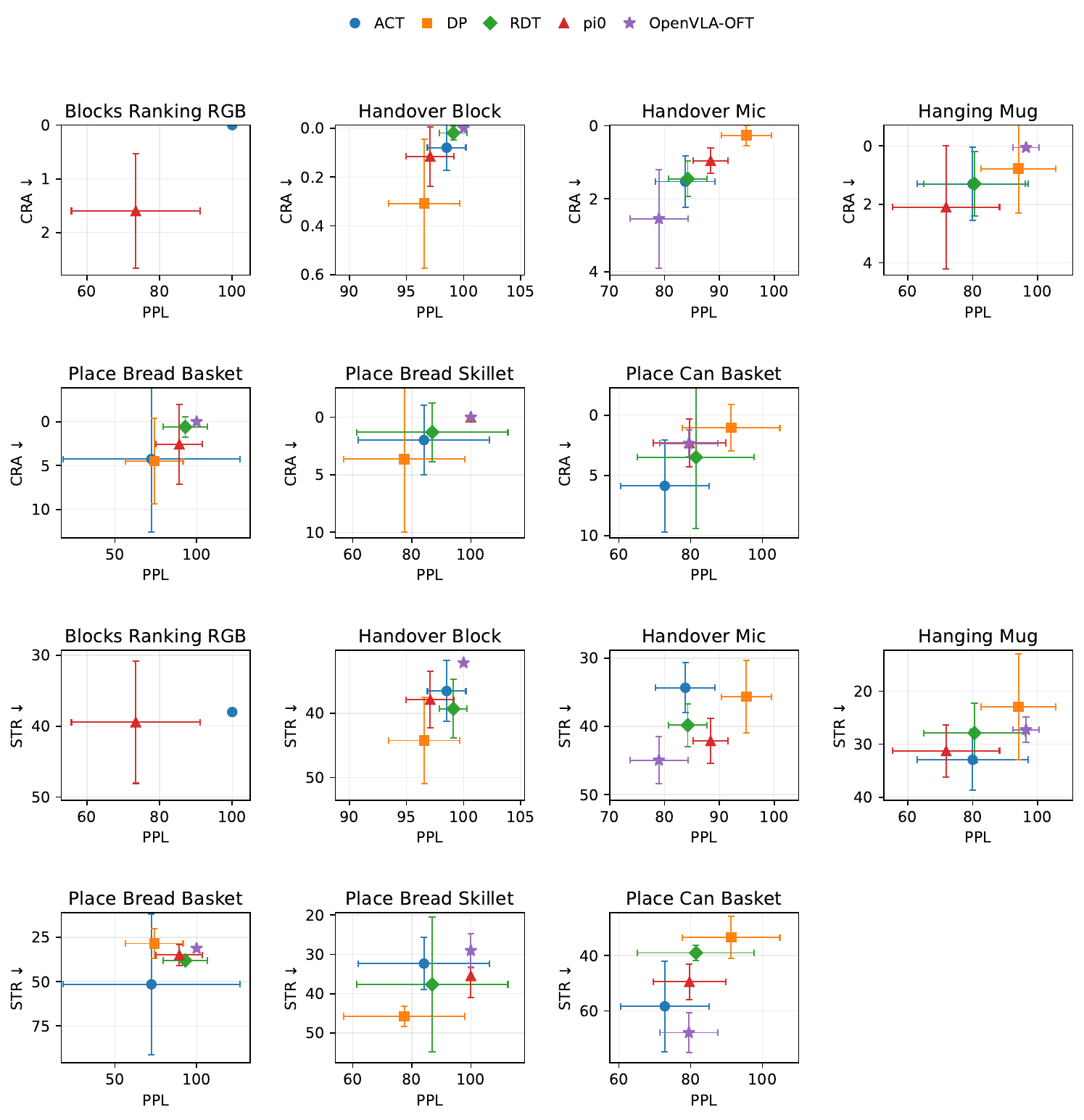}
    \caption{\textbf{Success-conditioned execution quality over all tasks.}
We compute PPL, CRA, and STR on successful episodes only and visualize cross-policy trade-offs.
This separates execution quality from success frequency and highlights policies that succeed efficiently versus those that succeed with higher correction cost or hesitation.}
    \label{fig:rq3_all_tasks}
\end{figure}

\begin{figure}[!ht]
    \centering
    \includegraphics[width=\linewidth]{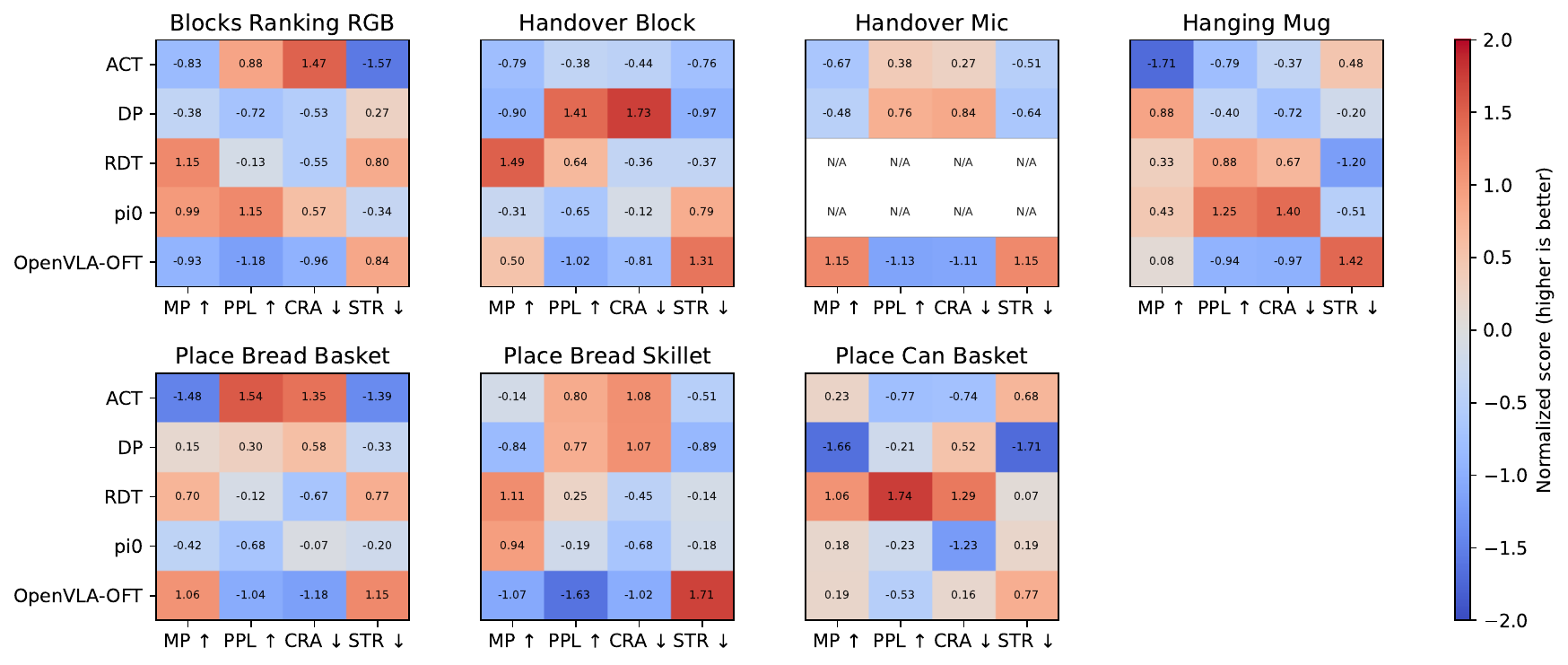}
    \caption{\textbf{Failure-only OPD fingerprints over all tasks.}
We aggregate MP, PPL, CRA, and STR over failed episodes and normalize scores within each task across policy families.
For CRA and STR, we flip the sign before normalization so that higher normalized scores consistently indicate more desirable behavior.}
    \label{fig:rq4_all_tasks}
\end{figure}

\clearpage

\section{Visualization}
\label{sec:visualization}

To intuitively understand how Robo-Dopamine evaluates robotic behavior, we visualize the reward curves generated during task execution across diverse real-world and simulation environments. Fig~\ref{fig:vis1} to \ref{fig:vis6} display the frame-by-frame progress scores (gray line) alongside key events annotated in the trajectories. We also report the computed OPD metrics (Milestone Coverage, Max Progress, PPL, CRA, STR) for each episode.

\textbf{Dense Feedback and Mistake Detection.}
A key property of an effective reward model is the ability to recognize negative progress. As shown in the \textit{Stack the wooden blocks} (Fig.~\ref{fig:vis1}) and \textit{Clean the table} (Fig.~\ref{fig:vis2}) tasks, the progress curve is non-monotonic. When the robot drops a block or jumbles up the tissues (highlighted by red arrows), Robo-Dopamine immediately penalizes the agent, causing a sharp drop in the score. This behavior is quantitatively captured by the \textit{\textbf{Cumulative Regret Area (CRA)}} metric. For instance, in Fig.~\ref{fig:vis2}, the significant regressions result in a CRA of $3.26\%$, accurately reflecting the wasted effort before recovery.

\textbf{Sensitivity to Stagnation.}
In long-horizon tasks such as \textit{Blocks Ranking RGB} (Fig.~\ref{fig:vis4}), the agent experiences multiple failures (e.g., empty grasps or grasping the incorrect block) before making meaningful progress. The reward curve remains flat or low during these phases, resulting in a higher \textit{\textbf{Stagnation Ratio (STR)}} of $16.22\%$. This demonstrates that the model does not hallucinate progress when the robot is merely moving without functional success.

\textbf{Recovery and Completion.}
Despite intermediate failures, the reward curves faithfully track the agent's recovery. In all visualized episodes, the agent eventually completes the task, indicated by the curve reaching $\approx 1.0$ at the final step. This aligns with the perfect \textit{\textbf{Milestone Coverage (MC)}} and \textit{\textbf{Max Progress (MP)}} of $100\%$. The \textit{\textbf{Path-weighted Progress Length (PPL)}} metric further contextualizes this efficiency; for example, the \textit{Place bread into the basket} task (Fig.~\ref{fig:vis5}) achieves a high PPL ($75.06\%$) because the regressions (moving away from goal) were brief and quickly corrected, whereas lower PPL scores indicate more ``struggle'' during the episode.

Overall, these visualizations suggest that Robo-Dopamine appears well aligned with the annotated progress signal in these examples and can diagnose distinct failure modes, such as regression and stagnation, through the OPD metric system.

\textbf{RoboPulse benchmark cases.}
Fig.~\ref{fig:robopulse_vis} visualizes representative RoboPulse instances, where the judge compares a BEFORE/AFTER frame pair using reference start/end frames as anchors and predicts the progress direction.
These examples illustrate a key failure mode of pure appearance matching: visually different frames can still correspond to valid progress or even terminal success, due to viewpoint changes, partial occlusion, and multiple acceptable end configurations.
In the first two cases, the AFTER frames are physically closer to completion, yet CLIP and a general-purpose VLM can mis-rank them because the dominant visual change is not aligned with task progress.
The third case highlights negative progress: the state moves away from the goal, which requires the judge to detect regression rather than similarity.
Overall, Robo-Dopamine remains consistent across these shifts, correctly recognizing both improvement and regression, while similarity-based scoring is more brittle to appearance variation.

\begin{figure*}[!t]
    \centering
    \includegraphics[width=\linewidth]{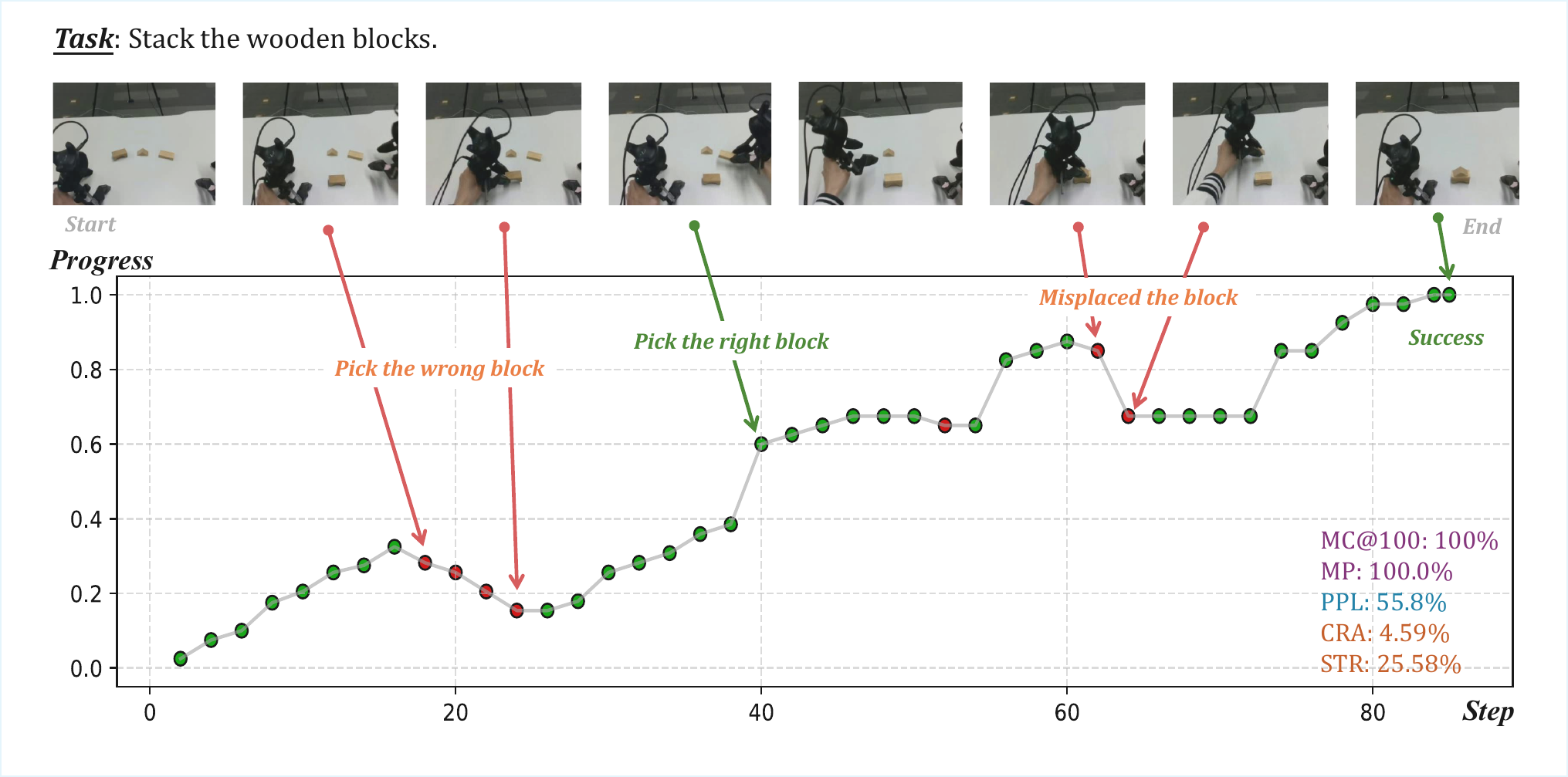}
    \caption{\textbf{Real-World Task: Stack the wooden blocks.} The reward curve demonstrates accurate tracking of manipulation errors. When the robot picks the wrong block or misplaces it (red arrows), the score drops, reflecting negative progress. The agent eventually recovers, reaching a success state (Progress $\approx 1.0$). The \textbf{CRA} of $4.59\%$ quantifies these regression events.}
    \label{fig:vis1}
\end{figure*}

\begin{figure*}[!t]
    \centering
    \includegraphics[width=\linewidth]{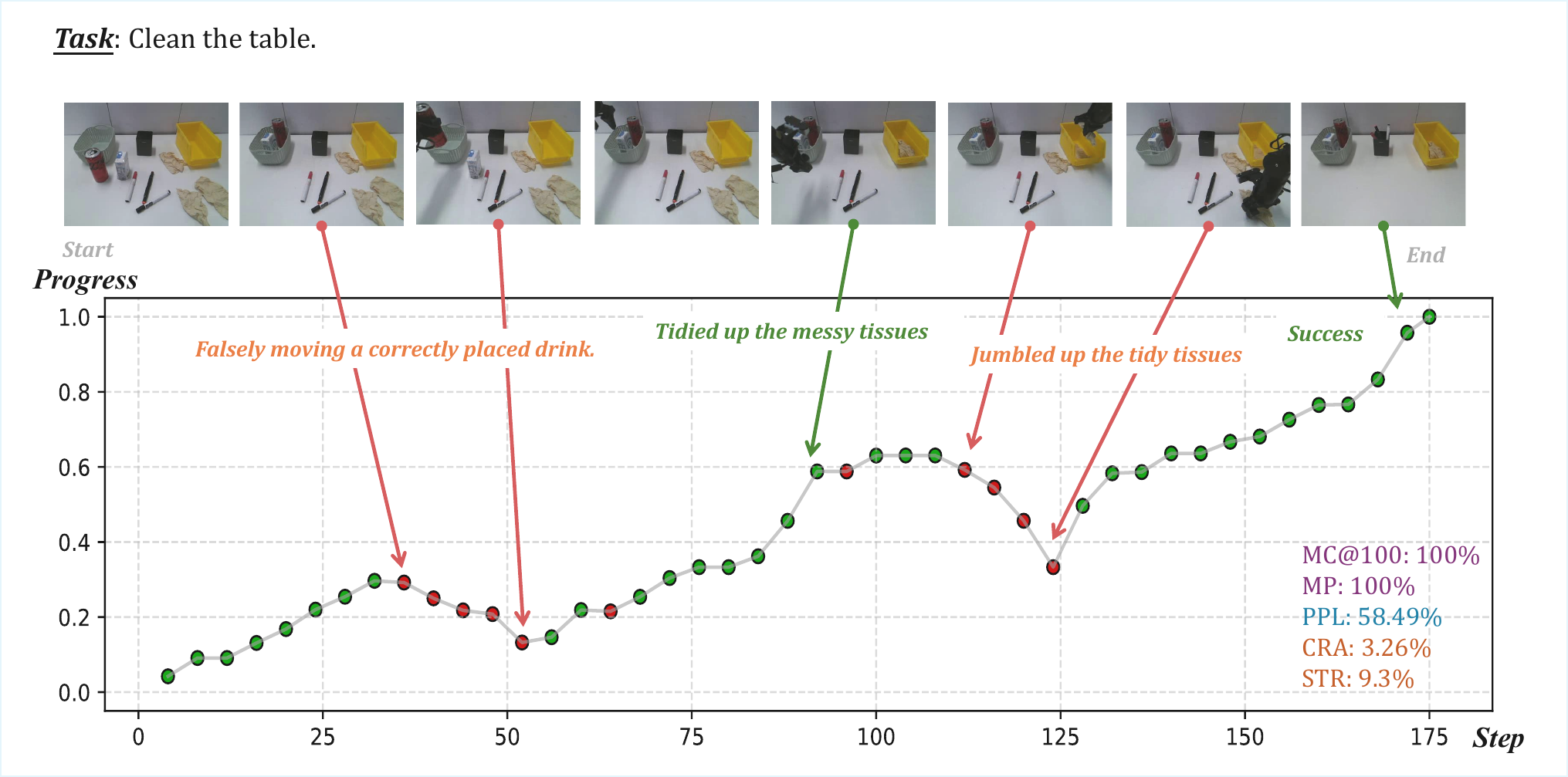}
    \caption{\textbf{Real-World Task: Clean the table.} This task involves semantic understanding of ``messiness.'' The model correctly penalizes the agent for moving a correctly placed drink and jumbling tissues, causing deep valleys in the progress curve. The successful recovery is captured by the final rise, though the efficiency (\textbf{PPL} $58.49\%$) is impacted by these mistakes.}
    \label{fig:vis2}
\end{figure*}

\begin{figure*}[!t]
    \centering
    \includegraphics[width=\linewidth]{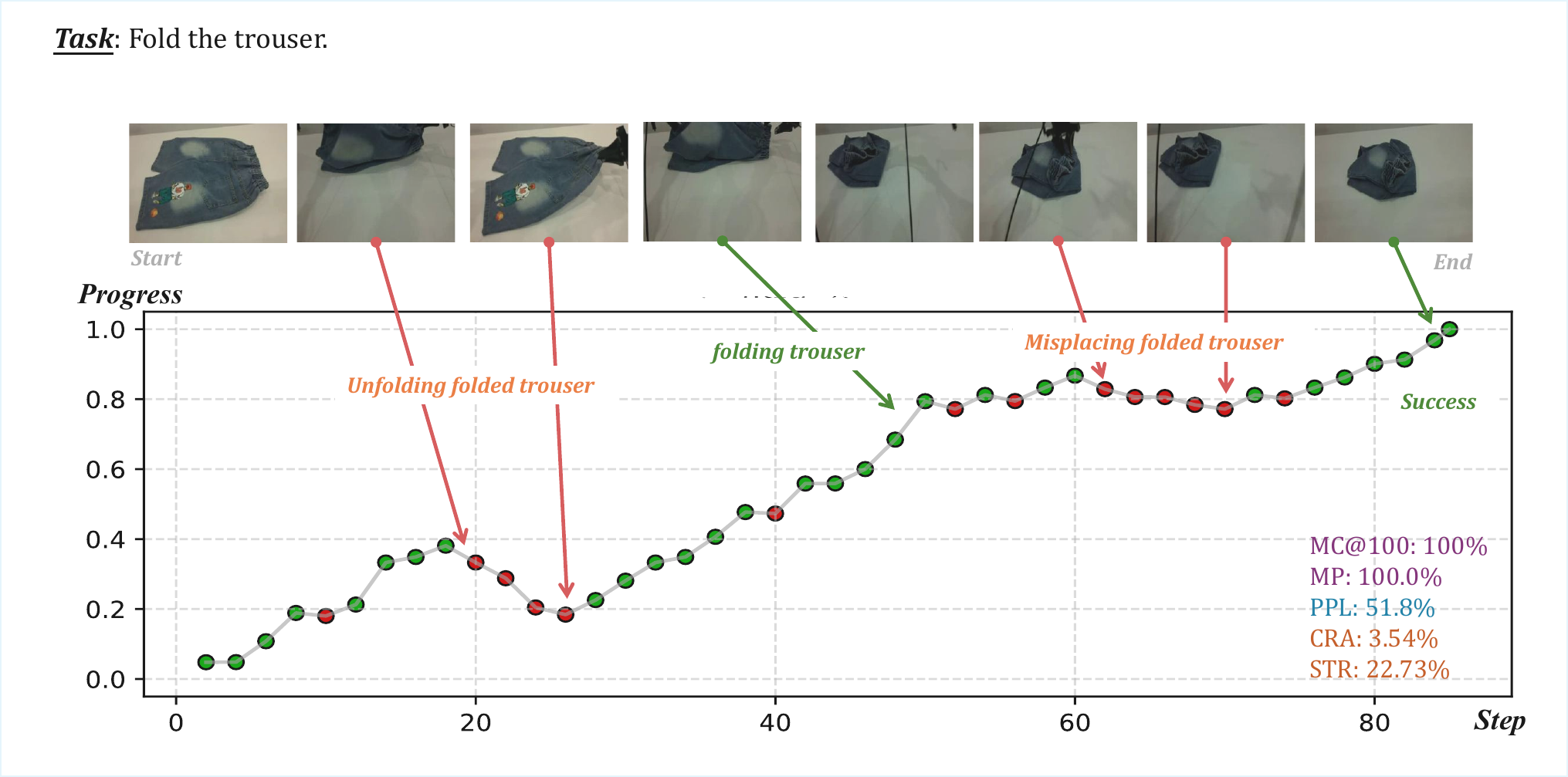}
    \caption{\textbf{Human Task: Fold the trouser.} Evaluating manipulation of deformable objects is challenging. Robo-Dopamine correctly identifies the regression when the trouser is accidentally unfolded (Step $\sim 18$) and when the folded trouser is misplaced (Step $\sim 60$). The curve aligns well with the visual state of the fabric.}
    \label{fig:vis3}
\end{figure*}

\begin{figure*}[!t]
    \centering
    \includegraphics[width=\linewidth]{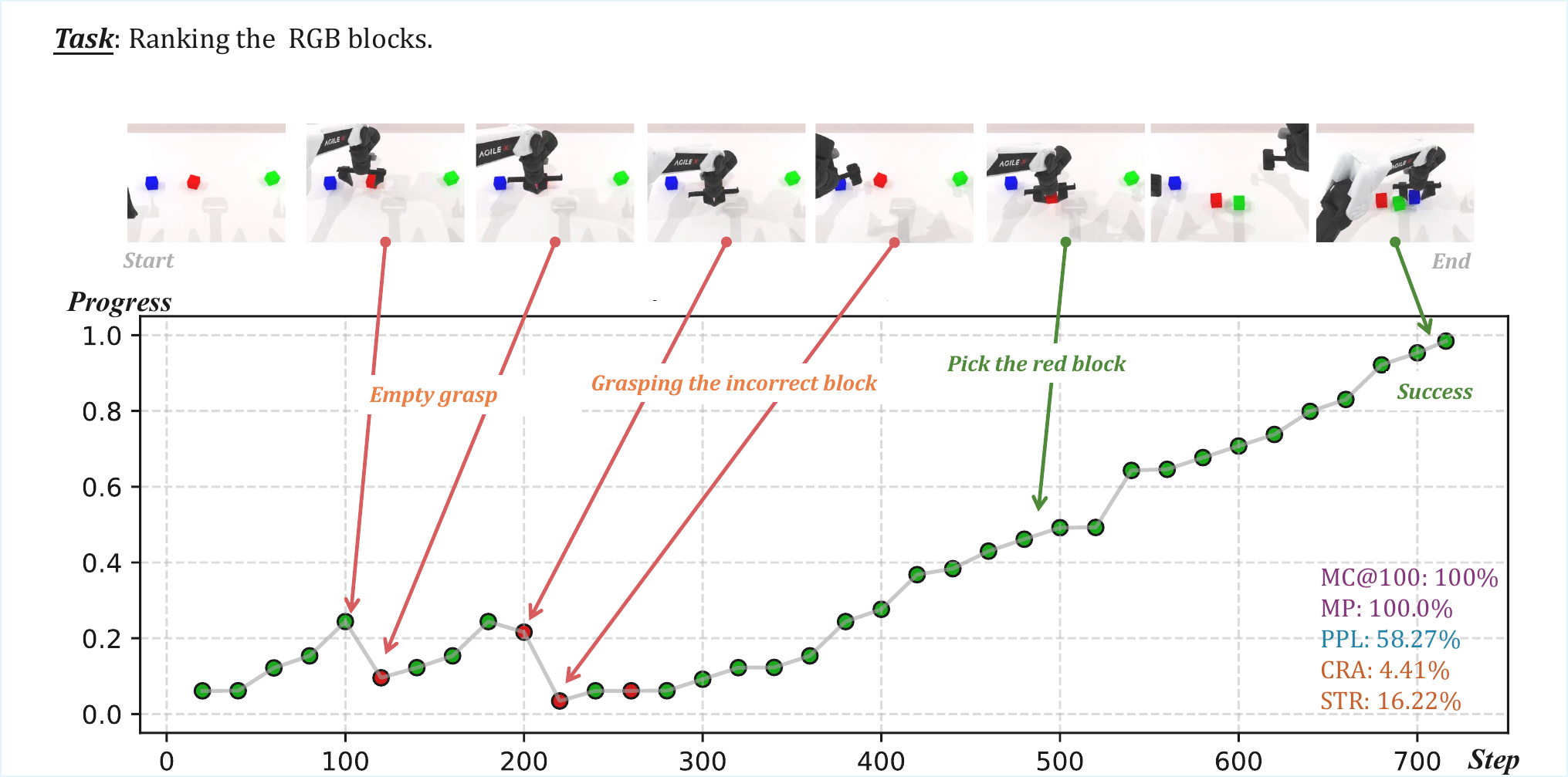}
    \caption{\textbf{Simulation Task: Blocks Ranking RGB.} A long-horizon task ($\sim 700$ steps). The initial phase is dominated by empty grasps and incorrect interactions, leading to a long period of stagnation (high \textbf{STR} of $16.22\%$ and flat curve). Once the correct sorting sequence begins (Step $\sim 500$), the reward signal increases steadily.}
    \label{fig:vis4}
\end{figure*}

\begin{figure*}[!t]
    \centering
    \includegraphics[width=\linewidth]{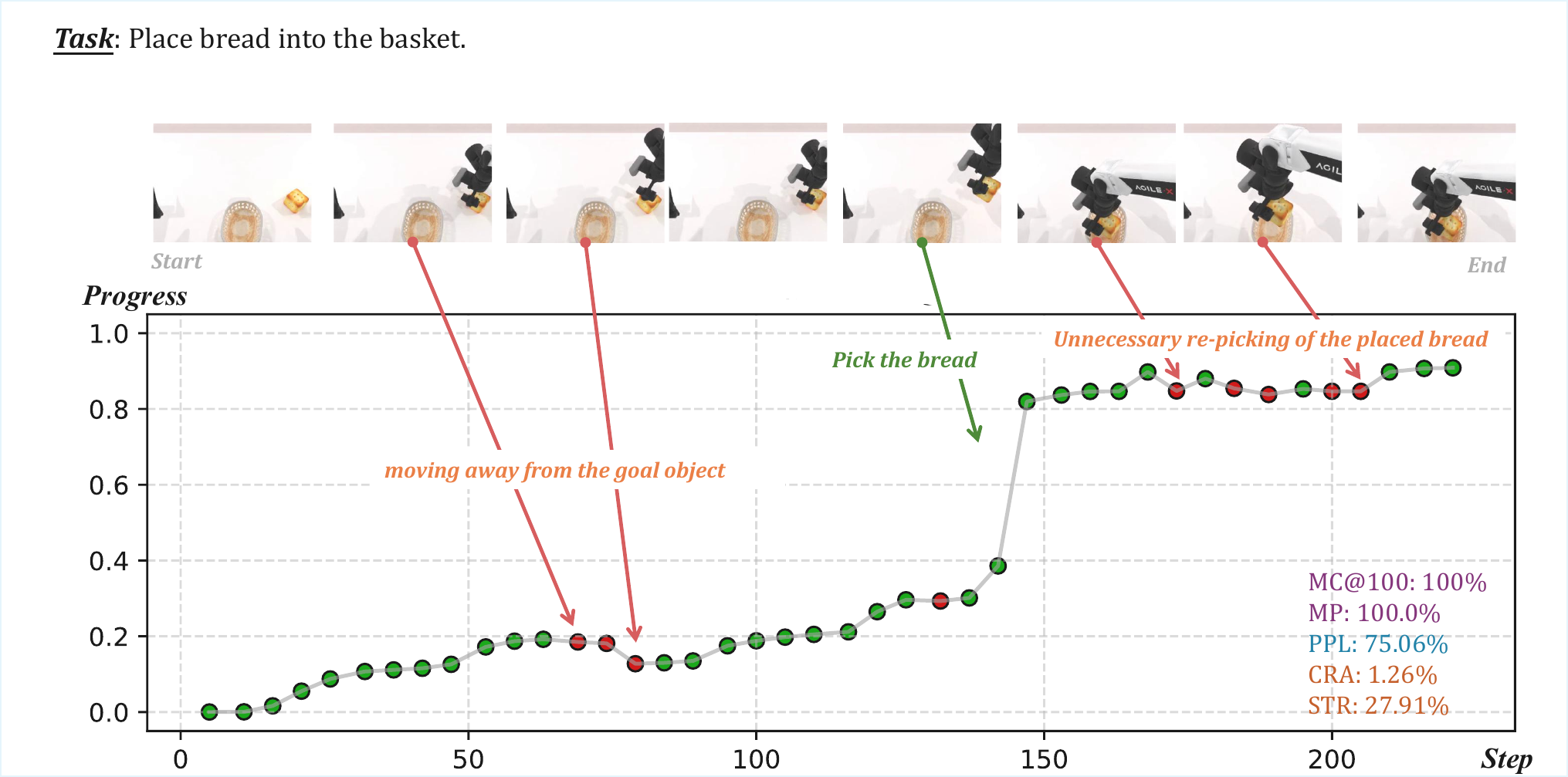}
    \caption{\textbf{Simulation Task: Place bread into the basket.} The model exhibits sensitivity to spatial relations. Moving the bread away from the basket (Step $\sim 75$) causes a dip in the score. The high \textbf{PPL} ($75.06\%$) indicates that despite minor deviations, the overall trajectory was relatively efficient compared to other tasks.}
    \label{fig:vis5}
\end{figure*}

\begin{figure*}[!t]
    \centering
    \includegraphics[width=\linewidth]{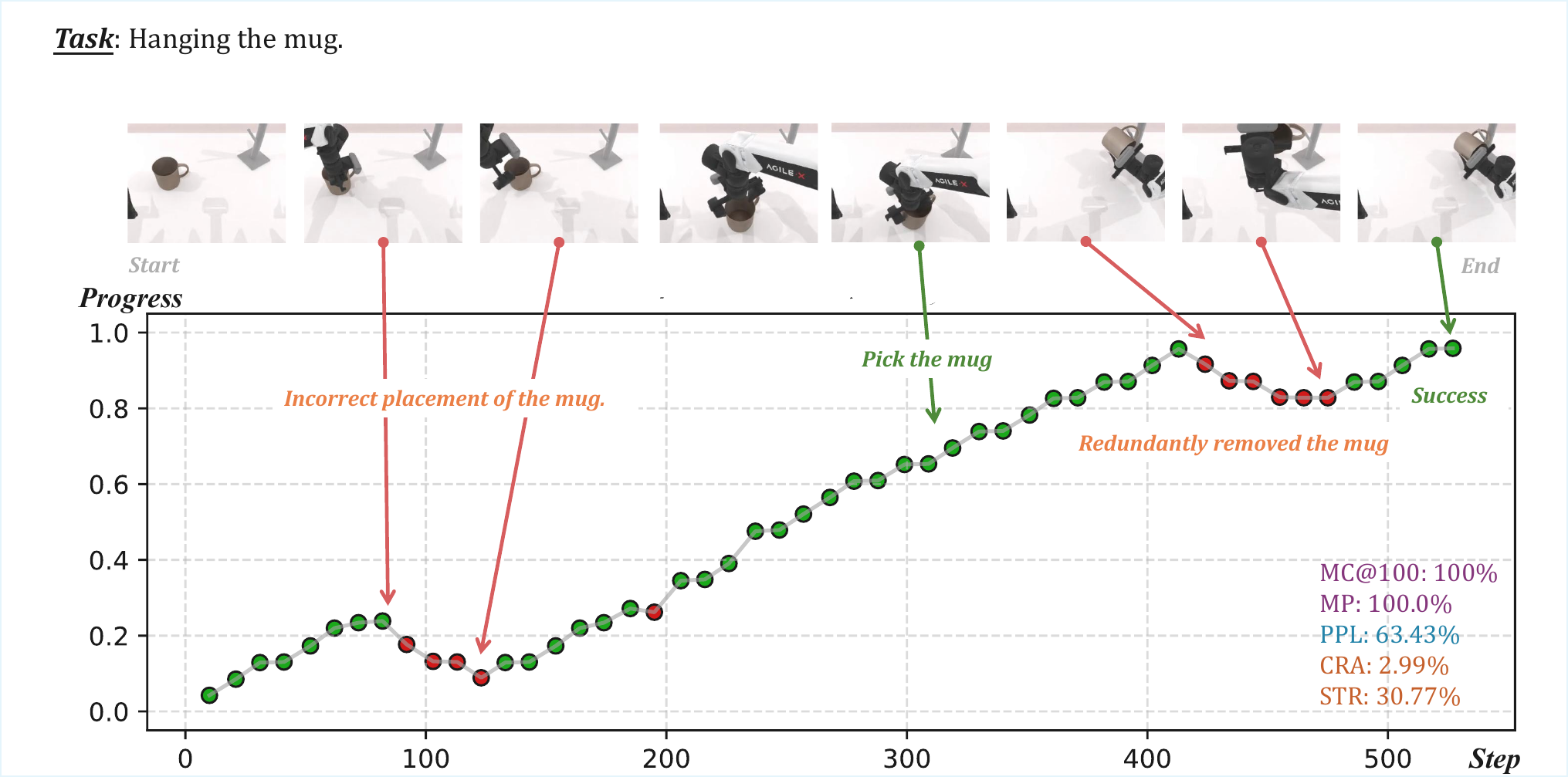}
    \caption{\textbf{Simulation Task: Hanging the mug.} Fine-grained manipulation tracking. The curve captures the ``redundant removal'' of the mug (Step $\sim 450$) as a regression before the final successful placement. This granularity allows for the detection of suboptimal behavior even in successful trajectories.}
    \label{fig:vis6}
\end{figure*}

\begin{figure*}[!t]
    \centering
    \includegraphics[width=1.03\linewidth]{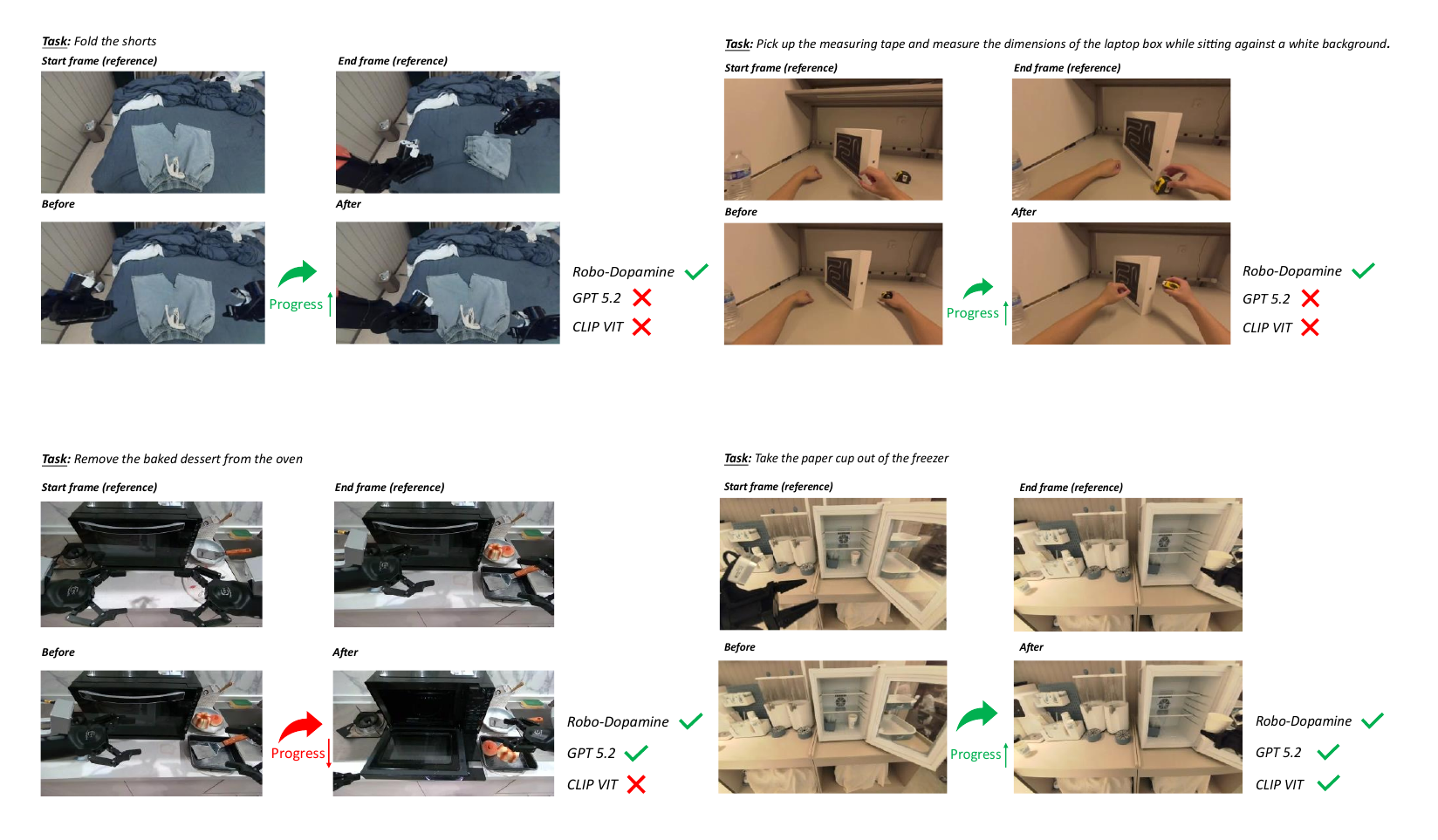}
    \caption{\textbf{RoboPulse qualitative examples for progress direction judging.}
Each panel shows the task, reference start/end frames, and a sampled BEFORE/AFTER pair from an episode.
The arrow indicates the ground-truth progress direction.
We report whether each judge predicts the correct direction.
Robo-Dopamine is robust to appearance shifts and detects both improvement and regression, whereas CLIP can fail when visual similarity is misaligned with physical progress.}
    \label{fig:robopulse_vis}
\end{figure*}

\clearpage

\section{Prompt}
\label{prompt}

In this section, we provide the exact prompt templates used to evaluate the Vision-Language Models (VLMs) and language-conditioned reward models on the RoboPulse benchmark. To ensure reproducibility and transparency, we disclose the full system instructions and input structures.

\textbf{General Foundation-Model Judges.} For general-purpose VLMs (Gemini 3 Pro, GPT-5.2, Qwen3-VL), we utilize a unified pairwise comparison prompt shown in Fig.~\ref{fig:prompt_fm}. The prompt is structured to enforce a rigorous evaluation protocol:
\begin{itemize}
    \item \textbf{Role Assignment:} Defines the model as an expert robotic judge to prime it for physical reasoning.
    \item \textbf{Task Context:} Dynamically inserts the language instruction (e.g., ``Stack the blocks'') to ground the visual analysis.
    \item \textbf{Visual Inputs:} Presents the ``Start State'' and ``Goal State'' (if available) as reference anchors, followed by the two frames to be compared ($S_A$ and $S_B$).
    \item \textbf{Output Schema:} Enforces a strict binary output in \texttt{<score>} tags for automated parsing.
\end{itemize}

\textbf{Baseline Reward Models.} For specific reward model baselines, we strictly follow the prompting protocols defined in their original implementations to ensure fair comparisons. Fig.~\ref{fig:prompt_roboreward}, Fig.~\ref{fig:prompt_vlac}, and Fig.~\ref{fig:prompt_gvl} display the templates used for RoboReward, VLAC, and GVL, respectively. These prompts are tailored to trigger the specific capabilities (e.g., failure detection or success classification) claimed by each method.

\begin{figure}[!h]
    \centering
    \includegraphics[width=\linewidth]{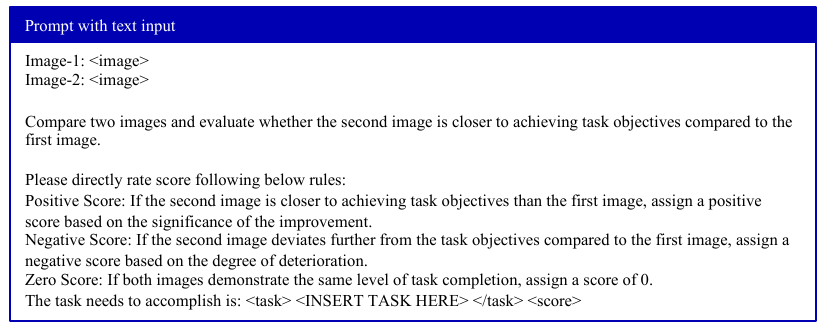}
    \caption{\textbf{Prompt template for {VLAC}~\cite{vlac} on RoboPulse.}}
\label{fig:prompt_vlac}
\end{figure}

\begin{figure}[!h]
    \centering
    \includegraphics[width=\linewidth]{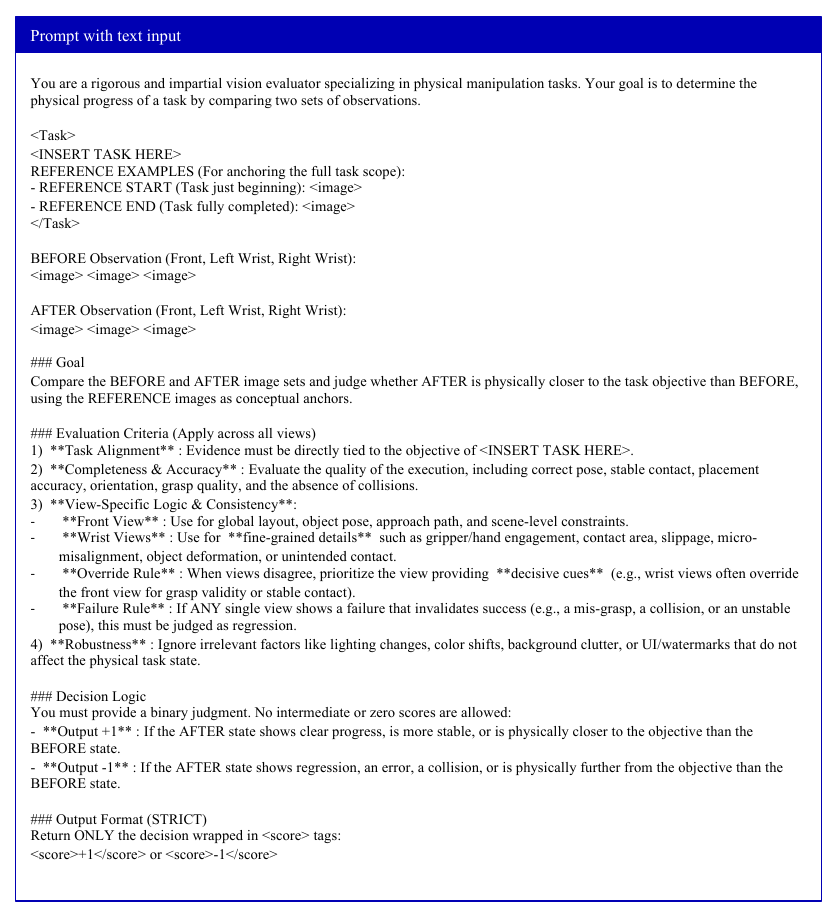}
    \caption{\textbf{Prompt template for Gemini 3 Pro~\cite{gemini3pro}, GPT-5.2~\cite{gpt5.2}, Qwen3-VL~\cite{qwen3vl} pairwise progress judgment on RoboPulse.}
We provide the full system and user prompt, including task description, reference anchors, multiview BEFORE/AFTER observations, and the required output schema used for automated parsing.}
\label{fig:prompt_fm}
\end{figure}

\begin{figure}[!h]
    \centering
    \includegraphics[width=\linewidth]{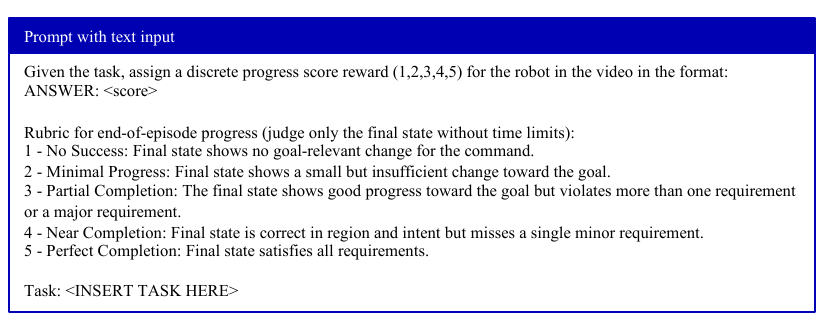}
    \caption{\textbf{Prompt template for {RoboReward}~\cite{roboreward} on RoboPulse.}}
\label{fig:prompt_roboreward}
\end{figure}

\begin{figure}[!h]
    \centering
    \includegraphics[width=1\linewidth]{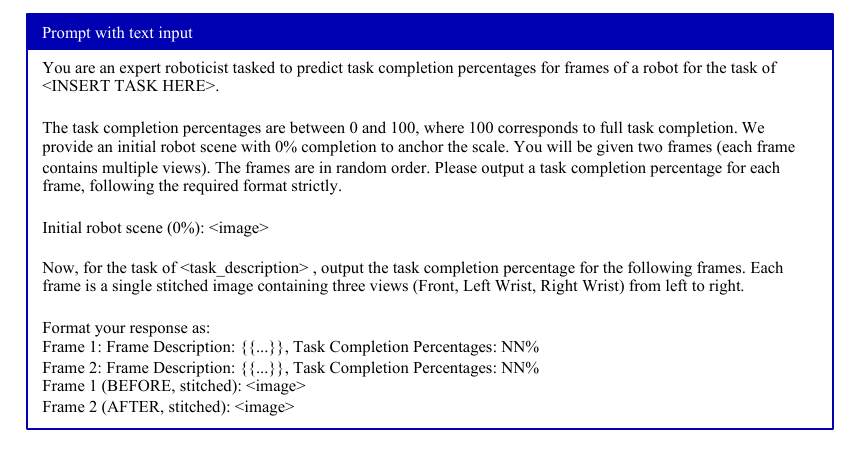}
    \caption{\textbf{Prompt template for {GVL}~\cite{gvl} on RoboPulse.}}
\label{fig:prompt_gvl}
\end{figure}

\end{document}